\RequirePackage[l2tabu,orthodox]{nag}
\documentclass
[letterpaper,11pt,]
{article}

\usepackage{etex}
\usepackage{verbatim}
\usepackage{xspace,enumerate}
\usepackage[dvipsnames]{xcolor}
\usepackage[T1]{fontenc}
\usepackage[full]{textcomp}
\usepackage[american]{babel}
\usepackage{mathtools}
\usepackage{amsthm}
\usepackage[
letterpaper,
top=1in,
bottom=1in,
left=1in,
right=1in]{geometry}
\usepackage{newpxtext} % T1, lining figures in math, osf in text
\usepackage{textcomp} % required for special glyphs
\usepackage[varg,bigdelims]{newpxmath}
\usepackage[scr=rsfso]{mathalfa}% \mathscr is fancier than \mathcal
\usepackage{bm} % load after all math to give access to bold math
\let\mathbb\varmathbb
\usepackage{microtype}
\usepackage[pagebackref,colorlinks=true,urlcolor=blue,linkcolor=blue,citecolor=OliveGreen]{hyperref}
\usepackage[capitalise,nameinlink]{cleveref}
\crefname{lemma}{Lemma}{Lemmas}
\crefname{fact}{Fact}{Facts}
\crefname{theorem}{Theorem}{Theorems}
\crefname{corollary}{Corollary}{Corollaries}
\crefname{claim}{Claim}{Claims}
\crefname{example}{Example}{Examples}
\crefname{algorithm}{Algorithm}{Algorithms}
\crefname{problem}{Problem}{Problems}
\crefname{definition}{Definition}{Definitions}
\crefname{model}{Model}{Models}
\crefname{exercise}{Exercise}{Exercises}
\crefname{condition}{Condition}{Conditions}
\usepackage{amsthm}

\newtheorem{theorem}{Theorem}[section]
\newtheorem*{theorem*}{Theorem}
\newtheorem{lemma}[theorem]{Lemma}
\newtheorem*{lemma*}{Lemma}
\newtheorem{fact}[theorem]{Fact}
\newtheorem*{fact*}{Fact}

\newtheorem*{proposition*}{Proposition}
\newtheorem{corollary}[theorem]{Corollary}
\newtheorem*{corollary*}{Corollary}

\newtheorem*{hypothesis*}{Hypothesis}

\newtheorem*{conjecture*}{Conjecture}
\theoremstyle{definition}
\newtheorem{definition}[theorem]{Definition}
\newtheorem*{definition*}{Definition}
\newtheorem{model}[theorem]{Model}
\newtheorem*{model*}{Model}

\newtheorem*{construction*}{Construction}

\newtheorem*{example*}{Example}

\newtheorem*{question*}{Question}
\newtheorem{algorithm}[theorem]{Algorithm}
\newtheorem*{algorithm*}{Algorithm}

\newtheorem*{assumption*}{Assumption}

\newtheorem*{problem*}{Problem}

\newtheorem*{openquestion*}{Open Question}
\theoremstyle{remark}

\newtheorem*{claim*}{Claim}
\newtheorem{remark}[theorem]{Remark}
\newtheorem*{remark*}{Remark}

\newtheorem*{observation*}{Observation}
\usepackage{paralist}
\let\originalleft\left
\let\originalright\right
\renewcommand{\left}{\mathopen{}\mathclose\bgroup\originalleft}
\renewcommand{\right}{\aftergroup\egroup\originalright}
\usepackage{turnstile}
\usepackage{mdframed}
\usepackage{tikz}
\usetikzlibrary{positioning}
\usepackage{caption}
\DeclareCaptionType{Algorithm}
\usepackage{newfloat}
\usepackage{array}
\usepackage{subfig}
\usepackage{bbm}
\usepackage{xparse}
\usepackage{amsthm} % for \@addpunct
\makeatletter
\let\latexparagraph\paragraph
\RenewDocumentCommand{\paragraph}{som}{%
  \IfBooleanTF{#1}
    {\latexparagraph*{#3}}
    {\IfNoValueTF{#2}
       {\latexparagraph{\maybe@addperiod{#3}}}
       {\latexparagraph[#2]{\maybe@addperiod{#3}}}%
  }%
}
\newcommand{\maybe@addperiod}[1]{%
  #1\@addpunct{.}%
}
\makeatother

\usepackage{boxedminipage}
\newcommand{\paren}[1]{(#1)}
\newcommand{\Paren}[1]{\left(#1\right)}

\newcommand{\brac}[1]{[#1]}
\newcommand{\Brac}[1]{\left[#1\right]}

\newcommand{\abs}[1]{\lvert#1\rvert}
\newcommand{\Abs}[1]{\left\lvert#1\right\rvert}

\newcommand{\card}[1]{\lvert#1\rvert}
\newcommand{\Card}[1]{\left\lvert#1\right\rvert}

\newcommand{\set}[1]{\{#1\}}
\newcommand{\Set}[1]{\left\{#1\right\}}

\newcommand{\norm}[1]{\lVert#1\rVert}
\newcommand{\Norm}[1]{\left\lVert#1\right\rVert}

\newcommand{\normt}[1]{\norm{#1}_2}
\newcommand{\Normt}[1]{\Norm{#1}_2}

\newcommand{\snorm}[1]{\norm{#1}^2}
\newcommand{\Snorm}[1]{\Norm{#1}^2}

\newcommand{\normi}[1]{\norm{#1}_\infty}
\newcommand{\Normi}[1]{\Norm{#1}_\infty}

\newcommand{\iprod}[1]{\langle#1\rangle}
\newcommand{\Iprod}[1]{\left\langle#1\right\rangle}

\newcommand{\Esymb}{\mathbb{E}}
\newcommand{\Psymb}{\mathbb{P}}

\DeclareMathOperator*{\E}{\Esymb}

\DeclareMathOperator*{\ProbOp}{\Psymb}
\renewcommand{\Pr}{\ProbOp}
\newcommand{\given}{\mathrel{}\middle\vert\mathrel{}}
\newcommand{\sge}{\succeq}

\newcommand{\from}{\colon}
\newcommand{\pE}{\tilde{\mathbb{E}}}
\newcommand{\mper}{\,.}

\newcommand\bdot\bullet
\DeclareMathOperator{\Ind}{\mathbf 1}
\DeclareMathOperator{\Tr}{Tr}

\DeclareMathOperator{\argmin}{argmin}
\DeclareMathOperator{\polylog}{polylog}
\DeclareMathOperator{\supp}{supp}

\DeclareMathOperator{\sign}{sign}

\newcommand{\iid}{i.i.d.\xspace}
\newcommand{\N}{\mathbb N}
\newcommand{\R}{\mathbb R}

\newcommand{\cA}{\mathcal A}
\newcommand{\cB}{\mathcal B}
\newcommand{\cC}{\mathcal C}

\newcommand{\cE}{\mathcal E}

\newcommand{\cG}{\mathcal G}

\newcommand{\cL}{\mathcal L}

\newcommand{\cN}{\mathcal N}

\newcommand{\cP}{\mathcal P}

\newcommand{\cR}{\mathcal R}
\newcommand{\cS}{\mathcal S}

\newcommand{\cU}{\mathcal U}

\newcommand{\bbP}{\mathbb P}

\renewcommand{\leq}{\leqslant}
\renewcommand{\le}{\leqslant}
\renewcommand{\geq}{\geqslant}
\renewcommand{\ge}{\geqslant}
\let\epsilon=\varepsilon
\numberwithin{equation}{section}
\newcommand\MYcurrentlabel{xxx}
\newcommand{\MYstore}[2]{%
  \global\expandafter \def \csname MYMEMORY #1 \endcsname{#2}%
}
\newcommand{\MYload}[1]{%
  \csname MYMEMORY #1 \endcsname%
}
\newcommand{\MYnewlabel}[1]{%
  \renewcommand\MYcurrentlabel{#1}%
  \MYoldlabel{#1}%
}
\newcommand{\MYdummylabel}[1]{}
\newcommand{\torestate}[1]{%
  \let\MYoldlabel\label%
  \let\label\MYnewlabel%
  #1%
  \MYstore{\MYcurrentlabel}{#1}%
  \let\label\MYoldlabel%
}
\newcommand{\restatetheorem}[1]{%
  \let\MYoldlabel\label
  \let\label\MYdummylabel
  \begin{theorem*}[Restatement of \cref{#1}]
    \MYload{#1}
  \end{theorem*}
  \let\label\MYoldlabel
}
\newcommand{\restatelemma}[1]{%
  \let\MYoldlabel\label
  \let\label\MYdummylabel
  \begin{lemma*}[Restatement of \cref{#1}]
    \MYload{#1}
  \end{lemma*}
  \let\label\MYoldlabel
}
\newcommand{\restateprop}[1]{%
  \let\MYoldlabel\label
  \let\label\MYdummylabel
  \begin{proposition*}[Restatement of \cref{#1}]
    \MYload{#1}
  \end{proposition*}
  \let\label\MYoldlabel
}
\newcommand{\restatefact}[1]{%
  \let\MYoldlabel\label
  \let\label\MYdummylabel
  \begin{fact*}[Restatement of \cref{#1}]
    \MYload{#1}
  \end{fact*}
  \let\label\MYoldlabel
}
\newcommand{\restate}[1]{%
  \let\MYoldlabel\label
  \let\label\MYdummylabel
  \MYload{#1}
  \let\label\MYoldlabel
}
\newcommand{\eps}{\epsilon}

\allowdisplaybreaks
\newcommand*{\Id}{\mathrm{Id}}

\newcommand*{\normf}[1]{\norm{#1}_{\mathrm{F}}}
\newcommand*{\Normf}[1]{\Norm{#1}_{\mathrm{F}}}
\newcommand{\normn}[1]{\norm{#1}_\textnormal{nuc}}

\newcommand{\Normm}[1]{\Norm{#1}_\textnormal{max}}

\newcommand{\ind}[1]{\mathbf{1}_{\Brac{#1}}}
\renewcommand{\normi}[1]{\norm{#1}_{\max}}
\renewcommand{\Normi}[1]{\Norm{#1}_{\max}}

\newcommand{\normin}[1]{\norm{#1}_{\text{inj}}}

\newcommand*{\transpose}[1]{{#1}{}^{\mkern-1.5mu\mathsf{T}}}
\newcommand*{\dyad}[1]{#1#1{}^{\mkern-1.5mu\mathsf{T}}}

\providecommand{\pdset}{\tilde{\Omega}}%_{\overline{\Omega}_{b}, p,  \cD_t}}

\title{
Higher degree sum-of-squares relaxations robust against oblivious outliers
  \thanks{This project has received funding from the European Research Council (ERC) under the European Union's Horizon 2020 research and innovation programme (grant agreement No 815464).}
}

\author{
  Tommaso d'Orsi\thanks{ETH Z\"urich.}
\and
Rajai Nasser\thanks{Google Z\"urich. This work was done at ETH Z\"urich.}
\and
Gleb Novikov\footnotemark[2]
\and
David Steurer\footnotemark[2]
}

\begin{document}

\pagestyle{empty}

% MAKE TITLE

\maketitle
\thispagestyle{empty} % seems to be required here to avoid page number on first page

% ABSTRACT

\begin{abstract}
We consider estimation models of the form $\mathbf{Y}=X^*+\mathbf{N}$, where $X^*$ is some $m$-dimensional structured signal we wish to recover, and $\mathbf{N}$ is symmetrically distributed noise that may be unbounded in all but a small $\alpha$ fraction of the entries.
This setting captures problems such as (sparse) linear regression, (sparse) principal component analysis (PCA), and tensor PCA, even in the presence of oblivious outliers and heavy-tailed noise.

We introduce a family of algorithms that under mild assumptions recover the signal $X^*$ in all estimation problems for which there exists a sum-of-squares algorithm that succeeds in recovering the signal $X^*$ when the noise $\mathbf{N}$ is Gaussian.
This essentially shows that it is enough to design a sum-of-squares algorithm for an estimation problem with Gaussian additive noise in order to get the algorithm that works with the symmetric noise model.

Our framework extends far beyond  previous results on symmetric noise models and is even robust to an $\epsilon$-fraction of adversarial perturbations.
As concrete examples, we investigate two problems for which no efficient algorithms were known to work for heavy-tailed noise: tensor PCA and sparse PCA.

For the former, our algorithm recovers the principal component in polynomial time when the signal-to-noise ratio is at least $\tilde{O}(n^{p/4}/\alpha)$, that matches (up to logarithmic factors) current best known algorithmic guarantees for Gaussian noise.
For the latter, our algorithm runs in quasipolynomial time and matches the state-of-the-art guarantees for quasipolynomial time algorithms in the case of Gaussian noise.
Using a reduction from the planted clique problem, we provide evidence that the quasipolynomial time is likely to be necessary for sparse PCA with symmetric noise.

In our proofs we use bounds on the covering numbers of sets of pseudo-expectations, which we obtain by certifying in sum-of-squares upper bounds on the Gaussian complexities of sets of solutions.
This approach for bounding the covering numbers of sets of pseudo-expectations may be interesting in its own right and may find other application in future works.

\end{abstract}

\clearpage

% TOC

% assumes microtype
\microtypesetup{protrusion=false}
\tableofcontents{}
\microtypesetup{protrusion=true}

\clearpage

\pagestyle{plain}
\setcounter{page}{1}

% SECTION

\section{Introduction} \label{sec:introduction}

Consider an estimation problem over $\R^m$, in which we observe (a realization of)\footnote{We denote random variables in boldface.} the random variable $\bm{Y} = X^* + \bm{N}$ where $X^*\in \Omega\subseteq \R^m$ is some structured signal we seek to recover and $\bm{N}$ is some additive noise.
This generic primitive captures widely studied models such as compressed sensing, linear regression, principal component analysis, clustering mixture distributions, matrix completion or tensor principal component analysis.

From both a statistical and computational point of view, as one weakens the assumptions on the noise $\bm{N}$, the  task of reconstructing the hidden signal $X^*$ becomes harder.	\footnote{The complexity of the problem may also be affected by the structure of $X^*$.} Recent years have seen tremendous advances in the design of efficient algorithms able to recover the planted structure $X^*$, under weaker and weaker assumptions on the noise (e.g. \cite{DBLP:conf-colt-KlivansKM18, hopkins2020mean, DBLP:conf-focs-dOrsiKNS20,DBLP:conf-focs-DingdNS21, DBLP:conf-stoc-BakshiDJKKV22}).
In particular, a certain line of work \cite{DBLP:journals/jacm/CandesLMW11, DBLP:conf/isit/ZhouLWCM10, tsakonas2014convergence, DBLP:conf/nips/Bhatia0KK17,   SuggalaBR019, sun2020adaptive, pensia2020robust, ICML-linear-regression, DBLP:conf-colt-Chend22, DBLP:conf-nips-dOrsiLNNST21}  
has aimed to identify the weakest possible requirements on the signal-to-noise ratio so that it is still possible to efficiently recover the signal $X^*$  with \textit{vanishing} error.
%\Gnote{it is a fine task for linear regression, where in some cases non-vanishing error also might make sense, but does it make any sense to consider any version of pca with error which is not close to zero?} \Tnote{But we are not saying that it is interesting to ask that. We are saying that iit makes sense  to ask under which conditions the error vanishes}. 
In this context an established model\footnote{Sometimes denoted the \textit{oblivious adversarial model}.}  is that of assuming the entries of $\bm{N}$ to be (i) independent, (ii) symmetric about zero and (iii) to have some probability mass in a neighborhood of $0$. That is, to satisfy $\bbP \Brac{\Abs{\bm{N}_i}\leq 1}\geq \alpha$, for  $i\in [m]$  with the parameter $\alpha$ possibly vanishingly small.

Remarkably, the framework emerging from these results  shows that the Huber loss estimator%\footnote{See \cref{sec:preliminaries} for a precise definition.} 
 --when equipped with an appropriate regularizer--  offers provably optimal error guarantees among efficient estimators. In particular, it recovers the  error convergence rates of classical least squares algorithms in the presence of  Gaussian noise.

The general recipe behind these results relies on two main points: first an upper bound on the gradient of the Huber loss function at the true solution $X^*$,\footnote{The attentive reader may notice that no such upper bound exists for the least square estimator, under the noise assumptions above. This is evidence confirming the intuition that "least squares estimator are fragile to outliers".} and second a lower bound on the curvature of the loss function (in the form of  a local strong convexity bound) within a \textit{structured} neighborhood of $X^*$.  Here the structured neighborhood of $X^*$ is a superset $\bar{\Omega}\supseteq \Omega$. 
The curvature of the loss function depends on the directions (and the radius) considered and (one expect that) it is sharper in the directions contained in $\Omega$. Thus one can establish stronger statistical guarantees by forcing the minimizer of the loss function to be in a small set of directions close to $\Omega$.
The crux of the argument is that the set $\bar{\Omega}$ is controlled by the regularizer: If the chosen regularizer is \textit{norm decomposable\footnote{We formally define decomposability in \cref{sec:decomposability}.} with respect to a meaningful set  $\bar{\Omega}$}, then indeed it will force the minimizer to fall in one of the desired structured directions.

 The inherent consequence of this approach is that, in settings where no such decomposable norm regularizer is known --such as for tensor principal component analysis-- these estimators cannot provide \textit{any} error guarantees.
 In this paper, we overcome this limitation and introduce a family of algorithms (based on sum-of-squares) that recover the parameter $X^*$ for a remarkably large set of models. More concretely, our result can be informally read  (under certain reasonable conditions) as: 
 
\begin{quote}
	\textit{Whenever there exists a degree-$\ell$ sum-of-squares algorithm that recovers $X^*$ from $\bm{Y}$ when the entries of $\bm{N}$ are Gaussian with standard deviation\footnote{As we will see in the context of sparse PCA problem, it is unlikely that we can relax the condition $\sigma > \normi{X}$.} $\sigma = (1+\Normm{X^*})/\alpha$, there also exists an algorithm running in time $m^{O(\ell)}$  that recovers $X^*$ with the same guarantees, even if $\bm{N}$ only satisfies (i), (ii), (iii).} 
\end{quote}
 
 In other words, we introduce a framework that allows to directly generalize sum-of-squares algorithms designed to recover the hidden signal in the presence of Gaussian noise, to the significantly more general settings of symmetric noise.

 Our result relies on a novel use of the sum-of-squares hierarchy. The core of the argument consists of \textit{bounds on the covering number of sets of pseudo-expectations}, which we obtain via sum-of-squares certificates of the Gaussian complexity of the space of solutions. We then use these small covers to ensure that  feasible solutions must fall in one of  a few directions close to $X^*$.

\subsection{Results}\label{sec:results}

Our main result is the following meta-theorem for recovering  a structured signal from symmetric noise.

\begin{theorem}[SoS meta-theorem, Informal]\label{theorem:metaTheorem}
	Let $m,\ell \in \N$. Let $\Omega\subset \R^m$ be a set defined by at most $m^{O(1)}$ polynomial constraints of degree\footnote{These constaints may use up to $m^{O(1)}$ auxiliary variables, and degrees of all polynomials in all variables are  at most $\ell$.} at most $\ell$.
	Suppose that for some $r > 0$ and $\gamma \ge r\sqrt{\ln m}$  the following bounds are certifiable by degree $O(\ell)$ sum-of-squares proofs (from the constraints that define $\Omega$):
	%\Dnote{there is a precise definition of what it means for a polynomial inequality to have an sos proof. i think it is not what is meant here (because for every W you might have a different sos proof). perhaps you want to write instead: the following bound on the guassian complexity of X is certifable by degree-ell sos proofs}
	\begin{enumerate}
		\item[(1)]  	$\sup_{X\in \Omega}\normi{X} \le 1\,,$
		\item[(2)]  	$\sup_{X\in \Omega}\normt{X} \le r\,,$
		\item[(3)]   	$\E_{\bm{W}\sim N(0,\Id_m)} \sup_{X\in \Omega}\iprod{X, \bm{W}} \leq \gamma\,.$
	\end{enumerate}
	
	Let $0 < \alpha \le 1$ and let $\bm{N}$ be a random $m$-dimensional vector with independent (but not necessarily identically distributed) symmetric about zero\footnote{I.e., $\bm{N}_i$ and $-\bm{N}_i$ have the same distribution for every $i\in[m]$.} entries satisfying $\bbP \brac{\Abs{\bm N_i}\leq 1}\geq \alpha$ for all $i\in[m]$. 
	
	There exists an algorithm running in time  $m^{O(\ell)}$
	that on input $\bm{Y} = X^*+\bm {N}$ outputs $\hat{\bm{X}}$ satisfying 
	\begin{align*}
	\Snorm{X^*-\hat{\bm{X}}}_2\leq O\Paren{\gamma/\alpha}
	\end{align*}
	with high probability.
	
	Moreover, for $\eps\lesssim \frac{\gamma^2}{r^2\cdot m\ln m}$, the same result holds if an arbitrary (adversarially chosen) $\epsilon$-fraction of entries of $\bm{Y}$ is replaced by adversarially chosen values.
\end{theorem}

%\Gnote{old version:}
%\begin{theorem}[SoS meta-theorem]\label{theorem:metaTheorem}
%	Let $0<\alpha(m)\leq 1$. Let $\Omega\subseteq\R^m$ be such that for every $X\in \Omega$, $\Normi{X}\leq 1$ and $\Normt{X}\leq r$. Consider
%	\begin{align*}
%		\bm{Y} = X^*+\bm{N}
%	\end{align*}
%	where $X^*\in \Omega$ and $\bm{N}$ is a random $m$-dimensional vector with independent and symmetrically distributed\footnote{I.e., $\bm{N}_i$ and $-\bm{N}_i$ have the same distribution for every $i\in[m]$.} entries satisfying $\bbP \brac{\Abs{N_i}\leq 1}\geq \alpha$ for all $i\in[m]$.
%	Suppose  the following bound is certifiable by degree-$\ell$ sum-of-squares proofs:
%  	%\Dnote{there is a precise definition of what it means for a polynomial inequality to have an sos proof. i think it is not what is meant here (because for every W you might have a different sos proof). perhaps you want to write instead: the following bound on the guassian complexity of X is certifable by degree-ell sos proofs}
%	\begin{align*}
%		\E_{\bm{W}\sim N(0,\Id_m)} \sup_{X\in \Omega}\iprod{X, \bm{W}} \leq \gamma\,.
%	\end{align*}
%	Then, there exists an algorithm running in time  $m^{O(\ell)}$, that on input $\bm{Y}$, with probability $1-\delta$, outputs $\hat{\bm{X}}$ satisfying 
%	\begin{align*}
%		\Snorm{X^*-\hat{\bm{X}}}_2\leq O\Paren{\frac{\gamma+r\sqrt{\log(1/\delta)}}{\alpha}}\,.
%	\end{align*}
%	Moreover, for $\eps\leq \frac{\gamma^2}{r^2\cdot m\log m}$, the same result holds if an arbitrary (unknown) $\epsilon$-fraction of entries of $\bm{Y}$ is replaced by adversarially chosen values.
%\end{theorem}

It is possible to gain an understanding of the importance of \cref{theorem:metaTheorem} even before applying it to specific problems.
First, notice that if $\gamma/\alpha \le o(r^2)$, the error guarantees are non-trivial. 
In particular this means that the fraction $\alpha$ of entries with bounded noise can be vanishingly small and the algorithm can still reconstruct a meaningful estimate.
Second, observe how the error rate crucially depends on the upper bounds we are able to certify on the Gaussian complexity of the space of solutions $\Omega$. By certifying tighter bounds on it one can obtain tighter guarantees on the error of the estimation. This shows the existence of a trade-off between error of the estimate and running time.
Finally we remark that the algorithm is robust to an $\epsilon$-fraction of adversarial corruptions, the magnitude of $\epsilon$ will become clearer when discussing the various applications.

Next we apply \cref{theorem:metaTheorem} to specific problems.

%\Dnote{the above theorem statement is super difficult to parse because of the many parameters involved.
%  It looks like b and zeta are the same (because they are sumed in the final bound).
%  So maybe one can scale things so that b and zeta are both 1.
%  i think that makes parsing things quite a bit easier.
%  Perhaps one could also fix delta to be $m^{100}$.
%}

%\Tnote{Things to point out:
%
%\begin{itemize}
%	\item the guarantees of the algorithm  depends on the Gaussian complexity
%	\item $\alpha$ can be vanishing small
%	\item it is robust to $\epsilon$ outliers
%\end{itemize}
%}

\paragraph{Tensor principal component analysis}
We consider the following tensor PCA model (we remark that one may consider further tensor models, in \cref{section:applications} we study other versions of \cref{model:obliviousTPCA} as well). 

\begin{model}[Tensor PCA with asymmetric tensor noise]
	\label{model:obliviousTPCA}
	Let $n,p\in \N$, $n,p\ge 2$, and $0<\alpha \leq 1$. 
	We observe (an instance of)  $\bm{Y} = \lambda \cdot v^{\otimes p}+\bm{N}\,,$ where $\lambda > 0$, $v\in\R^n$ is an unknown unit vector and $\bm N$ is a random order $p$ tensor with independent (but not necessarily identically distributed) symmetric about zero entries such that
	\begin{equation*}
	\Pr[|\bm{N}_{i_1\ldots i_p}|\leq 1]\geq \alpha\,,\quad\text{for all } 1\leq i_1,\ldots,i_p\leq n\,.
	\end{equation*}
	%Given $\bm{Y}$ and $\epsilon>0$,  find a unit vector $\hat{\bm{v}}$ such that $\norm{\hat{\bm{v}}-v}\leq \epsilon$ with probability at least $1-o(1)$.
\end{model}

In the significantly more restrictive settings when the noise is standard Gaussian (captured by \cref{model:obliviousTPCA} by the special case with $\alpha\ge \Omega(1)$), this model was studied in  \cite{DBLP:journals/corr/MontanariR14a, tensor_pca_sos}.
In these settings, one can recover the hidden vector $v$ in exponential time whenever the signal-to-noise ratio $\lambda$ is at least $\Omega(\sqrt{n})$, but existing polynomial time algorithms are known to require at least  $\lambda \geq \Omega\Paren{n^{p/4}}$.
Moreover,  evidence of an information-computation gap exists in the literature in the form of lower bounds against different computational models (sum-of-squares lower bounds
\cite{tensor_pca_sos, HopkinsKPRSS17} or low degree polynomial lower bounds \cite{ldp_lowerbound_tpca}),  showing that these computational models cannot recover the hidden vector in polynomial time if $\lambda <  n^{p/4}/\polylog(n)$. 
%\Dnote{regarding citations:
%  (1) for the future please use author year conventions for the bibtex keys.
%  otherwise, it is very difficult to infer what paper is cited
%  (2) i think there are more recent and much better sos lower bounds for tensor pca, e.g., HopkinsKPRSS17.
%  we should also cite those.}

Less restrictive noise models have been  considered more recently.
\cite{tpca_walks} proved that when the noise has zero mean and bounded variance and  $v$ is a \textit{random} vector whose entries have small fourth moment, then one can recover it as long as $\lambda \ge \Omega(n^{p/4})$.
Later, \cite{tpca_iterations} showed that if the noise has zero mean and bounded variance, there exists an algorithm that, under mild assumption on the magnitude of the entries of $v$, can recover $v$ as long as $\lambda \gtrsim n^{p/4}\cdot \Paren{\ln n}^{1/4}$.

However, an application of \cref{theorem:metaTheorem} shows that whenever the entries of the noise are symmetric about zero, \textit{no} assumption on the moments is needed to recover the parameter $v$.
The  application of \cref{theorem:metaTheorem} only relies on known  sum-of-squares certificates for the injective tensor norm of random tensors  \cite{tensor_pca_sos}.%, and for odd $p$ these certificates are a bit worse than for even $p$, so we get additional log factors in the bound on $\lambda$.

\begin{theorem}[Robust Tensor PCA]\label{theorem:tensorPca}
	Let $p\ge 2$. 
	There exists an absolute constant $C > 1$, and an algorithm running in time $n^{O(p)}$ that, 
	given $\bm{Y}$ as in \cref{model:obliviousTPCA},
	outputs a unit vector $\hat{\bm{v}}\in \R^n$ satisfying
	\begin{align*}
	\abs{\iprod{v,\hat{\bm{v}}}}  \ge 0.99
	\end{align*}
	with high probability, whenever
	\begin{itemize}
		\item If $p$ is even: $\lambda \ge \frac{C}{\alpha}\cdot n^{p/4} \;$  and 
		$\;\normi{v} \le  \frac{\alpha^{1/p}}{C} \cdot n^{-1/4}$\,.
		\item If $p$ is odd: $\lambda \ge \frac{C  \Paren{p\ln n}^{1/4}}{\alpha}  \cdot n^{p/4}\;$ and 
		$\;\normi{v} \le \frac{\alpha^{1/p}}{C(p\ln n)^{1/4p}} \cdot n^{-1/4}$\,.
	\end{itemize}
	
	Moreover, if $p$ is odd, the algorithm recovers the sign of $v$, that is, ${\iprod{v,\hat{\bm{v}}}}  \ge 0.99$ with high probability.
	
	Furthermore, for $\eps\le \Paren{C \cdot p\cdot n^{p/2} \cdot \ln n}^{-1}$, the same result holds if an arbitrary (adversarially chosen) $\epsilon$-fraction\footnote{For odd $p$ we allow slightly greater fraction of corruptions $\eps\le \Paren{C\cdot p\cdot   n^p \cdot \ln n}^{-1/2}$.} of entries of $\bm{Y}$ is replaced by adversarially chosen values.
\end{theorem}

Let us briefly and informally describe how this result can be obtained from \cref{theorem:metaTheorem}. Consider the case when $p$ is odd\footnote{The case when $p$ is even is  similar.}. Let $b = \frac{\alpha^{1/p}}{C(p\ln n)^{1/4p}} \cdot n^{-1/4}$.
We may rescale $\bm Y$ by $1/\Paren{\lambda b^p} \le 1$ so that $\normi{X^*}\le 1$ 
 and the bound $\Pr[|\bm{N}_{i_1\ldots i_p}|\leq 1]\geq \alpha$ still holds for all $1\leq i_1,\ldots,i_p\leq n$. 
 Note that  now $r := \normt{X^*} = 1/b^p$. So we trivially have the desired sum-of-squares certificates for (1) and (2) in \cref{theorem:metaTheorem}.
Most importantly, from \cite{tensor_pca_sos} we know that for the set $\Omega$ of rank-one symmetric tensors of norm $r$  there is a degree $O(p)$ sum-of-squares proof that certifies the bound 
\[
\E_{\bm{W}\sim N(0,\Id_{n^p})} \sup_{X\in \Omega}\iprod{X, \bm{W}} \le 
O\Paren{p\cdot \Paren{\ln n}\cdot n^p}^{1/4}  \cdot r\,. 
\]
Thus using the value on the right-hand side as $\gamma$, we get that $\hat{\bm X}$ that is obtained from \cref{theorem:metaTheorem} satisfies
	\begin{align*}
	\Snorm{X^*-\hat{\bm{X}}}_2\leq O\Paren{\gamma/\alpha}\leq  O\Paren{\frac{r}{(Cb)^p}}=O\Paren{\frac{1}{C^p}}\cdot r^2\,,
	\end{align*}
and hence $\hat{\bm X}$ is highly correlated with $X^*$ and the result follows\footnote{We also need to perform rounding to obtain the vector from the output tensor. See \cref{sec:rounding} for more details.}.

Concerning the noise $\bm N$, it is easy to observe that the  algorithm works with symmetric heavy tailed noise (e.g., Cauchy noise) and achieves guarantees similar to the best known guarantees for standard Gaussian noise. Moreover, the number of adversarial corruptions that the algorithm allows is nearly optimal: For instance, for constant even $p$ and constant $\alpha$ our bound on the entries allows $v$ to be $O\Paren{\sqrt{n}}$-sparse. Hence for such $v$, if the adversary is allowed to make more than ${n^{p/2}}$ corruptions, the signal can be completely removed and the problem becomes information-theoretically unsolvable. Our theorem guarantees that if the number of corruptions is $o\Paren{n^{p/2} / \log n}$, we can find a vector highly correlated with $v$ in polynomial time.

The dependence of $\lambda$ on $\alpha$ is also likely to be optimal since we match (up to $\Paren{\log n}^{1/4}$ factor) the current best known guarantees for Gaussian noise with standard deviation $\Theta(1/\alpha)$.

We remark that some bound on the magnitude of the entries is needed\footnote{In fact, this is a recurring theme for unbounded noise models.} even if we do not allow adversarial corruptions. For example, if the vector $v$ is $1$-sparse (so it has one large entry), then the unbounded  noise  removes the information about $v$ with probability $1-\alpha$. Indeed, if the noise entries are sampled from the mixture of the uniform distribution on $[-1,1]$ with weight $\alpha$ and the Gaussian $N(0, 2^{n})$ with weight $1-\alpha$, then with probability $1-\alpha$ the entry that corresponds to the support of $v$ has vanishing small signal-to-noise ratio.

Evidence of the tightness of these requirements can also be found in the observation that, for $p = O(1)$ and arbitrarily small constant $\delta > 0$,  it is unlikely that a $n^{1/2-\delta}$-sparse flat $v$ can be recovered in polynomial time from the upper simplex of the input (i.e. the set of entries $\bm Y_{i_1\ldots i_p} $such that $i_1<\ldots < i_p$). Indeed the planted clique in random hypergraph problem can be reduced to this question (see \cref{section:lowerboundTensorPca}). In other words, for certain vectors with $\normi{v} \le n^{-1/4 + \delta/2}$ the problem of recovering $v$ from the upper simplex is likely to be computationally hard. 
It is not difficult to see that if we can use our SoS-based approach to recover $k$-sparse flat vectors from $\bm Y$, 
then we can also add additional sparsity constraints and get 
an SoS-based algorithm that recovers  $k$-sparse flat vectors from the upper simplex of $\bm Y$ (if $p=O(1)$).
This shows that the assumption on $\normi{v}$ in \cref{theorem:tensorPca} is likely to be inherent, at least for our SoS-based approach.
It remains a fascinating open question whether for specific noise distributions (e.g., Cauchy) the bound on $\normi{v}$ from \cref{theorem:tensorPca} is tight.

\paragraph{Sparse principal component analysis}
We consider the following sparse PCA model with symmetric noise.

\begin{model}[Sparse PCA, single spike model]
	\label{model:obliviousSparsePCA}
	Let $n, k\in \N$, $k \le n$ and $0<\alpha\leq 1$. Observe (an instance of) $\bm{Y} = \lambda\cdot\dyad{v} + \bm{N}$, where $\lambda > 0$, $v\in \R^n$ is an unknown $k$-sparse unit vector and $\bm{N}$ is a random  $n$-by-$n$ matrix  with independent (but not necessarily identically distributed) symmetric about zero entries such that
	\begin{equation*}
	\Pr[|\bm{N}_{ij}|\leq 1]\geq \alpha\,,\quad\text{for all } 1\leq i,j\leq n\,.
	\end{equation*}
\end{model}

When the noise is Gaussian this model is called the \emph{spiked Wigner model} \cite{feral2007largest, johnstone2009consistency, DBLP:journals/jmlr/DeshpandeM16, bandeira_spca, DBLP:conf-focs-dOrsiKNS20}. For Gaussian noise, when $\lambda > \sqrt{n}$ (this is called the strong signal regime) the leading eigenvector of $\bm Y$ correlates with the signal and thus a simple singular value decomposition provides optimal guarantees.
In the weak signal regime --that is when $\lambda < \sqrt{n}$-- polynomial time algorithms are known to recover the principal component $v$ whenever $\lambda \gtrsim k\sqrt{\log(n/k^2)}$  \cite{DBLP:journals/jmlr/DeshpandeM16, DBLP:conf-focs-dOrsiKNS20}.
In the sparse regime $k< n^{0.5-\delta}$ (for arbitrary constant $\delta > 0$), one can improve over these results in quasipolynomial time. Concretely, there exist algorithms \cite{bandeira_spca, DBLP:conf-focs-dOrsiKNS20, sparse_tensor_pca} that can recover the signal $v$ in time $n^{O(t)}$  as long as $\lambda \gtrsim k\sqrt{\frac{\log n}{t}}$ for arbitrary $1\le t\le k$. 
So for $t = \Theta(\log n)$ these algorithms can recover the signal in time $n^{O(\log n)}$ as long as $\lambda \ge k$. 
In the  regime $k< n^{0.5-\delta}$ no $n^{o\Paren{\log n}}$ time algorithm is known to recover the signal if $\lambda \le O(k)$, and there exist lower bounds (see \cite{sparse_tensor_pca}) against restricted computational model  of low degree polynomials, showing that in this model such algorithms do not exist.

In the context of spare PCA, \cref{theorem:metaTheorem} provides guarantees matching those of known quasipolynomial time algorithms, \textit{but} also works with the heavy tailed noise of \cref{model:obliviousSparsePCA} (e.g., standard Cauchy noise):

\begin{theorem}[Robust Sparse PCA]\label{theorem:sparsePca}
	There exists an absolute constant $C>1$ such that
	if $k\ge C \cdot \ln(n)/\alpha^2$, $\lambda \ge k$ and $\Normm{v}\leq 100/\sqrt{k}$, then there exists an algorithm running in time $n^{O(\log(n)/\alpha^2)}$ that, given $\bm{Y}$ as in \cref{model:obliviousSparsePCA}, outputs a unit vector $\hat{\bm{v}}$ satisfying
	\begin{align*}
	\abs{\iprod{v,\hat{\bm{v}}}}  \ge 0.99
	\end{align*}
	with high probability.
	
	Moreover, for $\eps\le \frac{\alpha^2 k^2}{Cn^2\ln n}$, the same result holds if an arbitrary (adversarially chosen) $\epsilon$-fraction of entries of $\bm{Y}$ is replaced by adversarially chosen values.
\end{theorem}

%For example, consider the regime $\omega(\log n) \le k\le n^{0.49}$, $\normi{v}\le 100/\sqrt{k}$, $k \le \lambda\le O(k)$. As we discussed before, in this regime no algorithm is known to recover $v$ in time $n^{o(\log n)}$  in the case when the noise is standard Gaussian. Our algorithm can recover $v$ in time $n^{O(\log n)}$ not only under standard Gaussian noise, but also under heavy tailed symmetric noise, e.g. standard Cauchy noise.
A natural question to ask concerning \cref{theorem:sparsePca} is whether one could hope to obtain non-trivial guarantees in polynomial time.
In \cref{section:lowerboundSparsePca} we provide evidence that the quasipolynomial time requirement for the noise model in \cref{model:obliviousSparsePCA} might be inherent (and thus the running time of \cref{theorem:mainSparse} is nearly optimal) via a reduction from the Planted Clique problem. As in the context of tensor PCA, it is an interesting open question whether for specific heavy-tailed distributions (e.g., Cauchy) one can design polynomial time algorithms recovering the signal $v$ (for not very large $\lambda$, say, $\lambda = k\polylog n$).

Finally, we remark that the number of adversarial corruptions that the algorithm can handle is nearly optimal: If the adversary that can change $\eps = k^2/n^2$ fraction of the entries then all information about the signal may be removed.

\paragraph{Comparison with other results for symmetric unbounded noise models}

%\begin{remark}[Comparison with other results for oblivious noise models]
	Various other estimation problems in the presence of symmetric unbounded noise have been studied, such as linear regression, sparse regression and principal component analysis.
	We remark that our framework can be used to recover the best previously known results for these models \cite{ ICML-linear-regression, DBLP:conf-nips-dOrsiLNNST21}. We point out however that compared to these algorithms, \cref{theorem:metaTheorem} provides a slow rate of error convergence. That is, when those algorithms guarantee an error bound $O(\eps)$, \cref{theorem:metaTheorem} provides a bound $O(\sqrt{\eps})$.
	This phenomenon is a consequence of the decomposability of particular regularizers used in previous works. 
	Our framework does not require a decomposable regularizer and can thus deal with signal sets $\Omega$ that may be significantly more challenging than the $\ell_1$-ball and nuclear norm ball considered in other works. 
	We provide a more detailed discussion in \cref{sec:decomposability}. % \Tnote{\todo discussion about decomposability}. 

\section{Techniques}
\label{sec:techniques}

% We provide here an overview of our framework.

Let \(\Omega\subseteq \R^m\) be a set of structured signals we wish to recover (e.g., a sparse rank-1 matrix or a rank-1 tensor).
Let \(\bm N\) be an \(m\)-dimensional random noise with independent, symmetrically distributed entries such that \(\min_{i\in [m]}\Pr\set{\abs{\bm{N}_i}\leq 1} \ge \alpha\).
Given (a realization of) a random vector \(\bm Y = X^* + \bm N\) for some unknown signal \(X^*\in\Omega\), our task is to approximately recover the signal \(X^*\).

A common approach for this task is to minimize a loss function \(L(X-\bm Y)\) over \(X\in \Omega\).
In the special case of Gaussian noise, this approach recovers the maximum likelihood estimator if we choose the least-squares loss function \(L(X-\bm Y)=\norm{X-\bm Y}^2_2\).
However, a well known weakness of this estimator is that it is extremely susceptible to outliers, thus it cannot be used with
noise distributions with diverging moments. 
In contrast, an estimator that has been observed (both in practice and  theory) to be significantly more robust to outliers is the \emph{Huber loss function} $F_h(Z):= \sum_{i\in [m]}f_h(Z_i)$ where \(f_h\) is the following \emph{Huber penalty},
\begin{align}\label{eq:huber-penalty}
	f_h(Z_i):=
  \begin{cases}
    \frac{1}{2}Z_i^2&\text{for }\abs{Z_i}\leq h\,,\\
    h\cdot (\abs{Z_i}-\frac{h}{2})& \text{otherwise}\,.
  \end{cases}
\end{align}
Here, \(h>0\) is a parameter of the estimator to be determined later.

From a computational perspective, the problem is that for many (perhaps most) signal sets \(\Omega\) one may be interested in, this kind of loss minimization turns out to be NP-hard (regardless of the concrete choice of the loss function).
Therefore, we can only  expect to solve specific relaxations of this optimization problem.

Previous work \cite{DBLP:conf-nips-dOrsiLNNST21} considered these kinds of relaxations, but could only obtain meaningful error guarantees for sets  \(\Omega\) that admit convex regularizers with a certain decomposability property.
Unfortunately, only few regularizers with this property are known (e.g., the \(\ell_1\)-norm for vectors and the nuclear norm for matrices) and so this limitation turned out to be a fundamental obstacle to the application of this framework to many estimation problems.

Our machinery overcomes this obstacle, extending the approach in \cite{DBLP:conf-nips-dOrsiLNNST21} to a significantly broader set of choices for \(\Omega\).
Concretely, we can consider all choices of \(\Omega\) such that a natural family of convex relaxations --namely the sum-of-squares hierarchy-- succeeds in recovering the signal from Gaussian noise.

\paragraph{Tensor PCA as a running example}

%In order to illustrate our techniques we consider the signal set \(\Omega=\set{\lambda \cdot x^{\otimes 3} \mid x\in \R^{n},\ \norm{x}=1 }\) with 
%Here, \(\lambda\in [0,n^{3/2}]\) is a parameter to be determined later.
%We are interested in the smallest choice of \(\lambda\) such that, given \(\bm Y= X^*+\bm N\), we can recover an estimate \(\bm \hat{X}\) that with high probability achieves correlation at least 0.99 with \(X^*\).

In order to illustrate our techniques, we consider the following example.
Let $x\in \R^n$ be a unit vector and let $0 < \lambda \le n^{3/2}$.
For simplicity of the exposition, we assume here that $x$ has entries from $\Set{\pm{1/\sqrt{n}}}$.
We would like to recover a tensor 
$X^* = \lambda x^{\otimes 3}$ from ${\bm Y} =  X^* + \bm N$, 
and determine how small $\lambda$ can be so that the recovery of $X^*$ is possible. Notice that in these settings the signal set is \(\Omega=\set{\lambda \cdot x^{\otimes 3} \mid x\in \Set{\pm 1/\sqrt{n}}^n,\ \norm{x}=1 }\,.\) This set is non-convex --in fact the problem is NP-hard in general-- but let us temporarily disregard computational efficiency. Suppose we optimize the Huber loss with parameter $h= 3$ over this set $\Omega$ of rank-1 tensors.
%that satisfy $\normi{X} \le 1$ and $\normt{X} \le \lambda$.

Let $\hat{\bm X}\in \Omega$ be a minimizer, and denote $\bm \Delta = X^* - \hat{\bm X}$. A common approach is to apply Taylor's theorem and obtain
\begin{equation}\label{eq:optimization_inequality}
F_h\Paren{\bm Y - X^*} = F_h(\bm N) \ge F_h\Paren{\bm Y - \hat{\bm X}} \ge 
F_h\Paren{\bm N} + \Iprod{\nabla F_h\Paren{\bm N}, \bm \Delta} + 
\frac{1}{2} \bm\kappa(\bm \Delta)\,,
\end{equation}
where $\bm\kappa(\bm \Delta)$ is some lower bound on the values $\bm \Delta^\top H(X) \bm \Delta$ for all $X$ from the segment between $X^*$ and $\hat{\bm X}$, where $H(X)$ is the Hessian\footnote{The second derivative of the Huber penalty does not exit at the points \(\set{\pm h}\).
	However, the indicator function $I_h$ of the interval $[-h,h]$ is the second derivative of Huber penalty in $L_1$ sense, that is \(f_h'[b]-f_h'[a]=\int_{a}^b I_h(t)\,\mathrm d t\) for all \(a,b\in\R\).
	So by the Hessian at point $X$ we mean a quadratic form whose matrix in the standard basis is diagonal with (diagonal) entries $I_h(X_i)$.}
of the Huber loss at point $X$. 
It is not hard to see 
(see \cref{lem:second-order-behavior-huber}) that one can choose
\[
\bm\kappa(\bm \Delta) =
\sum_{i=1}^m\ind{\abs{\bm N_i} \le 1}\cdot \ind{\abs{\bm \Delta_i} \le h - 1} \bm \Delta_i^2 = 
\sum_{i=1}^m \ind{\abs{\bm N_i} \le 1} \bm \Delta_i^2\,.
\]

Now it is clear that if we can show  $\Abs{\Iprod{\nabla F_h\Paren{\bm N}, \bm \Delta}}\le \gamma(\bm \Delta)$ for some $\gamma(\bm \Delta)$  and 
$\bm\kappa(\bm \Delta) \ge 0.9\cdot \alpha \cdot \normt{\bm \Delta}^2$, \cref{eq:optimization_inequality} immediately implies the bound 
\begin{equation}\label{eq:abstract_error_bound}
\norm{\bm \Delta}^2_2 < 3 \gamma(\bm \Delta)/\alpha\,.
\end{equation}
That is, the estimator guarantee depends only on an \textit{upper bound on the gradient} and a \textit{lower bound on the curvature} of the space in the direction of $\bm \Delta$.

%\Tnote{Double check the scaling!}
Let us first obtain the bound 
\[
\bm\kappa(\bm \Delta) = \sum_{i=1}^m \ind{\abs{\bm N_i} \le 1} \bm \Delta_i^2 \ge 0.9\cdot \alpha\cdot \normt{\Delta}^2\,.
\]
For simplicity assume $\normt{\bm \Delta} = \tau$ for some\footnote{Estimation error $\tau $ cannot be $n^{-\omega(1)}$ in our parameter regime.} $\tau \ge n^{-O(1)}$.
A successful strategy here  is to derive a lower bound on $\bm \kappa(\Delta)$ for a fixed $\Delta$, and then construct an $\eps$-net over 
$\Omega' = \Set{X-X' : X, X'\in \Omega\,, \normt{X-X'}= \tau}$.
The idea is that \textit{if} our lower bound holds with sufficiently large probability and \textit{if} the size of the covering is not too large, then we will be able to show the desired curvature in all the possible directions of $\bm \Delta$.
Now for fixed $\Delta$, the expected value of $\bm \kappa(\Delta)$ is $\alpha \normt{\Delta}^2$, 
and by Hoeffding's inequality, the deviation from the mean is bounded by $O\Paren{\norm{\Delta}_4^2 \sqrt{\log(1/\delta)}}$ with probability at least $1-\delta$. Since $\normi{\Delta} \le 2$, $\norm{\Delta}_4^2\le 2\normt{\Delta}$.
Thus we have
\[
\bm \kappa( \Delta) \ge \alpha \normt{\Delta}^2 - \normt{ \Delta} \cdot O\Paren{\sqrt{\log(1/\delta)}} = \tau\Paren{\alpha\tau - O\Paren{\sqrt{\log(1/\delta)}}}\,,
\]
which is close to its expectation when $\tau\gtrsim \sqrt{\log 1/\delta}/\alpha\,.$

We need now to extend  this bound to all possible directions of $\bm \Delta$. To this end note that if $\Delta,\Delta'\in \Omega'$ are $\eps$-close to each other for some small enough
$\eps = n^{-O(1)}$, then\footnote{Here we assume that $\alpha > 1/n$, otherwise the problem is information theoretically intractable.}
\[
\Abs{\bm\kappa(\Delta) - \bm\kappa(\Delta')} \lesssim \alpha \tau^2\,.
\]
So it remains to show a cover of $\Omega$. Notice that the size of the cover determines in a very strong way the quality of the error guarantees of the estimator. For example one could try to use the  $\eps$-net covering the unit ball in $\R^{n^3}$, this does not exploit the structure of $\Omega$ and  has thus size $\Paren{O\Paren{1/\eps}}^{n^3}$. By the above calculations, with this $\eps$-net we could provide a meaningful lower bound only when $\tau\gtrsim n^{3/2}/\alpha$. In other words, our error estimate would be worse than the trivial estimator outputting the zero tensor! 
To obtain a tighter covering, recall $\Omega$ is a subset of the set of rank one  tensors of norm $\lambda \le n^{3/2}$. The size of minimal $\eps$-net in $\Omega$ is at most $O\Paren{n^{O(1)}/\eps}^n$ (since the mapping $x \mapsto x^{\otimes 3}$ is $n^{O(1)}$-Lipschitz for $\normt{x}\le n^{O(1)}$). 
Hence the size of the $\eps$-net in $\Omega'$ is bounded by $n^{O(n)}$, and by union bound we get
\[
\bm \kappa(\bm \Delta) \ge \alpha \normt{\bm \Delta}^2 - \normt{\bm \Delta} \cdot O\Paren{\sqrt{n\log n}} = \alpha \tau^2 - O\Paren{\tau \sqrt{n\log n}}\,,
\]
with high probability. 
Hence for $\tau \gtrsim \frac{\sqrt{n\log n}}{\alpha}$ we get the desired bound.

We can now focus on bounding the gradient $\Abs{\Iprod{\nabla F_h\Paren{\bm N}, \bm \Delta}}$. The choice of the Huber loss function makes this very easy:
$\nabla F_h\Paren{\bm N}$ is a random vector with symmetric independent entries bounded by $h = O(1)$ in absolute value, so for fixed $\Delta$, 
$\Abs{\Iprod{\nabla F_h\Paren{\bm N}, \Delta}}$ is bounded by
$O\Paren{\normt{\Delta}\sqrt{\log(1/\delta)}}$ with probability $1-\delta$. By union bound over the $\eps$-net in $\Omega'$, with high probability
\[
\Abs{\Iprod{\nabla F_h\Paren{\bm N}, \bm \Delta}}\le \normt{\bm \Delta}\cdot O\Paren{\sqrt{n\log n}}\,.
\]

Hence by \cref{eq:abstract_error_bound}, we can conclude that with high probability
\[
\normt{\bm \Delta} \le O\Paren{\frac{\sqrt{n\log n}}{\alpha}}\,.
\]
Therefore, the minimizer $\hat{\bm X}$ of this inefficient estimator is highly correlated with $X^*$ as long as $\lambda \gtrsim \sqrt{n\log n}/\alpha$.
This bound is nearly optimal: if $\bm N$ has iid Gaussian entries with standard deviation $\Theta\Paren{\alpha}$, it is information-theoretically impossible to recover $X^*$ if $\lambda \le o\Paren{\sqrt{n}/\alpha}$ (see \cite{PWB20}).

\paragraph{Tensor PCA as a running example: efficient estimation}
We take now into account the computational complexity of computing the desired estimator.
To have a loss function we can minimize efficiently, the idea is  to replace the set of rank-1 tensors $\Omega$ 
by some set $\tilde{\Omega} \supset \Omega$ over which we can efficiently optimize.
We cannot do this via the framework in \cite{DBLP:conf-nips-dOrsiLNNST21} since no appropriate decomposable regularizer is known for high-order tensors. Thus we use instead sum-of-squares relaxations, 
and take $\tilde{\Omega} = \tilde{\Omega}_t$ to be the set of 
pseudo-expectations of degree $t$ that satisfy certain constraints.
Crucially, in order to apply the  argument of the previous paragraph, 
we need a tight upper bound on the covering number of the set of pseudo-expectations $\tilde{\Omega}_t$.

In the exponential time algorithm described above we had a natural $n^{O(1)}$-Lipschitz  mapping from $n$-dimensional space to $n^3$-dimensional space, which allowed us to construct such a covering.
In the case of pseudo-expectations, we do not have such a mapping, so different techniques are required to to get a bound on the size of $\eps$-net.

We use  Sudakov minoration: For every bounded set $A\subset \R^m$, the size of the minimal $\eps$-net of $A$ is bounded by $\exp\Paren{O\Paren{\cG(A)^2/\eps^2}}$, where
\[
\cG(A) = \E_{\bm w \sim N(0, \Id_m)}\Brac{\sup_{a\in A} \sum_{i=1}^m a_i\bm w_i}\,.
\]
The quantity $\cG(A)$ is called the \emph{Gaussian complexity} of the set $A$. 
So in order to bound the size of optimal $\eps$-net of the set $\tilde{\Omega}$ of pseudo-expectations it is enough to bound its Gaussian complexity. 
The good news is that we can bound  the Gaussian complexity of the set of pseudo-expectations by certifying in sum-of-squares a bound on the Gaussian complexity of the set $\Omega$ of rank-1 tensors!
Concretely, $\Omega$ can be defined by polynomial constraints with variables $X\in \R^{n^3}$ and auxiliary variables $x\in \R^n$: 
\[
\cS_{X,x} = \Set{X = \lambda x^{\otimes 3},\quad \normt{x}^2= 1, \quad \forall i\in[n],\;\; x_i^{2} \le 1/n}\,.
\] 
If we can show that with high probability\footnote{For Gaussian distribution it is not hard to obtain from this a bound on expectation since we have good tail bounds for it.} over the tensors $\bm W$ with iid Gaussian entries there exists a degree $t$ sum-of-square proof that these constraints imply
\[
\sum_{1\le i\le j\le k\le n} x_i x_j x_k \bm W_{ijk} \le \gamma_{t}\,,
\]
then we can conclude that $\cG\Paren{\tilde{\Omega}_t}\le O\Paren{\lambda \gamma_{t}}$. 

In \cite{tensor_pca_sos, HopkinsKPRSS17} it was shown that there exists a $4$-degree sum-of-squares proof that $\cS_{X,x}$ imply the inequality
\[
\sum_{1\le i\le j\le k\le n} x_i x_j x_k \bm W_{ijk} \le 
O\Paren{\ln(n)}^{1/4} \cdot n^{3/4}\,.
\]
Hence, $\cG\Paren{\tilde{\Omega}_4}\le \tilde{O}\Paren{\lambda n^{3/4}}$ as desired. 

Note that the analysis of the exponential time algorithm does not work here because the dependence of the size of 
$\eps$-net on $\eps$ in Sudakov's minoration is exponential and not polynomial as in the case of $\ell_2$-ball.
However, it turns out that via a more careful analysis we can show
\[
\normt{\bm \Delta}^2 \le \tilde{O}\Paren{\frac{\lambda n^{3/4}}{\alpha}}\,.
\]

This bound implies that $\hat{\bm X}$ is highly correlated with $X^*$ as long as $\lambda \gtrsim (\log n)^{1/4} n^{3/4}/\alpha$, which
 matches (up to a logarithmic factor) the current best known guarantees for polynomial time algorithms 
 when $\bm N$  has \iid Gaussian entries with standard deviation
 $\Theta\Paren{\alpha}$, but also works with significantly more general noise (e.g., Cauchy noise at scale $\Theta\Paren{\alpha}$).

 %\subsection{Adversarial corruptions}\label{sec:techniques-corruptions}
\paragraph{Recovery in the presence of adversarial corruptions and oblivious noise} 
Our framework is robust to additional adversarial corruptions resulting from an adversary corrupting an $\eps$-fraction of the entries of $\bm Y$.
In light of our discussion so far, to show this it suffices to check how do the values $\Abs{\Iprod{\nabla F_h\Paren{\bm N}, \bm \Delta}}$
 and 
 $\bm\kappa(\bm \Delta) =  \sum_{i=1}^n \ind{\abs{\bm N_i} \le 1} \bm \Delta_i^2 $ change in the presence of corruptions. For simplicity, we limit our discussion to the first inefficient estimator introduced in previous paragraphs. %will show how  the analysis \Gnote{refer to exptime analysis} changes if the adversary corrupts $\eps$-fraction of the entries of $\bm Y$.
 
 First assume that the adversary corrupts a set of entries of size $\eps n^3$ that is random (not adversarially chosen). 
 In this case $\Abs{\Iprod{\nabla F_h\Paren{\bm N}, \bm \Delta}}$ can only be increased by an additive term
 \[
 h\cdot 2\lambda\normi{v}^3 \cdot n^3\eps \le O\Paren{n^{3/2} \lambda \eps}\,,
 \]
 since the entries of $\nabla F_h\Paren{\bm N}$ are bounded by $h$, and the entries of $\bm \Delta$ are bounded by $2\lambda\normi{v}^3$.
 The value $\bm\kappa(\bm \Delta)$ also does not change significantly if a small random set of entries is corrupted.
 Hence in this case, if $n^{3/2} \lambda \eps \le \normt{\bm \Delta} \sqrt{n\log n}$, the error does not increase in any significant way.
 Note that in the regime $\normt{\bm \Delta} \ge \Omega(\lambda)$ (when we can still have $0.99$ correlation with the signal), the number of corruptions $\eps n^3$ is allowed to be up to $n^2$. %Note that for vectors $v$ with larger entries than $1/\sqrt{n}$ our bound on the number of corruptions becomes worse.
 
In the general case, when the adversary is allowed to choose the corrupted set, we need to use a union bound over all sets of size $\eps n^3$ (we use it to bound both $\Abs{\Iprod{\nabla F_h\Paren{\bm N}, \bm \Delta}}$
 and  $\bm\kappa(\bm \Delta)$). 
Here, the gradient bound becomes
 \[
 \Abs{\Iprod{\nabla F_h\Paren{\bm N}, \bm \Delta}} \le \normt{\Delta}\cdot O\Paren{\sqrt{n\log n} + \sqrt{\eps n^3 \log n}}\,.
 \]
 Hence the number of corruptions is only allowed to be at most $n$. Observe that \cref{theorem:tensorPca} is robust up to $\tilde{\Omega}\Paren{n^{3/2}}$ corrupted entries. This is not surprising. The reason is that the algorithm requires signal strength $\lambda = \tilde{\Omega}\Paren{n^{3/4}}$ compared to $\lambda = \tilde{\Omega}\Paren{n^{1/2}}$ that is required by the exponential time algorithm.

% \subsection{Sparse PCA}\label{sec:techniques-spca}
\paragraph{Sparse PCA}
As a second example of the applications of \cref{theorem:metaTheorem} consider the sparse PCA problem: We are given 
$\bm Y = \lambda \cdot vv^\top + \bm N$, where $v \in \R^n$ is a $k$-sparse vector, and the goal is to recover $v$. 
For simplicity we assume here that $v$ is flat, i.e., that its non-zero entries are in $\set{\pm 1/\sqrt{k}}$.
 
In order to use our framework, we need to certify in sum-of-squares an upper bound on the Gaussian complexity of the set of sparse vectors. 
 So we need to show that for some (as small as possible) $\gamma$, with high probability over matrices $\bm W$ with \iid Gaussian entries there exists a (not very high degree) sum-of-squares proof that some system of constraints $\cC$ defining sparse vectors implies
 \[
 \sum_{1\le i,j\le n} x_ix_j \bm W_{ij}\le \gamma\,,
 \]
 where $x$ are variables that satisfy sparsity constraints of $\cC$.
 
We use the system of constraints $\cC_t$ (the subscript $t \in \N$ indicates that the constraints involve degree $t$ polynomials) from \cite{DBLP:conf-focs-dOrsiKNS20} (see \cref{section:obliviousSparsePca} for a precise definition). 
The authors in \cite{DBLP:conf-focs-dOrsiKNS20} used the program for a different sparse PCA model, but it is possible to adapt their proof and show that with high probability there exist a degree $O(t)$ sum-of-squares proof that $\cC_t$ implies the inequality
\[
\sum_{1\le i,j\le n} x_ix_j \bm W_{ij}\le O\Paren{k\sqrt{\frac{\log n}{t}}}\,.
\]

Hence if $\lambda = k$ and $t \gtrsim \log(n)/\alpha^2$, 
then \cref{theorem:metaTheorem} implies that the Huber loss minimizer has $0.999$ correlation with $vv^\top$ 
(and hence its top eigenvector has correlation $0.99$ with $v$ or $-v$).

The running time is $n^{O(t)} = n^{O(\log(n)/\alpha^2)}$, and it is likely to be inherent: For $\alpha = 1$ 
we can reduce the planted clique problem (with clique size $k$)  
to the problem of recovering $v$ from the upper triangle (without the diagonal) of matrix $\bm Y$. 
The best currently known algorithmic guarantees for sparse PCA are captured by algorithms that can recover $v$ from the upper triangle of the input matrix, hence it is likely that sparse PCA with symmetric noise is at least as hard as the planted clique problem.  Finally, we remark that there is a conjecture stating that there is no $n^{o(\log n)}$-time algorithm that can solve the planted clique problem for some values of $k$ (see \cite{planted_clique_conjecture}).

The reduction works as follows. We use the notation $\cU(M)$ to denote the upper triangle of matrix $M$. 
It is not hard to see that if $\bm A$ is an instance of the planted clique problem (the adjacency matrix of the graph) and $J$ is the matrix with all entries equal to one, then $\cU(2\bm A - J)$ is the upper triangle of an instance of the sparse PCA problem with symmetric noise, 
where $\lambda = k$, $\sqrt{k}\cdot v$ is the 0/1 indicator of the clique, and the noise $\bm N$ is as follows: 
For the entries $i,j\in \supp\paren{v}$, $\bm N_{ij} = 0$, 
and for other entries $\bm N_{ij}$ are iid sampled from the uniform distribution on $\Set{\pm 1}$.

\section{Preliminaries}
\label{sec:preliminaries}

\paragraph{Notation}
We use boldface to denote random variables.
We hide absolute constant multiplicative factors using the standard notations $O(\cdot), \Omega(\cdot), \gtrsim, \lesssim$. 
Similarly, we hide multiplicative logarithmic (in the dimension $m$ of the input) factors using the notation $\tilde{O}(\cdot), \tilde{\Omega}(\cdot)$.  We use the notation
$\Normt{\cdot}$ for the Euclidean norm, $\Normf{\cdot}$ for the Frobenius norm,
$\Normi{X}=\max_{i \in[m]}\Abs{X_i}$. We write $\log$ for the logarithm to the base $e$. 

%We introduce here related background notions.

%We uselowercase letters to denote vectors,  capital letters to denote matrices and  callygraphic letters to denote tensors, e.g.: $\cT\in(\R^n)^{\otimes 3}$ denote a third-order tensor. 
%We also use calligraphic letters to denote sets, context will clarify ambiguities. To ease the reading we might at times deviate from this notation, this will be clear from context.

\begin{definition}[Huber loss function]\label{def:huber-loss-function-copy}
	The Huber loss penalty is defined as:
	\begin{align}\label{eq:huber-penalty-copy}
		f_h(t):=\begin{cases}
			\frac{1}{2}t^2&\text{for }\abs{t}\leq h\,,\\
			h(\abs{t}-\frac{h}{2})& \text{otherwise}.
		\end{cases}
	\end{align}
	
	For a vector $x\in \R^n$ we denote by $F_h(x):= \sum_{i\in [n]}f_h(x_i)$.
\end{definition}

The Huber loss satisfies the following inequality.

\begin{lemma}\label[lemma]{lem:second-order-behavior-huber}
	Let $h > 0$.
	For all $t, \delta\in \R$, and all $0\le \zeta\le h$,
	\begin{align}
		f_h(t+\delta)-f_h(t)-f'_h(t)\cdot \delta\geq \frac{\delta^2}{2}\ind{\abs{t}\leq \zeta}\cdot \ind{\Abs{\delta}\leq h-\zeta}\,.
	\end{align}
\end{lemma}
\begin{proof}
	We have two cases:
	
	\begin{itemize}
		\item  If $|t|>\zeta$ or $|\delta|>h-\zeta$, then either $\ind{\abs{t}\leq \zeta}=0$ or $
		\ind{\Abs{\delta}\leq h-\zeta}=0$. Hence,
		\begin{equation}
			\ind{\abs{t}\leq \zeta}\cdot\ind{\Abs{\delta}\leq h-\zeta}=0\,.
		\end{equation}
		For this case, we simply use the convexity of $f_h$ to get
		\begin{align}
			f_h(t+\delta)-f_h(t)-f'_h(t)\cdot \delta&\geq 0\\
			&= \frac{\delta^2}{2}\ind{\abs{t}\leq \zeta}\cdot \ind{\Abs{\delta}\leq h-\zeta}\,.
		\end{align}
		\item If $|t|\leq \zeta$ and $|\delta|\leq h-\zeta$, then $|t+\delta|\leq h$. In this case, we have $f_h(t+\delta)=\frac{1}{2}(t+\delta)^2$, $f_h(t)=\frac{1}{2}t^2$ and $f_h'(t)=t$. By direct inspection, we get
		\begin{align}
			f_h(t+\delta)-f_h(t)-f'_h(t)\cdot \delta&=\frac{1}{2}(t+\delta)^2-\frac{1}{2}t^2-t\delta=\frac{\delta^2}{2}\\
			&=\frac{\delta^2}{2}\ind{\abs{t}\leq \zeta}\cdot \ind{\Abs{\delta}\leq h-\zeta}\,.
		\end{align}
	\end{itemize}
\end{proof}

%\paragraph{Subspace distortion} We use $\distrtf{\Omega}$ to denote the \todo

%\paragraph{Tensors}
%For tensors $S,T\in (\R^n)^{\otimes 3}$ we define the operation
%\begin{align*}
%	\iprod{{S}\,, {T}} = \sum_{i,j,k\in[n]}{S}_{i,j,k}\cdot {T}_{i,j,k}\,.
%\end{align*}
%The injective norm of a tensor is defined as
%\begin{align*}
%	\Normin{{T}}:=\max_{\substack{x,y,z\in\R^n:\\\norm{x}=\norm{y}=\norm{z}=1}}\iprod{{T}, x\tensor y \tensor z}
%\end{align*}

%and its degree $t$-sum-of-squares relaxation as
%\begin{align*}
%	\Normins{{T}}{t} := \max_{\text{level-$t$ } \mu(x,y,z)} \pE \iprod{{T}, x\tensor y \tensor z} 	\quad\text{subj. to } \Snorm{x}=\Snorm{y}=\Snorm{z}=1\,,
%\end{align*}
%The corresponding dual norm is then
%\begin{align*}
%	\Normns{{T}}{t}:=\max_{\substack{A\in (\R^n)^{\otimes %3}:\\\Normins{{A}}{t}\leq 1}}\iprod{{T}, A}\,.
%\end{align*}
%We will refer to this norm as the sum-of-squares tensor nuclear norm.

\subsection{Sum of squares and pseudodistributions}

Let $x = (x_1, x_2, \ldots, x_n)$ be a tuple of $n$ indeterminates and let $\R[x]$ be the set of polynomials with real coefficients and indeterminates $x_1,\ldots,x_n$.
We say that a polynomial $p\in \R[x]$ is a \emph{sum-of-squares (sos)} if there are polynomials $q_1,\ldots,q_r$ such that $p=q_1^2 + \cdots + q_r^2$.

\subsection{Pseudo-distributions}

Pseudo-distributions are generalizations of probability distributions.
We can represent a discrete (i.e., finitely supported) probability distribution over $\R^n$ by its probability mass function $D\from \R^n \to \R$ such that $D \geq 0$ and $\sum_{x \in \mathrm{supp}(D)} D(x) = 1$.
Similarly, we can describe a pseudo-distribution by its mass function.
Here, we relax the constraint $D\ge 0$ and only require that $D$ passes certain low-degree non-negativity tests.

Concretely, a \emph{level-$\ell$ pseudo-distribution} is a finitely-supported function $D:\R^n \rightarrow \R$ such that $\sum_{x} D(x) = 1$ and $\sum_{x} D(x) f(x)^2 \geq 0$ for every polynomial $f$ of degree at most $\ell/2$.
(Here, the summations are over the support of $D$.)
A straightforward polynomial-interpolation argument shows that every level-$\infty$-pseudo distribution satisfies $D\ge 0$ and is thus an actual probability distribution.
We define the \emph{pseudo-expectation} of a function $f$ on $\R^d$ with respect to a pseudo-distribution $D$, denoted $\pE_{D(x)} f(x)$, as
\begin{equation}
	\pE_{D(x)} f(x) = \sum_{x} D(x) f(x) \,\mper
\end{equation}
The degree-$\ell$ moment tensor of a pseudo-distribution $D$ is the tensor $\E_{D(x)} (1,x_1, x_2,\ldots, x_n)^{\otimes \ell}$.
In particular, the moment tensor has an entry corresponding to the pseudo-expectation of all monomials of degree at most $\ell$ in $x$.
The set of all degree-$\ell$ moment tensors of probability distribution is a convex set.
Similarly, the set of all degree-$\ell$ moment tensors of degree $d$ pseudo-distributions is also convex.
Key to the algorithmic utility of pseudo-distributions is the fact that while there can be no efficient separation oracle for the convex set of all degree-$\ell$ moment tensors of an actual probability distribution, there's a separation oracle running in time $n^{O(\ell)}$ for the convex set of the degree-$\ell$ moment tensors of all level-$\ell$ pseudodistributions.

\begin{fact}[\cite{MR939596-Shor87,parrilo2000structured,MR1748764-Nesterov00,MR1846160-Lasserre01}]
	\label[fact]{fact:sos-separation-efficient}
	For any $n,\ell \in \N$, the following set has a $n^{O(\ell)}$-time weak separation oracle (in the sense of \cite{MR625550-Grotschel81}):
	\begin{equation}
		\Set{ \pE_{D(x)} (1,x_1, x_2, \ldots, x_n)^{\otimes d} \mid \text{ degree-d pseudo-distribution $D$ over $\R^n$}}\,\mper
	\end{equation}
\end{fact}
This fact, together with the equivalence of weak separation and optimization \cite{MR625550-Grotschel81} allows us to efficiently optimize over pseudo-distributions (approximately)---this algorithm is referred to as the sum-of-squares algorithm.

The \emph{level-$\ell$ sum-of-squares algorithm} optimizes over the space of all level-$\ell$ pseudo-distributions that satisfy a given set of polynomial constraints---we formally define this next.

\begin{definition}[Constrained pseudo-distributions]
	Let $D$ be a level-$\ell$ pseudo-distribution over $\R^n$.
	Let $\cA = \{f_1\ge 0, f_2\ge 0, \ldots, f_m\ge 0\}$ be a system of $m$ polynomial inequality constraints.
	We say that \emph{$D$ satisfies the system of constraints $\cA$ at degree $r$}, denoted $D \sdtstile{r}{} \cA$, if for every $S\subseteq[m]$ and every sum-of-squares polynomial $h$ with $\deg h + \sum_{i\in S} \max\set{\deg f_i,r}\leq \ell$,
	\begin{displaymath}
		\pE_{D} h \cdot \prod _{i\in S}f_i  \ge 0\,.
	\end{displaymath}
	We write $D \sdtstile{}{} \cA$ (without specifying the degree) if $D \sdtstile{0}{} \cA$ holds.
	Furthermore, we say that $D\sdtstile{r}{}\cA$ holds \emph{approximately} if the above inequalities are satisfied up to an error of $2^{-n^\ell}\cdot \norm{h}\cdot\prod_{i\in S}\norm{f_i}$, where $\norm{\cdot}$ denotes the Euclidean norm\footnote{The choice of norm is not important here because the factor $2^{-n^\ell}$ swamps the effects of choosing another norm.} of the cofficients of a polynomial in the monomial basis.
\end{definition}

We remark that if $D$ is an actual (discrete) probability distribution, then we have  $D\sdtstile{}{}\cA$ if and only if $D$ is supported on solutions to the constraints $\cA$.

We say that a system $\cA$ of polynomial constraints is \emph{explicitly bounded} if it contains a constraint of the form $\{ \|x\|^2 \leq M\}$.
The following fact is a consequence of \cref{fact:sos-separation-efficient} and \cite{MR625550-Grotschel81},

\begin{fact}[Efficient Optimization over Pseudo-distributions]
	There exists an $(n+ m)^{O(\ell)} $-time algorithm that, given any explicitly bounded and satisfiable system\footnote{Here, we assume that the bitcomplexity of the constraints in $\cA$ is $(n+m)^{O(1)}$.} $\cA$ of $m$ polynomial constraints in $n$ variables, outputs a level-$\ell$ pseudo-distribution that satisfies $\cA$ approximately. 
\end{fact}

\subsection{Sum-of-squares proofs}

Let $f_1, f_2, \ldots, f_r$ and $g$ be multivariate polynomials in $x$.
A \emph{sum-of-squares proof} that the constraints $\{f_1 \geq 0, \ldots, f_m \geq 0\}$ imply the constraint $\{g \geq 0\}$ consists of  sum-of-squares polynomials $(p_S)_{S \subseteq [m]}$ such that
\begin{equation}
	g = \sum_{S \subseteq [m]} p_S \cdot \prod_{i \in S} f_i
	\mper
\end{equation}
We say that this proof has \emph{degree $\ell$} if for every set $S \subseteq [m]$, the polynomial $p_S \prod_{i \in S} f_i$ has degree at most $\ell$.
If there is a degree $\ell$ SoS proof that $\{f_i \geq 0 \mid i \leq r\}$ implies $\{g \geq 0\}$, we write:
\begin{equation}
	\{f_i \geq 0 \mid i \leq r\} \sststile{\ell}{}\{g \geq 0\}
	\mper
\end{equation}

Sum-of-squares proofs satisfy the following inference rules.
For all polynomials $f,g\colon\R^n \to \R$ and for all functions $F\colon \R^n \to \R^m$, $G\colon \R^n \to \R^k$, $H\colon \R^{p} \to \R^n$ such that each of the coordinates of the outputs are polynomials of the inputs, we have:

\begin{align}
	&\frac{\cA \sststile{\ell}{} \{f \geq 0, g \geq 0 \} } {\cA \sststile{\ell}{} \{f + g \geq 0\}}, \frac{\cA \sststile{\ell}{} \{f \geq 0\}\,, \cA \sststile{\ell'}{} \{g \geq 0\}} {\cA \sststile{\ell+\ell'}{} \{f \cdot g \geq 0\}}\,, \tag{addition and multiplication}\\
	&\frac{\cA \sststile{\ell}{} \cB, \cB \sststile{\ell'}{} C}{\cA \sststile{\ell \cdot \ell'}{} C}\,,  \tag{transitivity}\\
	&\frac{\{F \geq 0\} \sststile{\ell}{} \{G \geq 0\}}{\{F(H) \geq 0\} \sststile{\ell \cdot \deg(H)} {} \{G(H) \geq 0\}} \tag{substitution}\mper
\end{align}

Low-degree sum-of-squares proofs are sound and complete if we take low-level pseudo-distributions as models.

Concretely, sum-of-squares proofs allow us to deduce properties of pseudo-distributions that satisfy some constraints.

\begin{fact}[Soundness]
	\label{fact:sos-soundness}
	If $D \sdtstile{r}{} \cA$ for a level-$\ell$ pseudo-distribution $D$ and there exists a sum-of-squares proof $\cA \sststile{r'}{} \cB$, then $D \sdtstile{r\cdot r'+r'}{} \cB$.
\end{fact}

If the pseudo-distribution $D$ satisfies $\cA$ only approximately, soundness continues to hold if we require an upper bound on the bit-complexity of the sum-of-squares $\cA \sststile{r'}{} B$  (number of bits required to write down the proof).

In our applications, the bit complexity of all sum of squares proofs will be $n^{O(\ell)}$ (assuming that all numbers in the input have bit complexity $n^{O(1)}$).
This bound suffices in order to argue about pseudo-distributions that satisfy polynomial constraints approximately.

The following fact shows that every property of low-level pseudo-distributions can be derived by low-degree sum-of-squares proofs.

\begin{fact}[Completeness]
	\label{fact:sos-completeness}
	Suppose $d \geq r' \geq r$ and $\cA$ is a collection of polynomial constraints with degree at most $r$, and $\cA \vdash \{ \sum_{i = 1}^n x_i^2 \leq B\}$ for some finite $B$.
	
	Let $\{g \geq 0 \}$ be a polynomial constraint.
	If every degree-$d$ pseudo-distribution that satisfies $D \sdtstile{r}{} \cA$ also satisfies $D \sdtstile{r'}{} \{g \geq 0 \}$, then for every $\epsilon > 0$, there is a sum-of-squares proof $\cA \sststile{d}{} \{g \geq - \epsilon \}$.
\end{fact}

We will repeatedly use the following SoS version of Cauchy-Schwarz inequality and its generalization, Hölder's inequality:
\begin{fact}[Sum-of-Squares Cauchy-Schwarz]
	Let $x,y \in \R^d$ be indeterminites. Then,
	\[ 
	\sststile{4}{x,y} \Set{\Paren{\sum_i x_i y_i}^2 \leq \Paren{\sum_i x_i^2} \Paren{\sum_i y_i^2}} \,.
	\]
	\label{fact:sos-cauchy-schwarz}
\end{fact} 

We will also use the following facts about triangle inequalities and spectral certificates within the SoS proof system.

\begin{lemma}\label[lemma]{lem:triangleSoS}
	There is a degree-2 sum-of-squares proof of the following weak triangle inequality:
	\begin{equation}
		2\left(\sum_{i=1}^n a_i^2\right) + 2\left(\sum_{i=1}^n b_i^2\right) - \left(\sum_{i=1}^n (a_i+b_i)^2\right) = \sum_{i=1}^n (a_i-b_i)^2\,.
	\end{equation}
\end{lemma}

\begin{fact}[Spectral Certificates] \label{fact:spectral-certificates}
	For any $m \times m$ matrix $A$, 
	\[
	\sststile{2}{u} \Set{ \iprod{u,Au} \leq \Norm{A}\cdot \Norm{u}_2^2}\mper
	\]
\end{fact}

We will also use the following results about pseudo-distributions.

\begin{fact}[Cauchy-Schwarz for Pseudo-distributions]
	Let $f,g$ be polynomials of degree at most $d$ in indeterminate $x \in \R^d$. Then, for any degree d pseudo-distribution $D$,
	$\pE_{D}[fg] \leq \sqrt{\pE_{D}[f^2]} \sqrt{\pE_{D}[g^2]}$.
	\label{fact:pseudo-expectation-cauchy-schwarz}
\end{fact} 

\section{Meta-theorem}\label{sec:meta-theorem}

In this section we prove \cref{theorem:metaTheorem}.

\begin{theorem}[Meta-theorem]\label{thm:technical-meta-theorem}
	Let $\delta, \alpha \in (0,1)$ and $\zeta \ge 0$.
    Let $\pdset \subseteq \R^m$ be a compact convex set. Let $b,r,\gamma\in\R$ be such that 
    $$\max_{X\in \pdset}\normi{X}\leq b\,,$$ 
    $$\max_{X\in \pdset}\norm{X}_2\leq r\,,$$ and $$\E_{\bm W\sim N(0,\Id)}\Brac{\sup_{X\in\pdset} \iprod{X, \bm W}}\leq \gamma\,.$$
	Consider
	\begin{align*}
		\bm Y = X^*+ \bm N\,,
	\end{align*}
	where $X^*\in \tilde{\Omega}$ and $\bm N$ is a random $m$-dimensional vector with independent (but not necessarily identically distributed) symmetric about zero entries satisfying $\Pr \brac{\Abs{\bm N_i}\leq \zeta}\geq \alpha$. 
	
	Let 
	\[
	\eps = {\frac{\gamma^2}{r^2m\log m}}\,,
	\]
	 and let $Z$ be an $m$-dimensional vector such that at least $(1-\eps)m$ entries of $Z$ coincide with entries of $\bm Y$, and other entries are arbitrary.
	
	Then the minimizer $\hat{X} = \argmin_{X\in \pdset}{F_h(Z-X)}$ of the Huber loss with parameter $h \ge 2b + \zeta$ satisfies
	\begin{align*}
		\Normt{\hat{X}-{X}^*}\leq 
		O\Paren{\sqrt{{\frac{h}{\alpha}}\Paren{\gamma + r\sqrt{\log(1/\delta)}}}} 
	\end{align*}
	with probability at least $1-\delta$ over the randomness of $\bm N$.
\end{theorem}

Note that without loss of generality we can assume that $h\cdot \varepsilon \cdot m \log m\le \gamma$. Indeed, otherwise we would get $h\cdot \gamma > r^2$, and the error bound becomes trivial. Similarly, we can assume that $h\log(1/\delta) \le r\sqrt{\log(1/\delta)}$.

To prove the theorem we  need the next two intermediate lemmas.

\begin{lemma}[Gradient bound]\label{lem:meta-gradient-bound}
	Consider the settings of \cref{thm:technical-meta-theorem}. Then with probability at least $1-\delta$ over the randomness of $\bm N$, for every $X, X'\in \pdset$, we have
	\begin{align*}
		\Abs{\iprod{\nabla F_h(\bm N + Z - \bm Y), X-X'}}\leq  100h \cdot\Paren{\gamma + r\sqrt{\log(1/\delta)}}\,.
	\end{align*}
	\begin{proof}
		Let $C$ be the set of entries where $\bm Y$ differs from $Z$. The size of $C$ is at most $\eps m$, hence
\begin{equation}
\label{eq:meta-gradient-bound-1}
			\Abs{\sum_{i\in  C} f'_h(\bm N_i + Z_i - \bm Y_i) \cdot \Paren{X_i - X_i'}}\leq  h \cdot 2b \cdot \eps m\le h^2 \cdot \eps m \le h\gamma\,.
\end{equation}
		Now consider some fixed (non-random) subset $S$ of entries of size $(1-\eps)m$. 
		By \cref{lem:subgaussianity_of_gradient}, the random variable 
		$\sup_{X,X'\in \pdset}\sum_{i\in S} f'_h(\bm N_i) \cdot \Paren{X_i-X_i'}$ 
	    has expectation bounded by $6\cdot h\cdot \gamma$, and 
		for every $0<\delta'<1$, we get that with probability at least $1-\delta'$,
		\[
		\Abs{\sum_{i\in S} f'_h(\bm N_i) \cdot \Paren{X_i-X_i'} }\leq  h \cdot\Paren{6\gamma + 10r\sqrt{\log(1/\delta')}}\,.
		\]
		Now choose 
		$$\delta':=\frac{\delta}{\binom{m}{\epsilon m}}\geq \frac{\delta}{\Paren{\frac{e}{\epsilon}}^{\epsilon m}}.$$
By taking a union bound over all subsets\footnote{The number of such subsets is $\binom{m}{(1-\epsilon)m}=\binom{m}{\epsilon m}$.} of size $(1-\eps)m$, we can see that with probability at least $1-\delta$, we have
 \begin{align*}
		\Abs{\sum_{i\in S} f'_h(\bm N_i) \cdot \Paren{X_i-X_i'} }
&\leq  h \cdot\Paren{6\gamma + 10r\sqrt{\log(1/\delta')}}\\		
		&\leq  h \cdot\Paren{6\gamma + 10r\sqrt{\log(1/\delta)}+10r\sqrt{\eps m\log(e/\eps)}}
\end{align*}		
for all subsets of size $(1-\epsilon)m$. In particular, for $S=[m]\setminus C$, with probability at least $1-\delta$, we have
		
\begin{equation}
\label{eq:meta-gradient-bound-2}
		\Abs{\sum_{i\in [m]\setminus C} f'_h(\bm N_i) \cdot \Paren{X_i-X_i'} }
		\leq  h \cdot\Paren{6\gamma + 10r\sqrt{\log(1/\delta)}+10r\sqrt{\eps m\log(e/\eps)}}\,.
\end{equation}

Now notice that
\begin{align*}
r^2\eps m\log(e/\eps) &= r^2\cdot \frac{\gamma^2}{r^2 m\log m} \cdot m\Paren{1+ \log\frac{r^2 m\log m}{\gamma^2}} \\
&= \frac{\gamma^2}{\log m} \Paren{1+ \log m + \log \log m + \log\frac{r^2 }{\gamma^2}} \leq 4\gamma^2\,.
\end{align*}
		
By combining this with \eqref{eq:meta-gradient-bound-1} and \eqref{eq:meta-gradient-bound-2}, we get that with probability at least $1-\delta$, we have
		\begin{align*}
\Abs{\iprod{\nabla F_h(\bm N + Z - \bm Y), X-X'}}&\leq  h \cdot\Paren{10\gamma +  10r\sqrt{\eps m\log(e/\eps)} +  10r\sqrt{\log(1/\delta)}}\\
&\le  100h \cdot\Paren{\gamma + r\sqrt{\log(1/\delta)}}\,.		
		\end{align*}
	\end{proof}
\end{lemma}

\begin{lemma}[Local strong convexity]\label{lem:meta-local-strong-convexity}
	Consider the settings of \cref{thm:technical-meta-theorem} and let $h \ge \zeta+2b$.
	Denote $M = \bm{N} + Z - \bm Y = Z-X^*$.
	With probability at least $1-\delta$  over the randomness of $\bm N$, for every $X,X'\in \pdset$ satisfying
	 $\norm{X-X'}\geq 20\sqrt{\frac{b}{\alpha}\Paren{4\gamma+b\log(1/\delta)}}$, we have
	\begin{equation*}
		F_h(M + X - X')\geq F_h(M) + \iprod{\nabla F_h(M),X-X'} + \frac{ \alpha}{10}\Normt{X-X'}^2\,.
	\end{equation*}
	\begin{proof}
		Fix some $X, X'$ and let $\Delta = X - X'$.
		Using \cref{lem:second-order-behavior-huber}, 
		\begin{align*}
F_h(M+\Delta)
			&=\sum_{i\in [m]}f_h(M_{i}+\Delta_i)\\
			&\geq \sum_{i\in[m]}\Paren{f_h(M_{i})+f_h'(M_{i})\cdot \Delta_i + \frac{\Delta_i^2}{2}\ind{\abs{M_{i}}\leq \zeta}\cdot \ind{\Abs{\Delta_i}\leq h-\zeta}}\\
				&= F_h(M) + \iprod{\nabla F_h(M), \Delta} + \frac{1}{2}\sum_{i\in [m]}\Ind_{\{|M_{i}|\leq \zeta\}}\Delta_i^2\,.
		\end{align*}
		Denote $ U = \Set{i\in [m] : Z_i = \bm Y_i}$.
		It suffices to show that with probability at least $1-\delta$, for every $X, X'\in \pdset$ satisfying $\norm{X-X'}\geq 20\sqrt{\frac{b}{\alpha}\Paren{4\gamma+b\log(1/\delta)}}$, we have
		\begin{equation*}
			\sum_{i\in  U}\Ind_{\{|\bm{N}_{i}|\leq \zeta\}}\cdot \Delta_i^2\geq \frac{\alpha}{10} \Normt{\Delta}^2\,,
		\end{equation*}
		where $\Delta = X-X'$.
		To this end fix such an $X,X'\in \pdset$.
		For every $i\in[m]$, define the random variable
		\begin{equation*}
			\bm{z}_{i} := \Ind_{\{|\bm{N}_{i}|\leq \zeta\}}\cdot \Delta_i^2\,.
		\end{equation*}
		
		Let $S$ be an arbitrary fixed (non-random) set of size $(1-\eps)m$
		and let
		\begin{equation*}
			\bm{z} = \sum_{i\in S} \bm{z}_{i}\,.
		\end{equation*} We have
		\begin{align*}
			\E[\bm z] = \sum_{i\in S} \E[\bm{z}_{i}] = \sum_{i\in S} \Pr[|\bm{N}_{i}|\leq \zeta]\cdot \Delta_i^2\geq 
			\sum_{i\in [m]} \alpha\cdot \Delta_i^2 - 4\eps m\cdot b^2
			\ge \frac{\alpha}{2}\cdot\normt{\Delta}^2\,,
		\end{align*}
where the last inequality holds because we assumed (without loss of generality) that $h\cdot \epsilon \cdot m\log m\leq \gamma$, and hence
$$4\eps m \cdot b^2 \le h^2 \cdot \eps \cdot m \log m \le  h \gamma \le \frac{\alpha}{2}\cdot\normt{\Delta}^2.$$
		On the other hand, since $0\leq \bm{z}_{i}\leq \Delta_i^2\leq\normi{\Delta}^2\leq 4b^2$ for all $i\in[m]$,  we have
		\[
\sum_{i\in S} \E[\bm{z}_i^2] \leq \sum_{i\in S} 4b^2\cdot\E[\bm{z}_i] =4b^2\cdot \E[\bm z]\,.
		\]
		
		 It follows from Bernstein's inequality that for every $t>0$, we have
		\begin{align*}
			\Pr[\bm{z}-\E[\bm{z}]\leq -t] 
			\le \exp\Paren{-\frac{t^2/2}{4b^2\cdot \E[\bm z]+\normi{\Delta}^2 \cdot t}} 
			\le \exp\Paren{-\frac{t^2/2}
				{4b^2\cdot \E[\bm z]+4b^2 \cdot t}}
			\,.
		\end{align*}
		By taking $t=\frac{\E[\bm z]}{2}\geq \frac{\alpha}{4}\cdot\normt{\Delta}^2$, we get
		\begin{align*}
			\Pr\left[\bm{z}\leq \frac{\alpha}{4}\Normt{\Delta}^2\right] &\le \Pr\left[\bm{z}\leq \frac{\E[\bm z]}{2}\right]\leq \Pr\left[\bm{z}-\E[\bm{z}]\leq -\frac{\E[\bm z]}{2}\right]\\
			&\leq \exp\Paren{-\frac{\E[\bm z]^2/8}{4b^2\E[\bm z]+4b^2\E[\bm z]/2}}
			\le  \exp\Paren{-\frac{\E[\bm z]}{64 b^2}}\leq \exp\Paren{-\frac{\alpha\normt{\Delta}^2}{128 b^2}}\,.
		\end{align*}

By taking a union bound over all subsets\footnote{The number of such subsets is $\binom{m}{(1-\epsilon)m}=\binom{m}{\epsilon m}$.} of size $(1-\eps)m$, we get		

\begin{align}
			\Pr\left[\sum_{i\in U} \Ind_{\{|\bm{N}_{i}|\leq \zeta\}}\cdot \Delta_i^2\leq \frac{\alpha}{4}\Normt{\Delta}^2\right] &\leq \binom{m}{(1-\epsilon)m} \exp\Paren{-\frac{\alpha\normt{\Delta}^2}{128 b^2}}=\binom{m}{\epsilon m} \exp\Paren{-\frac{\alpha\normt{\Delta}^2}{128 b^2}}\nonumber\\
			&\leq \Paren{\frac{e}{\epsilon}}^{\epsilon m}\exp\Paren{-\frac{\alpha\normt{\Delta}^2}{128 b^2}}=\exp\Paren{\epsilon m\log(e/\epsilon)-\frac{\alpha\normt{\Delta}^2}{128 b^2}}\nonumber\\
			&\leq\exp\Paren{\frac{2\gamma}{b}-\frac{\alpha\normt{\Delta}^2}{128 b^2}} \label{eq:meta-local-strong-convexity-1}\,,
		\end{align}
		where in the last inequality, we used the fact that we assumed (without loss of generality) that $h\cdot \epsilon \cdot m\log m\leq \gamma$, which means that
		\begin{align*}
\eps m\log(e/\eps) &= \epsilon m \Paren{1+ \log\frac{r^2 m\log m}{\gamma^2}} \\
&= \epsilon m\Paren{1+ \log m + \log \log m + \log\frac{r^2 }{\gamma^2}} \leq 4\epsilon m \log m\leq  \frac{4\gamma}{h}\leq \frac{2\gamma}{b}\,.
\end{align*}

Now define
$$L= 20\sqrt{\frac{b}{\alpha}\Paren{4\gamma+b\log(1/\delta)}}\,,$$

$$\mathfrak{D}:=\Set{X-X': X,X'\in \pdset~,~\normt{X-X'}\geq L}\,,$$
and
$$\tilde{\epsilon}=\frac{L\sqrt{\alpha}}{10}\,,$$
and let $\cN_{\tilde{\epsilon}}\Paren{\mathfrak{D}}$ be an $\tilde{\epsilon}$-net of $\mathfrak{D}$ of minimal size. Using Sudakov's minoration \cref{fact:SudakovMinoration}, we have
\begin{align*}
\frac{\tilde{\epsilon}}{2}\sqrt{\log\Card{\cN_{\tilde{\epsilon}}\Paren{\mathfrak{D}}}} &\leq \E_{\mathbf{g}\sim N(0,I_m)}\Brac{\sup_{\Delta\in \mathfrak{D}}\Iprod{\mathbf{g},\Delta}}\leq \E_{\mathbf{g}\sim N(0,I_m)}\Brac{\sup_{X,X'\in \pdset}\Iprod{\mathbf{g},X-X'}}\\
&\leq \E_{\mathbf{g}\sim N(0,I_m)}\Brac{\sup_{X,X'\in \pdset}\Iprod{\mathbf{g},X} + \sup_{X,X'\in \pdset}\Iprod{\mathbf{g},-X'}}\\
&= \E_{\mathbf{g}\sim N(0,I_m)}\Brac{\sup_{X\in \pdset}\Iprod{\mathbf{g},X}} +\E_{\mathbf{g}\sim N(0,I_m)}\Brac{ \sup_{X'\in \pdset}\Iprod{\mathbf{g},X'}} = 2\gamma\,.
\end{align*}	

By taking a union bound over $\cN_{\tilde{\epsilon}}\Paren{\mathfrak{D}}$ and applying \eqref{eq:meta-local-strong-convexity-1}, it follows that
\begin{align*}
			&\Pr\left[\exists \Delta\in \cN_{\tilde{\epsilon}}\Paren{\mathfrak{D}}~,~ \sum_{i\in U} \Ind_{\{|\bm{N}_{i}|\leq \zeta\}}\cdot \Delta_i^2\leq \frac{\alpha}{4}\Normt{\Delta}^2\right] \\
			&\quad\quad\quad\leq  \Card{\cN_{\tilde{\epsilon}}\Paren{\mathfrak{D}}}\exp\Paren{\frac{2\gamma}{b}-\frac{\alpha L^2}{128 b^2}} \leq \exp\Paren{\frac{16 \gamma^2}{\tilde{\epsilon}^2}+ \frac{2\gamma}{b}-\frac{\alpha L^2}{128 b^2}} \leq  \exp\Paren{\frac{1600\gamma^2}{\alpha L^2}+ \frac{2\gamma}{b}-\frac{\alpha L^2}{128 b^2}}\\
			&\quad\quad\quad\stackrel{(\ast)}{\leq}  \exp\Paren{\frac{\alpha L^2}{1600 b^2}+ \frac{\alpha L^2}{800 b^2}-\frac{\alpha L^2}{128 b^2}}\leq \exp\Paren{-\frac{\alpha L^2}{400 b^2}} \stackrel{(\dagger)}{\leq }\delta\,,
		\end{align*}
		where $(\ast)$ follow from the fact that $L\geq 20\sqrt{\frac{4\gamma b}{\alpha}}$, which implies that $\frac{1600\gamma^2}{\alpha L^2}\leq \frac{\alpha L^2}{1600 b^2}$ and $\frac{2\gamma}{b}\leq \frac{\alpha L^2}{800 b^2}$. $(\dagger)$ follows from the fact that $L\geq 20\sqrt{\frac{b^2}{\alpha}\log(1/\delta)}$.
		
		We conclude that with probability at least $1-\delta$, it holds that for every $\Delta\in \cN_{\tilde{\epsilon}}\Paren{\mathfrak{D}}$, we have
		
		$$\sum_{i\in U} \Ind_{\{|\bm{N}_{i}|\leq \zeta\}}\cdot \Delta_i^2\geq \frac{\alpha}{4}\Normt{\Delta}^2\,.$$
		
		Assume that this event happens, and let $\Delta\in \mathfrak{D}$ be arbitrary. We can decompose $\Delta$ as
		$$\Delta=A+B\,,$$
		where $A\in \cN_{\tilde{\epsilon}}\Paren{\mathfrak{D}}$ and $\normi{B}\leq \tilde{\epsilon}$. We have
\begin{align*}
\sum_{i\in U} \Ind_{\{|\bm{N}_{i}|\leq \zeta\}}\cdot \Delta_i^2&= \sum_{i\in U} \Ind_{\{|\bm{N}_{i}|\leq \zeta\}}\cdot \Paren{A_i+B_i}^2\geq \frac{1}{2} \sum_{i\in U} \Ind_{\{|\bm{N}_{i}|\leq \zeta\}}\cdot A_i^2 - \sum_{i\in U} \Ind_{\{|\bm{N}_{i}|\leq \zeta\}}\cdot B_i^2\\
&\geq \frac{\alpha}{4}\normt{A}^2 - \normt{B}^2=\frac{\alpha}{4}\normt{\Delta - B}^2-\normt{B}^2\geq \frac{\alpha}{4}\Paren{\normt{\Delta}-\normt{B}}^2-\normt{B}^2\\
&\stackrel{(\ddagger)}{\geq} \frac{\alpha}{4}\Paren{\normt{\Delta}-\frac{\normt{\Delta}\sqrt{\alpha}}{10}}^2-\Paren{\frac{\normt{\Delta}\sqrt{\alpha}}{10}}^2\geq\frac{\alpha}{10}\normt{\Delta}^2\,,
\end{align*}		
		
		where $(\ddagger)$ follows from the fact that $\displaystyle \normi{B}\leq \tilde{\epsilon}=\frac{L\sqrt{\alpha}}{10}\leq \frac{\normt{\Delta}\sqrt{\alpha}}{10}$.
	\end{proof}
\end{lemma}

We can now prove the theorem.

\begin{proof}[Proof of \cref{thm:technical-meta-theorem}]
	We may assume $\normt{X^*-\hat{X}}\geq 20\sqrt{\frac{b}{\alpha}\Paren{4\gamma+b\log(1/\delta)}}$, since otherwise the statement is trivially true. By definition, for $M=Z-X^*$, we have
	\begin{align*}
		F_h(Z-\hat{X})\leq F_h(Z-X^*)= F_h(M)\,,
	\end{align*}
	 and by \cref{lem:meta-local-strong-convexity}, with probability $1-\delta$, we have
	\begin{align*}
		F_h(Z-\hat{ X})\geq F_h(M) -\iprod{\nabla F_h(M), X^*-\hat{ X}} + \frac{\alpha}{10}\Normt{X^*-\hat{ X}}^2\,.
	\end{align*}
	Combining the two inequalities and rearranging, we get
	\begin{align*}
		\Normt{X^*-\hat{X}}^2 &\leq {\frac{10}{\alpha}\Abs{\iprod{\nabla F_h(\bm{N}), X^*-\hat{ X}} }}\leq 
O\Paren{{\frac{h}{\alpha}}\Paren{\gamma + r\sqrt{\log(1/\delta)}}}\,,
	\end{align*}
	and the result follows.
\end{proof}

\section{Applications}\label{section:applications}

In this section we apply \cref{thm:technical-meta-theorem} to various estimation problems.
\subsection{Tensor PCA with oblivious outliers}
\label{sec:obliviousTensorPca}

We show here how \cref{thm:technical-meta-theorem} can be used to recover a rank-1 tensor under symmetric noise.
We will need the following fact:

\begin{fact}[\cite{tensor_pca_sos}]\label{fact:injectiveNormCertificateProb}
	Let $p\ge 3$ be an odd number, and let $\bm W \in \Paren{\R^{n}}^{\otimes p}$ be a tensor with \iid entries from $N(0,1)$. 
	Then with probability\footnote{Note that \Cref{fact:injectiveNormCertificateProb} follows from applying Theorem 56 of \cite{tensor_pca_sos} to $p$-tensors in the same way as it was applied to 3-tensors in order to prove Theorem 11 and Corollary 12 of \cite{tensor_pca_sos}. It is worth mentioning that the probability that was reported in \cite[Theorem 56]{tensor_pca_sos} is $1-O(n^{-100})$. However, a closer look at the proof of Theorem 56 in Page 47 of \cite{tensor_pca_sos}, we can see that the probability at least $1-\delta$ can be easily obtained for arbitrarily small $\delta$.} $1-\delta$ (over $\bm W$) every pseudo-distribution $\mu$ of degree at least $2p-2$ on indeterminates $x=(x_1,\ldots,x_n)$ satisfies
	\[
	\pE_{x\sim\mu} \iprod{x^{\otimes p}, \bm W} \le C\cdot \Paren{\Paren{n^p \cdot p \cdot \ln n}^{1/4} + 
		n^{p/4} \Paren{\ln(1/\delta)}^{1/4} + n^{1/4}\Paren{\ln(1/\delta)}^{3/4}}
	\cdot \Paren{\pE_{x\sim\mu}\norm{x}^{2p-2}}^{\frac{p}{2p-2}}
	\] 
	for some absolute constant $C$.
\end{fact}

The following corollary is a simple application of \cref{lemma:expectation_from_tails} to \cref{fact:injectiveNormCertificateProb}.
\begin{corollary}
\label{cor:injectiveNormCertificateExpectation}
Let $p\ge 3$ be an odd number and let $\bm W \in \Paren{\R^{n}}^{\otimes p}$ be a tensor with \iid entries from $N(0,1)$. Then for every $d\geq 2p-2$, we have
	\[
	\E\Brac{\sup_{X\in \tilde{\Omega}_{n,d}}\iprod{X, \bm W}} \le O\Paren{\Paren{n^p \cdot p \cdot \ln n}^{1/4}}\,,
	\] 
	where 
\[
\tilde{\Omega}_{n,d}=\Big\{\pE_{x\sim\mu}x^{\otimes p}: \mu\in \mathcal{P}_d~,~ \pE_{x\sim\mu}\norm{x}^{2}\leq 1\Big\}\,,
\]	
and  $\mathcal{P}_d$ is the set of pseudo-distributions over $\R[x]=\R[x_1,\ldots,x_n]$ of degree $d$.
\end{corollary}

The following fact is an easy consequence of \cref{fact:spectral-certificates} and the bound on the expected value of the spectral norm of Gaussian $n^{p/2}\times n^{p/2}$ matrix:
\begin{fact}\label{fact:gaussian-complexity-even-tensors}
\label{fact:injectiveNormCertificateExpectationEven}
	Let $p\ge 2$ be an even number and let $\bm W \in \Paren{\R^{n}}^{\otimes p}$ be a tensor with \iid entries from $N(0,1)$. Then for every $d\geq p$, we have
	\[
	\E\Brac{\sup_{X\in \tilde{\Omega}_{n,d}}\iprod{X, \bm W}} \le O\Paren{n^{p/4}}\,,
	\] 
	where 
	\[
	\tilde{\Omega}_{n,d}=\Big\{\pE_{x\sim\mu}x^{\otimes p}: \mu\in \mathcal{P}_d~,~ \pE_{x\sim\mu}\norm{x}^{2}\leq 1\Big\}\,,
	\]	
	and  $\mathcal{P}_d$ is the set of pseudo-distributions over $\R[x]=\R[x_1,\ldots,x_n]$ of degree $d$.

\begin{proof}
Fix $X\in \tilde{\Omega}_{n,d}$ and let $\mu\in \mathcal{P}_d$ be such that $\pE_{x\sim\mu}\norm{x}^{2}\leq 1$ and $X=\pE_{x\sim\mu}x^{\otimes p}$. Denote by $\bm W'$ the $n^{p/2}\times n^{p/2}$ matrix that is obtained by reshaping $\bm W$. We have:
\begin{align*}
\iprod{X, \bm W} &= \pE_{x\sim\mu}\Iprod{x^{\otimes p}, \bm W} = \pE_{x\sim\mu}\Iprod{x^{\otimes p/2}, \bm W' x^{\otimes p/2}}\stackrel{(\ast)}{\leq} \normt{\bm W'}\pE_{x\sim\mu}\normt{x^{\otimes p/2}}^2\stackrel{(\dagger)}{\leq} \normt{\bm W'}\,,
\end{align*}
$(\ast)$ follows from \Cref{fact:spectral-certificates} and $(\dagger)$ follows from the fact that $$\pE_{x\sim\mu}\normt{x^{\otimes p/2}}^2=\pE_{x\sim\mu}(\normt{x}^2)^{p/2}\leq 1\,.$$ Therefore,
\[
	\E\Brac{\sup_{X\in \tilde{\Omega}_{n,d}}\iprod{X, \bm W}} \le \E\Brac{\normt{\bm W'}}\leq O\Paren{n^{p/4}}\,.
\] 
The last inequality follows from the well known bounds on the expected spectral norm of a Gaussian matrix, and can be immediately seen from \Cref{fact:spectral_norm_gaussian} and \cref{lemma:expectation_from_tails}.
\end{proof}

\end{fact}

The following two theorems are about tensor PCA with asymmetric tensor noise.

\begin{theorem}[Asymmetric Tensor Noise of odd order]\label{theorem:mainTensor} 
	Let $p\ge 3$ be an odd number. 
Let $n\in \N$, $n\ge 2$,  $\lambda > 0$ and $\alpha \in (0,1]$. 
	Let $\bm T = \lambda\cdot v^{\otimes p} + \bm N$,
	 where $v\in \R^n$ is a unit vector and $\bm N$ is a random tensor whose entries are independent 
	 (but not necessarily identically distributed), symmetric about zero
	 and satisfy $\Pr\Brac{\Abs{\bm N_{i_1\ldots i_p}} \le 1}\ge \alpha$ for all $i_1,\ldots,i_p\in[n]$. 
	
	There exists an absolute constant $C>1$ and an algorithm such that if 
	\[
	\lambda \geq \frac{C}{\alpha}\cdot\Paren{p \ln n}^{1/4} \cdot n^{p/4}
	\]
	and 
	\[
	\normi{v} \leq \frac{(\alpha/C)^{1/p}}{n^{1/4}(på\ln n)^{1/(4p)}}\,,
	\]
	then the algorithm on input $\bm T$ runs in time  $\Paren{n^{p}}^{O(1)}$  and outputs a unit vector $\hat{\bm v} \in \mathbb{R}^n$ satisfying
	\[
			{\iprod{v,\hat{\bm v}}}\ge 0.99
	\]
	 with probability at least $1-2^{-n}$. 
	 
	Furthermore,  for $\eps \leq \Paren{C n^p \cdot p\ln n}^{-1/2}$, the same result holds if an arbitrary (adversarially chosen) $\epsilon$-fraction of entries of $\bm{T}$ is replaced by adversarially chosen values.
	\begin{proof}
		We can apply \cref{thm:technical-meta-theorem} for input $$\bm Y = \bm T/\lambda = v^{\otimes p}+\bm N'\,,$$
where
$$\bm N'=\frac{1}{\lambda}\bm N\,.$$		

In order to do so, define the set $$\pdset = \Set{ \pE_{x\sim\mu} x^{\otimes p} : \mu \in \cE}\,,$$
		where $\cE$ is the set of pseudo-distributions over $\R[x_1,\ldots, x_n]$ of degree $2p$
		that satisfy the constraints $\pE_{x\sim \mu}\normt{x}^2 \le 1$ and
		${\pE_{x\sim \mu} x^2_i}\le \frac{1}{\lambda^{2/p}}$ for all $i\in[n]$.
		
Define $\zeta = \frac{1}{\lambda}$ so that for every $i_1,\ldots,i_p\in[n]$, we have
$$\Pr\Brac{\Abs{\bm N_{i_1\ldots i_p}'}\leq\zeta}=\Pr\Brac{\Abs{\bm N_{i_1\ldots i_p}} \le 1}\ge \alpha\,.$$

Now notice that

\begin{align*}
\max_{X\in \pdset}\normi{X}&=\max_{\mu\in \cE}\normi{\pE_{x\sim\mu} x^{\otimes p}}=\max_{\mu\in \cE} \max_{i_1,\ldots,i_p\in[n]}\Abs{\pE_{x\sim\mu} x_{i_1}\ldots x_{i_p}}\\
&\leq \max_{\mu\in \cE} \max_{i_1,\ldots,i_p\in[n]}\sqrt{\pE_{x\sim\mu} x_{i_1}^2\ldots x_{i_p}^2}\stackrel{(\ast)}{\leq} \sqrt{\Paren{ \frac{1}{\lambda^{2/p}}}^p}=\frac{1}{\lambda}\,,
\end{align*}
where $(\ast)$ follows from the fact that ${\pE_{x\sim \mu} x^2_i}\le  \frac{1}{\lambda^{2/p}}$ for all $\mu\in\cE$  and all $i\in[n]$. Furthermore,
		
		\begin{align*}
\max_{X\in \pdset}\norm{X}_2&=\max_{\mu\in \cE}\normt{\pE_{x\sim\mu} x^{\otimes p}}\leq \max_{\mu\in \cE}\sqrt{\pE_{x\sim\mu} \normt{x^{\otimes p}}^2}=\max_{\mu\in \cE}\sqrt{\pE_{x\sim\mu} \normt{x}^{2p}}\leq 1\,,
\end{align*}
where the last inequality follows from the fact that $\pE_{x\sim \mu}\normt{x}^2 \le 1$ for all $\mu\in\cE$. 

Let $b=\frac{1}{\lambda}$ and $r=1$ so that $\max_{X\in \pdset}\normi{X}\leq b$ and $\max_{X\in \pdset}\norm{X}_2\leq r$. By defining $h=\frac{3}{\lambda}$, we can see that
$$h\geq \zeta+2b\,,$$

Now from \cref{cor:injectiveNormCertificateExpectation} and \cref{fact:injectiveNormCertificateExpectationEven}, we can see that there is 
$$\gamma = \Theta\Paren{O\Paren{\Paren{p \ln n}^{1/4} \cdot n^{p/4}}}$$ such that the Gaussian complexity of $\pdset$ satisfies
$$\E\Brac{\sup_{X\in\pdset} \iprod{X, \bm W}}\leq \gamma\,,$$
where $\bm W$ is a random tensor in $(\R^n)^{\otimes p}$ whose  entries are i.i.d. standard Gaussian $N(0,1)$.
		 
		 If follows from \cref{thm:technical-meta-theorem} that with probability at least $1-2^{-n}$, the pseudo-distribution $\pE$ that minimizes the Huber loss satisfies 
\begin{align*}
		 \Normt{v^{\otimes p}-\pE x^{\otimes p}} &\leq O\Paren{\sqrt{{\frac{h}{\alpha}}\Paren{\gamma + r\sqrt{\log(1/2^{-n})}}}}
		  \\
		  &\leq O\Paren{\sqrt{{\frac{3}{\lambda\alpha}}\Paren{\Paren{p \ln n}^{1/4} \cdot n^{p/4} + 1\sqrt{n\log 2}}}}=O\Paren{\frac{1}{C}}\,.
\end{align*}		 
By making $C>0$ arbitrarily large, we can make the above bound on $ \Normt{v^{\otimes p}-\pE x^{\otimes p}}$ arbitrarily small.
		
		Now we take $\hat{\bm v} = \pE x /\norm{\pE x}$.
		
		By \cref{lem:vectorRecoveryOdd}, $\hat{\bm v}$ satisfies the desired bound with probability at least $1-2^{-n}$. Note that Theorem \cref{thm:technical-meta-theorem} also implies that we can also afford an $\epsilon$ fraction of arbitrary adversarial changes in the observed tensor as long as
\begin{align*}
\epsilon\leq {\frac{\gamma^2}{r^2n^p\log (n^p)}} = \frac{O\Paren{\Paren{p \ln n}^{1/2} \cdot n^{p/2}}}{n^p p\log n}=O\Paren{ \Paren{n^p \cdot p\ln n}^{-1/2}}\,.
\end{align*}		
		
	\end{proof}
\end{theorem}

\begin{theorem}[Asymmetric Tensor Noise of even order]\label{theorem:mainTensorEven}
	Let $p\ge 2$ be an even number. 
	Let $n\in \N$, $n\ge 2$,  $\lambda > 0$ and $\alpha \in (0,1]$. 
	Let $\bm T = \lambda\cdot v^{\otimes p} + \bm N$,
	where $v\in \R^n$ is a unit vector and $\bm N$ is a random tensor whose entries are independent 
	(but not necessarily identically distributed), symmetric about zero
	and satisfy $\Pr\Brac{\Abs{\bm N_{i_1\ldots i_p}} \le 1}\ge \alpha$ for all $i_1,\ldots,i_p\in[n]$. 
	
	There exists an absolute constant $C>1$ and an algorithm such that if 
	\[
	\lambda \ge \frac{C}{\alpha}\cdot n^{p/4}
	\]
	and 
	\[
	\normi{v} \le \frac{\Paren{\alpha/C}^{1/p}}{n^{1/4}}\,,
	\]
	then the algorithm on input $\bm T$ runs in time  $\Paren{n^{p}}^{O(1)}$  and outputs a unit vector $\hat{\bm v} \in \mathbb{R}^n$ satisfying
	\[
	\abs{\iprod{v,\hat{\bm v}}}\ge 0.99
	\]
	with probability at least $1-2^{-n}$. 
	
	Furthermore,  for $\eps \le \Paren{C n^{p/2} \cdot p\ln n}^{-1}$,  the same result holds if an arbitrary (adversarially chosen) $\epsilon$-fraction of entries of $\bm{T}$ is replaced by adversarially chosen values.
	\begin{proof}
		The proof is very similar to the proof of \cref{theorem:mainTensor}, we only need to use \cref{fact:gaussian-complexity-even-tensors} to bound the Gaussian complexity. For rounding, we can take $\hat{\bm v}$ to be the top eigenvector of $\pE xx^\top$ and by \cref{lem:vectorRecoveryEven} $\hat{\bm v}$ satisfies the desired bound with probability at least $1-2^{-n}$.
	\end{proof}
\end{theorem}

The next two theorems is about tensor PCA with symmetric tensor noise.

\begin{theorem}[Symmetric Tensor Noise of odd order]\label{theorem:mainTensorSymmetric}
	Let $p\ge 3$ be an odd number. 
Let $n\in \N$, $n\ge 2$,  $\lambda > 0$ and $\alpha \in (0,1]$. 
	Let $\bm T = \lambda\cdot v^{\otimes p} + \bm N$, 
		 where $v\in \R^n$ is a unit vector and $\bm N$ is a random symmetric tensor whose entries $\bm N_{i_1\ldots i_p}$ with indices $i_1\le i_2\le\ldots\le i_p$ are independent 
	(but not necessarily identically distributed), symmetric about zero
	and satisfy $\Pr\Brac{\Abs{\bm N_{i_1\ldots i_p}} \le 1}\ge \alpha$. 
	
	There exists an absolute constant $C>1$ and an algorithm such that if 
	\[
	\lambda \geq \frac{C p!}{\alpha}\cdot\Paren{p \ln n}^{1/4} \cdot n^{p/4}
	\]
	and 
	\[
	\normi{v}\leq \frac{\alpha^{1/p}}{(Cp!)^{1/p}\cdot n^{1/4}\cdot(p\ln n)^{1/4p}}\,,
	\]
	then the algorithm on input $\bm T$ runs in time  $\Paren{n^{p}}^{O(1)}$  and outputs a unit vector $\hat{\bm v} \in \mathbb{R}^n$ satisfying
	\[
	{\iprod{v,\hat{\bm v}}}\ge 0.99
	\]
	with probability at least $1-2^{-n}$. 
	
	Furthermore,  for $\eps \leq \Paren{C n^p \cdot p\ln n}^{-1/2}$, the same result holds if an arbitrary (adversarially chosen) $\epsilon$-fraction of entries of $\bm{T}$ is replaced by adversarially chosen values.
\end{theorem}
\begin{proof}
The proof is similar to the asymmetric $\bm N$, but we apply \cref{thm:technical-meta-theorem} for input $$\bm Y = \bm T'/\lambda=X'+\bm N'\,,$$ where $\bm T', X'$ and $\bm N'$ are the restrictions of $\bm T, v^{\otimes p}$ and $\frac{1}{\lambda}\bm N$ to the entries 
$\bm T_{i_1\ldots i_p}$ , $(v^{\otimes})_{i_1\ldots i_p}$ and $ \frac{1}{\lambda}\bm N_{i_1\ldots i_p}$ with indices $i_1\le i_2\le\ldots\le i_p$, respectively.
We also use the set $$\pdset = \Set{\Paren{\pE_{x\sim \mu} x_{i_1}\cdots x_{i_p}}_{i_1\le\ldots\leq i_p} : \mu \in \cE}\,,$$
where $\cE$ is as in \cref{theorem:mainTensor}.

We define $\zeta=\frac{1}{\lambda}$  so that for every $1\leq i_1\leq \ldots\leq i_p\leq n$, we have
$$\Pr\Brac{\Abs{\bm N_{i_1\ldots i_p}'}\leq\zeta}=\Pr\Brac{\Abs{\bm N_{i_1\ldots i_p}} \le 1}\ge \alpha\,.$$

By defining $r=1$ and $b=\frac{1}{\lambda}$, we can show similarly to \cref{theorem:mainTensor} that $\max_{X\in \pdset}\normi{X}\leq b$ and $\max_{X\in \pdset}\norm{X}_2\leq r$. By defining $h=\frac{3}{\lambda}$, we can see that
$$h\geq \zeta+2b\,.$$

Also similarly to \cref{theorem:mainTensor}, we can show that for some $\gamma =\Theta\Paren{\Paren{p \ln n}^{1/4} \cdot n^{p/4}}$, the Gaussian complexity of $\pdset$ can be bounded\footnote{We use \cref{cor:injectiveNormCertificateExpectation} and \cref{fact:injectiveNormCertificateExpectationEven}, together with \cref{fact:lipschitz_transformation_complexity} which implies that the Gaussian complexity does not increase if we restrict to a subset of coordinates.} as
$$\E\Brac{\sup_{X\in\pdset} \iprod{X, \bm W}}\leq \gamma\,,$$
where $(\bm W)_{i_1\ldots i_p}$ are i.i.d. standard Gaussian $N(0,1)$ for $1\leq i_1\leq \ldots\leq i_p\leq n$.

If follows from \cref{thm:technical-meta-theorem} that with probability at least $1-2^{-n}$, the pseudo-distribution $\pE$ that minimizes the Huber loss satisfies 
\begin{align*}
\sum_{1\le i_1\le\ldots\le i_p\le n} \Paren{v_{i_1}\cdots v_{i_p} - \pE x_{i_1}\cdots x_{i_p}}^2
&\leq O\Paren{\sqrt{{\frac{h}{\alpha}}\Paren{\gamma + r\sqrt{\log(1/2^{-n})}}}}\\
&\leq O\Paren{\sqrt{{\frac{3}{\lambda\alpha}}\Paren{\Paren{p \ln n}^{1/4} \cdot n^{p/4} + 1\sqrt{n\log 2}}}}
\\
&\leq O\Paren{ \frac{1}{C p!}}\,.
\end{align*}
Therefore, with probability at least $1-2^{n}$, we have
\[
\Normt{v^{\otimes p}-\pE x^{\otimes p}} = p! \cdot \sum_{1\le i_1\le\ldots\le i_p\le n} \Paren{v_{i_1}\cdots v_{i_p} - \pE x_{i_1}\cdots x_{i_p}}^2  \leq O\Paren{\frac{1}{C}}\,.
\]
The remaining of the proof is the same as in the asymmetric case.
\end{proof}

We can similarly modify the proof of \cref{theorem:mainTensorEven} to get the theorem for symmetric tensor noise:

\begin{theorem}[Symmetric Tensor Noise of even order]\label{theorem:mainTensorSymmetricEven}
	Let $p\ge 2$ be an even. 
	Let $n\in \N$, $n\ge 2$,  $\lambda > 0$ and $\alpha \in (0,1]$. 
	Let $\bm T = \lambda\cdot v^{\otimes p} + \bm N$, 
	where $v\in \R^n$ is a unit vector and $\bm N$ is a random symmetric tensor whose entries $\bm N_{i_1\ldots i_p}$ with indices $i_1\le i_2\le\ldots\le i_p$ are independent 
	(but not necessarily identically distributed), symmetric about zero
	and satisfy $\Pr\Brac{\Abs{\bm N_{i_1\ldots i_p}} \le 1}\ge \alpha$. 
		
	There exists an absolute constant $C>1$ and an algorithm such that if 
	\[
\lambda \ge \frac{Cp!}{\alpha}\cdot n^{p/4}
\]
and 
\[
\normi{v} \le \frac{\Paren{\alpha/(Cp!)}^{1/p}}{n^{1/4}}\,,
\]
	then the algorithm on input $\bm T$ runs in time  $\Paren{n^{p}}^{O(1)}$  and outputs a unit vector $\hat{\bm v} \in \mathbb{R}^n$ satisfying
	\[
	\abs{\iprod{v,\hat{\bm v}}}\ge 0.99
	\]
	with probability at least $1-2^{-n}$. 
	
	Furthermore,  for $\eps \le \Paren{Cn^{p/2} \cdot p\ln n}^{-1}$,  the same result holds if an arbitrary (adversarially chosen) $\epsilon$-fraction of entries of $\bm{T}$ is replaced by adversarially chosen values.
\end{theorem}

\subsection{Sparse PCA with oblivious outliers}\label{section:obliviousSparsePca}
We will use the system of constraints for sparse PCA from \cite{DBLP:conf-focs-dOrsiKNS20}.
  
Let $t\leq k$ and let $\cS_t$ be the set of all $n$-dimensional vectors with values in $\set{0,1}$ that have exactly $t$ nonzero coordinates.

We start with the following definition.

\begin{definition}
	For every $u\in \cS_t$ we define the following polynomial in variables $s:=(s_1,\ldots,s_n)$
	\[
	p_u(s) = \binom{k}{t}^{-1}\cdot \underset{i \in \supp\Set{u}}{\prod} s_i\,.
	\]
\end{definition}

Note that if $v$ is a $k$-sparse vector and $s$ is the indicator of its support, then for every
$u\in \cS_t$, we have
\begin{align*}
p_u(s)=\begin{cases}
\binom{k}{t}^{-1}&\text{ if} \supp\Set{u}\subseteq\supp\Set{v}\,, \\
0& \text{ otherwise}\,.
\end{cases}
\end{align*}

Now consider the following system $\cC_{s,x}$ of polynomial constraints.  
%\Tnote{We have two formulations but the number of constraints is the same.}

\begin{equation}\label{eq:sparseConstraints}
\cC_{s,x}\colon
\left \{
\begin{aligned}
&\forall i\in [n],
& s_i^2
& =s_i \\
&&\textstyle \underset{i \in [n]}{\sum}s_i&=k\\
&\forall i \in [n], &s_i\cdot x_i &=x_i\\
&&\textstyle \underset{i \in [n]}{\sum}x_i^2&=1\\
%\textstyle \underset{u\in \cN_t}{\sum}\;\,\underset{i \in \supp\paren{u}}{\prod}s_i &= \binom{k}{t}\\
&&\underset{u \in \cS_t}{\sum} p_u(s) &=1\\
&\forall i\in [n],
&\underset{u \in \cS_t}{\sum} u_ip_u(s) &=\frac{t}{k} \cdot s_i
%\\&& v\transpose{v} &= \frac{k}{t}\underset{u,u' \in \cN_t}{\sum} u'\transpose{u}P(v)_u P(v)_{u'}
\end{aligned}
\right \}
\end{equation} 

	It is easy to see that if $x$ is $k$-sparse and $s$ is the indicator of its support, then $x$ and $s$ satisfy these constraints.

In \cite{DBLP:conf-focs-dOrsiKNS20} a different model of Sparse PCA is considered than the one we study here. There, a bound on $v^\top \bm M v$ is certified where $\bm M$ is the centered Wishart matrix, while we need to certify the bound for standard Gaussian matrix $\bm M$. The proofs of \cite{DBLP:conf-focs-dOrsiKNS20} can be easily adapted for our case. In \cref{sec:sparse_pca_sos} we show that the Gaussian complexity of the set of degree $4t$ pseudo-distributions that satisfy the constraints $\cC_{s,x}$ is bounded by $O\Paren{k\sqrt{\frac{\log n}{t}}}$. 

Now we are able to show how the algorithm from \cref{thm:technical-meta-theorem} can be used to solve the sparse PCA problem with general noise with symmetric independent entries.

\begin{theorem}\label{theorem:mainSparse}
	Let $n,k\in \N$, $k\le n$, $\lambda > 0$ and $\alpha \in (0,1]$. 
	Let $\bm M = \lambda\cdot vv^\top+ \bm N$,
	 where $v\in \R^n$ is a  $k$-sparse unit vector  and 
	 $\bm N$ is a random matrix with independent (but not necessarily identically distributed)
	 symmetric about zero entries that satisfy $\Pr\Brac{\Abs{\bm N_{ij}} \le 1}\ge \alpha$. 
	
		There exists an absolute constant $C>1$ and an algorithm such that if  
		$\lambda\geq k \geq C\ln(n)/\alpha^2$ and $\normi{v} \le 100/\sqrt{k}$,
	then the algorithm on input $\bm M$ runs in time  $n^{O(\log (n)/\alpha^2)}$
	and outputs $\hat{\bm v} \in \mathbb{R}^n$ satisfying
	\[
	\abs{\iprod{v,\hat{\bm v}}}\ge 0.99
	\]
	 with probability at least $1-n^{-100}$. 
	 
	 Moreover, the same result holds if we only get the upper triangle (without the diagonal) of the matrix $\bm M$ as input.
		 
	Furthermore,  for $\eps \leq \frac{k^2\alpha^2}{C n^2\ln n}$, the same result holds if an arbitrary (adversarially chosen) $\epsilon$-fraction of entries of $\bm{M}$ is replaced by adversarially chosen values.
	\begin{proof}
		We can apply \cref{thm:technical-meta-theorem} for input $$\bm Y = \bm M/\lambda=vv^\top + \bm N'\,,$$
		where $\bm N' = \frac{1}{\lambda}\bm N$. We also use the set $$\pdset = \Set{ \pE_{x\sim \mu} xx^\top : \mu \in \cE}\,,$$ 
where $\cE$ is the set of pseudo-distributions over $\R[x_1,\ldots, x_n]$ of degree 
		$4t$ that satisfy the constraints \cref{eq:sparseConstraints} and additional constraints $\pE x_i^2 \le 100^2/k$ for all $i\in [n]$. Note that we choose $ t=\lceil C\ln(n)/\alpha^2\rceil \leq k$.
		
		If we define $b=\frac{100}{k}$ and $r=1$, it is not hard to see from the constraints \cref{eq:sparseConstraints} and $\pE x_i^2 \le 100^2/k$ for all $i\in [n]$ that  $\max_{X\in \pdset}\normi{X}\leq b$ and $\max_{X\in \pdset}\norm{X}_2\leq r$. Let $h=\frac{201}{k}$, so that
$$h\geq \zeta+2b\,.$$
		
From \cref{lem:sparcePCAGassianComplexity}, we can see that there exists $\gamma = \Theta\Paren{k\sqrt{\frac{\ln n}{t}}}$ such that the Gaussian complexity of $\pdset$ can be bounded by $\gamma$. We conclude from \cref{thm:technical-meta-theorem} that with probability at least $1-n^{-100}$, the pseudo-distribution $\pE$ that minimizes the Huber loss satisfies 
\begin{align*}
\Normt{vv^\top -\pE xx^\top}&\leq O\Paren{\sqrt{{\frac{h}{\alpha}}\Paren{\gamma + r\sqrt{\log(1/n^{-100})}}}}\\
&\leq O\Paren{\sqrt{{\frac{201}{k \alpha}}\Paren{k\sqrt{\frac{\ln n}{t}} + 1\sqrt{100\log n}}}}\\
&\leq O\Paren{\sqrt{\frac{1}{\alpha}\sqrt{\frac{\ln n}{t}} + \frac{\sqrt{\log n}}{k \alpha}}}\leq O\Paren{\sqrt{\frac{1}{\sqrt{C}}+\frac{\alpha}{C\sqrt{\log n}}}}\\
&\leq O\Paren{\frac{1}{C^{1/4}}}\,,
\end{align*}
where the last inequality follows from the fact that $ t\geq C\ln(n)/\alpha^2$ and $k\geq C\ln(n)/\alpha^2$. If we choose $C>1$ to be large enough, we can make the above bound on $\Normt{vv^\top -\pE xx^\top}$ to be arbitrarily small. Therefore, by \cref{lem:vectorRecoveryEven}, the top eigenvector $\hat{\bm v}$ of $\pE xx^\top$ satisfies the desired bound with probability at least $1-n^{-100}$. Note that Theorem \cref{thm:technical-meta-theorem} also implies that we can also afford an $\epsilon$ fraction of arbitrary adversarial changes in the observed matrix as long as
\begin{align*}
\epsilon\leq \frac{\gamma^2}{r^2n^2\log (n^2)} =\frac{O\Paren{k^2\frac{\ln n}{t}}}{2n^2\log (n)} = O\Paren{\frac{k^2}{tn^2}} =O\Paren{\frac{k^2\alpha^2}{n^2\ln(n) }} \,.
\end{align*}		
		
		If we only get the upper triangle of $\bm M$ as input, we can optimize the Huber loss over  
		$$\pdset = \Set{ \Paren{\pE_{x\sim \mu} x_ix_j}_{i<j} : \mu \in \cE}\,.$$
		The Gaussian complexity of $\pdset$ is bounded\footnote{We use \cref{fact:lipschitz_transformation_complexity} which implies that the Gaussian complexity does not increase if we restrict to a subset of coordinates.} by $O\Paren{k\sqrt{\frac{\ln n}{t}}}$,
		 hence the pseudo-distribution $\pE$ that minimizes the Huber loss satisfies
		 \[
		 		 \sum_{i\neq j}\Paren{v_iv_j - \pE x_ix_j}^2 \le 2\sum_{1\le i < j\le n}\Paren{v_iv_j - \pE x_ix_j}^2 \leq O\Paren{\frac{1}{\sqrt{C}}}
		 \]
		 with probability at least $1-n^{-100}$. Moreover,
		 \[
		 \sum_{1\le i\le n} \Paren{v_i^2 - \pE x_i^2}^2 \le 
		 \max_{1\le j\le n}\Paren{v_j^2 + \pE x_j^2} \sum_{1\le i\le n} \Paren{v_i^2 +  \pE x_i^2} 
		 \le O(1/k)\,.
		 \]
		 
		 Hence, with probability at least $1-n^{-100}$, we have
		 \[
		 \Normt{vv^\top -\pE xx^\top} \leq O\Paren{\sqrt{\frac{1}{\sqrt{C}}+\frac{1}{k}}}\,.
		 \]
		 The remaining of the proof is the same as when the input is the whole matrix $\bm M$.
	\end{proof}
\end{theorem}

\section{Reduction from the planted clique problem}\label{section:LowerBounds}
\subsection{Sparse PCA}\label{section:lowerboundSparsePca}
In this section we show that the running time $n^{O(\log n)}$ for sparse PCA with symmetric noise is likely to be inherent. We will use a reduction from the planted clique problem. Reductions from the planted clique problem to different models of sparse PCA were studied in \cite{DBLP:conf/colt/BerthetR13, DBLP:journals/corr/abs-1304-0828, WBS16, GMZ17, BBH18, DBLP:conf/colt/BrennanB19}. Our analysis is simpler since our noise model is less restrictive than models considered in prior works. In fact, the planted clique problem can be seen as a special case of sparse PCA with symmetric noise (when only upper triangle without the diagonal is given as input).

Recall that the instance of planted clique problem is a random graph sampled according to the following distribution $G(n, q, k)$: 
First, some graph is sampled from Erd\H{o}s-R\'enyi distribution $G(n,q)$ (where $q\in(0,1)$ is the probability of including an edge), 
and then a random subset of vertices of size $k \le n$ is chosen and the clique corresponding to these vertices is added to the graph. 
The goal is to find the clique. It is possible to find in time $n^{O(\log n)}$ for constant\footnote{I.e., $\Omega(1)\leq q\leq 1-\Omega(1)$.} $q$ if $k \ge \omega(\log n)$, 
but no polynomial time algorithm is known for $\omega(\log n) \le k \le o(\sqrt{n})$.

In this section we assume that $\omega(\log n) < k < n^{0.49}$. 
Currently no $n^{o(\log n)}$-time algorithm is known to solve this problem in this regime (for constant $q$), and for $q=1/2$ and for some $k = n^{\Omega(1)}$ it is conjectured to be impossible to solve it in time $n^{o(\log n)}$ (see \cite{planted_clique_conjecture} for more details). 

Let $\bm M = \lambda  \cdot vv^\top + \bm N$, 
 where $v$ is a $k$-sparse unit vector whose nonzero entries are equal to $1/\sqrt{k}$, 
 $\bm N$ is a random matrix with independent (but not necessarily identically distributed) entries 
 that satisfy $\Pr\Brac{\Abs{\bm N_{ij}} \le 1} = 1$.

Also suppose that we get only the upper triangle (without the diagonal) of the matrix $\bm M$ as input. 
There are algorithms that can solve sparse PCA problem that only observe the upper triangle and match (up to a constant factor) current best known guarantees (if $\norm{v}_4^4 \ll 1$ which is true if $\normi{v}\le 100/\sqrt{k}$). 
Hence we assume that for flat vectors the problem does not become harder if we get only the upper triangle of $\bm M$ as input. We denote the upper triangle matrix of $\bm M$ as $\cU(\bm M)$.

Now let $\bm G \sim G(n,1/2,k)$ be a random graph with a planted clique of size $k$. Let $\bm A$ be the adjacency matrix of $\bm G$. Let $J$ be the matrix with all entries equal to $1$ and let $\bm C = 2\bm A-J$. Note that $\cU(\bm C) = \cU\Paren{k \cdot vv^\top + \bm N}$, where $\sqrt{k}\cdot v$ is the indicator vector of the vertices of the clique (so it is $k$-sparse), and $\bm N$ is the noise whose entries that correspond to the vertices of the clique are zero, and other entries are iid $\{\pm1\}$.

If we could recover $v$ from $\cU(k \cdot vv^\top + \bm N)$ in time $n^{o\Paren{\log n}}$, we would be able to find the planted clique in $\bm G \sim G(n,1/2,k)$ in time $n^{o\Paren{\log n}}$.

Moreover, we can make the noise even smaller and the problem is likely to remain hard. That is, if we could recover $v$ from $\cU(k \cdot vv^\top + \bm N)$ where for all $(i,j)\in[n]^2$, $\Pr\Brac{\bm N_{ij} = 0} \ge \alpha$, then we would be able to find the planted clique in $G(n, (1-\alpha)/2, k)$. 
Indeed, for $p= (1-\alpha)/2$ let $\bm G \sim G(n,p,k)$ and let $\bm A$ be the adjacency matrix of $\bm G$. Let $\bm B$ be a random matrix such that $\bm B_{ij}=0$ for all $(i,j)$ such that $\bm A_{ij} =1$, and for other $(i,j)$, $\bm B_{ij}$ is $0$ with probability  $p/(1-p)$ and $0.5$ with probability $1 - p/(1-p) = \alpha/(1-p)$. Let $J$ be the matrix with all entries equal to $1$ and let $\bm C = 2\bm A + 2\bm B -J$. Note that $\cU(\bm C) = \cU\Paren{k \cdot vv^\top + \bm N}$, where $\sqrt{k}\cdot v$ is the indicator vector of the vertices of the clique, and $\bm N$ is the noise matrix with independent entries that satisfy  $\Pr\Brac{\bm N_{ij} = 0} \ge \alpha$.

Hence if we could recover $v$ from $\cU(k \cdot vv^\top + \bm N)$, where $\Pr\Brac{\bm N_{ij} = 0} \ge 0.99$ in time $n^{o\Paren{\log n}}$, 
then we could find the planted clique in $\bm G \sim G(n,0.005,k)$ in time $n^{o\Paren{\log n}}$, which currently known algorithms cannot do.

\begin{remark}
Exact recovery of $v$ by the sparse PCA algorithm is not necessary in order for the reduction to work: As we shall see, if we only get unit $\hat{\bm v}$ 
	that has correlation $\rho = \Omega(1)$ with $v$, we can still find the clique. First notice that since $\sum_{i\in[n]}|v_i|\hat{\bm v}_i\geq \sum_{i\in[n]}v_i\hat{\bm v}_i$, we can assume without loss of generality that the entries of $\hat{\bm v}$ are nonnegative. Now consider the set $S\subseteq[n]$ containing the indices the of top $4k/\rho^2$ entries of $\hat{\bm v}$. Then 
	$$\sum_{i\notin S} v_i \hat{\bm v}_i =\sum_{i\in \supp(v)\setminus S} \frac{1}{\sqrt{k}}\hat{\bm v}_i \stackrel{(\ast)}{\le} k\cdot\frac{1}{\sqrt{k}}\cdot\frac{\rho}{2\sqrt{k}}=\rho/2\,,$$
	where $(\ast)$ follows from the fact that $|\supp(v)\setminus S|\leq |\supp(v)|=k$ and that for every $i\notin S$, we have\footnote{This is a consequence of $\sum_{i\in[n]}\hat{\bm v}_i^2=1$ and the fact that $S\subseteq[n]$ contains the indices the of top $4k/\rho^2$ entries of $\hat{\bm v}$.} $\hat{\bm v}_i^2\leq \frac{\rho^2}{4k}$. 
	We conclude that
	$$\sum_{i\in S} v_i \hat{\bm v}_i =\iprod{v,\hat{\bm v}}-\sum_{i\notin S} v_i \hat{\bm v}_i\ge \rho/2\,.$$ 
	
	Now let $S' = S\cap \supp(v)$. We have
	$$\sqrt{k}\rho/2\leq \sqrt{k}\sum_{i\in S} v_i \hat{\bm v}_i= \sqrt{k}\sum_{i\in S\cap\supp(v)} \frac{1}{\sqrt{k}} \hat{\bm v}_i= \sum_{i\in S'}\hat{\bm v}_i \le \sqrt{\card{S'}}\,,$$
	i.e., $$\card{S'} \ge k\rho^2/4\,.$$ 
	So if we restrict the graph to the vertices corresponding to $S$, 
	we can find\footnote{We apply the well-known spectral algorithm for the planted clique problem. It is worth mentioning that the $\log n$ factor comes from the fact that $S$ is not independent from the graph, and hence the distribution of the graph that is induced by the vertices in $S$ does not exactly match that of the planted clique problem. By taking a union bound over all sets of size $|S|$ and using standard concentration bounds for the spectral norm of symmetric matrices with i.i.d. subgaussian entries, one can show that the maximal spectral norm among all submatrices of the centered adjacency matrix of the random graph is bounded by $O(|S|\log n)$ with high probability.} 
	the clique corresponding to $S'$ in polynomial time as long as $\card{S'} \ge \omega\Paren{\sqrt{\card{S}\log n}}$, which is true for $\rho = \Omega(1)$ and $k \ge \omega(\log n)$. Then we can then easily find the remaining of the clique by searching for all vertices that are adjacent to every vertex in $S'$.
	
%	
%	Consider some function $f(n)$ such that $\omega\Paren{\sqrt{\log n}}\le f(n) \le o\Paren{\rho\sqrt{k}}$. 
%	Let $v'_i = \hat{\bm v}_i \cdot \ind{\hat{\bm v}_i \ge f(n)/k}$. The vector $v'$ has norm at most $1$ and $\iprod{v,v'} \ge (1-o(1))\rho$.
%	Exact recovery of $v$ by the sparse PCA algorithm is not necessary: if we only get unit $\hat{\bm v}$ 
%	that has correlation $\rho = \omega\Paren{\sqrt{\log(n)/k}}$ with $v$, we can still find the clique. Indeed, without loss of generality we can assume that the entries of $\hat{\bm v}$ are nonnegative. 
%	Consider some function $f(n)$ such that $\omega\Paren{\sqrt{\log n}}\le f(n) \le o\Paren{\rho\sqrt{k}}$. 
%	Let $v'_i = \hat{\bm v}_i \cdot \ind{\hat{\bm v}_i \ge f(n)/k}$. The vector $v'$ has norm at most $1$ and $\iprod{v,v'} \ge (1-o(1))\rho$.
%%	Then $\sum_{i\in S}\hat{\bm v}_iv_i \le o(\rho)$, where $S$ is the set of top $f(n)$ entries of $\hat{\bm v}$, and $\sum_{i\notin S'}\hat{\bm v}_iv_i \le o(\rho)$, where  $S'$ is the set of top $k^2/f^2(n)$ entries. 
%%	Note that the of $\hat{\bm v}$ entries from $S'\cap\Paren{\supp(v)\setminus S}$ have magnitude at most $1/\sqrt{f(n)} = o\Paren{1/\sqrt{\log n}}$, 
%%	but the sum of these entries is at least $\rho \sqrt{k} = \omega\Paren{\sqrt{\log n}}$, hence this set has size at least  
%%	$\rho \sqrt{kf(n)} =\omega\Paren{{\log n}}$.
\end{remark}

\subsection{Tensor PCA}\label{section:lowerboundTensorPca}
In this section, we provide evidence that the assumption on $\normi{v}$ in \cref{theorem:mainTensor} is likely to be inherent, at least for our SoS-based approach.

First, we notice that exactly the same reasoning as that of \Cref{section:lowerboundSparsePca} for obtaining a reduction from the planted clique problem to sparse PCA,
can also be applied to get a reduction from the problem of recovering a planted clique in a random $p$-hypergraph to the problem of 
recovering a $k$-sparse unit vector $v$ from the upper simplex of 
$\bm Y = k^{p/2}\cdot v^{\otimes p} + \bm N$, i.e., from the entries $\bm Y_{i_1,\ldots,i_p}$ such that $i_1 < \ldots < i_p$. 
It is conjectured that  for every constant $p$, if $ k < n^{0.49}$, then the problem of recovering a planted clique in a random $p$-hypergraph cannot be solved in polynomial time  (see \cite{planted_clique_hypergraph} for more details). Hence, we expect that if $ k < n^{0.49}$ then it is not possible to efficiently recover a $k$-sparse unit vector $v$ from the upper simplex of $\bm Y = k^{p/2}\cdot v^{\otimes p} + \bm N$.

Second, we show that recovering from the upper simplex is not harder than recovering from the entire tensor $\bm Y = k^{p/2}\cdot v^{\otimes p} + \bm N$, at least for the algorithmic approach that is provided in \cref{theorem:mainTensor}. We proceed similarly to how we showed in \cref{theorem:mainSparse} that recovering from the upper triangle matrix (without the diagonal) is not harder than recovering from the entire matrix. More precisely, we show that if it is possible to get a sum-of-squares certificate of the bound on the Gaussian complexity in such a way that shows that the algorithm in \cref{theorem:mainTensor} can recover a $k$-sparse vector $v$ from the entire $\bm Y$, then by slightly modifying the algorithm in \cref{theorem:mainTensor} we can also recover $v$ from the upper simplex of $\bm Y$.

Let us start by considering the case $p=3$. Let $b=O(1/\sqrt{k})$ be a bound on the entries of $v$. Since we know that $v$ is $k$-sparse, we can restrict the optimization problem in the algorithm of \cref{theorem:mainTensor} to the pseudodistributions satisfying the constraint $\sum_{j=1}^n \abs{\pE x_j} \le bk$. Similarly to the proof of \cref{theorem:mainSparse}, we notice that

\begin{align*}
\sum_{1\le i\le n} \Paren{v_i^3- \pE x_i^3}^2 \le 
\max_{1\le j\le n}\Paren{\abs{v_j^3} +\abs{\pE x_j^3}} \sum_{1\le i\le n} \Paren{\abs{v_i^3}+  \abs{\pE x_i^3}}
\stackrel{(\ast)}{\le} O(kb^6)\leq o(1)\,,
\end{align*}
and
\begin{align*}
 \sum_{1\le i, j\le n} \Paren{v_i^2v_j - \pE x_i^2x_j}^2 &\le  \max_{1\le i', j'\le n}
\Paren{\Abs{v_{i'}^2v_{j'}} + \Abs{\pE x_{i'}^2x_{j'}}}\cdot\sum_{1\le i,j\le n}\Paren{\Abs{v_i^2v_j} + \Abs{\pE x_i^2x_j}}\\
&\leq 2b^3\sum_{1\le j\le n}\Paren{\Abs{v_j} + \Abs{\pE x_j}}
\stackrel{(\dagger)}{\le} O(k b^{4}) \le o(1)\,,
\end{align*}
where $(\ast)$ and $(\dagger)$ follow from the constraint $\sum_{j=1}^n \abs{\pE x_j} \le bk$. The remaining of the proof is similar to the proof of \cref{theorem:mainSparse}.

For general $p$-order tensors with $3\leq p \leq O(1)$, we can get a similar bound if we add the constraints  
$\sum_{1\le i_1,\ldots, i_r \le n}\abs{\pE x_{i_1}\cdots x_{i_r}} \le b^rk^r$ for all $1\le r \le p-2$, from which we can deduce that
\begin{align*}
\sum_{\substack{1\le i_1,\ldots,i_p\le n:\\\exists  j\neq j',~ i_j=i_{j'}}} \Paren{(v^{\otimes p})_{i_1\ldots i_p}- \pE (x^{\otimes p})_{i_1\ldots i_p}}^2 \le o(1)\,.
\end{align*}

%Then for all $1 \le q \le \lfloor p/2\rfloor$ we can show a bound of the form
%\[
%f(p,q) \cdot b^p \cdot 2\sum_{1\le i_1,\ldots, i_{p-2q} \le n}\abs{\pE x_{i_1}} \le O\Paren{b^p \cdot b^{p-2q} k^{p-2q}} \le O(b^{2q}) \le o(1)\,,
%\]
%where $f(p,q) $ is some quantity  that depends only on $p$ and $q$.

The above implies that for pseudodistributions that satisfy the added constraints, we have
\begin{equation}
\label{eq:constraintsImplyNormEquiv}
\normt{v^{\otimes p}-\pE x^{\otimes p}}^2\leq  o(1) + p!\cdot \sum_{1\le i_1<\ldots<i_p\le n} \Paren{(v^{\otimes p})_{i_1\ldots i_p}- \pE (x^{\otimes p})_{i_1\ldots i_p}}^2\,.
\end{equation}

If we could get a sum-of-squares certificate of the bound on the Gaussian complexity showing that our degree-$\ell$ SoS-based algorithm in \cref{theorem:mainTensor} can recover $k$-sparse flat vectors from $\bm Y = k^{p/2}\cdot v^{\otimes p} + \bm N$, 
then the same bound would imply that in the case where we only observe the upper simplex of $\bm Y$, the algorithm of \cref{theorem:mainTensor} restricted to pseudoexpectations on the upper simplex would give\footnote{This is mainly because of \cref{fact:lipschitz_transformation_complexity} which implies that the Gaussian complexity does not increase if we restrict to a subset of coordinates.} a pseudodistribution satisfying
$$\sum_{1\le i_1<\ldots<i_p\le n}\Paren{(v^{\otimes p})_{i_1\ldots i_p}- \pE (x^{\otimes p})_{i_1\ldots i_p}}^2\ll 1\,.$$
If we also require that the pseudodistributions satisfy the constraints $\sum_{1\le i_1,\ldots, i_r \le n}\abs{\pE x_{i_1}\cdots x_{i_r}} \le b^rk^r$ for all $1\le r \le p-2$, then \eqref{eq:constraintsImplyNormEquiv} implies that we can get a pseudodistribution satisfying
$$\normt{v^{\otimes p}-\pE x^{\otimes p}}^2\ll 1\,,$$
from which we can recover $v$. Since we know that this is not likely to be possible if $k\leq n^{0.49}$, we can see that the assumption on $\normi{v}$ in \cref{theorem:mainTensor} is likely to be inherent, at least for our SoS-based approach.

% BIBLIOGRAPHY
\newpage

% assumes hyperref
\phantomsection
\addcontentsline{toc}{section}{References}
\bibliographystyle{amsalpha}
\bibliography{bib/custom,bib/dblp,bib/mathreview,bib/scholar,bib/oldOblivious}

\newcommand{\etalchar}[1]{$^{#1}$}
\providecommand{\bysame}{\leavevmode\hbox to3em{\hrulefill}\thinspace}
\providecommand{\MR}{\relax\ifhmode\unskip\space\fi MR }
\providecommand{\MRhref}[2]{%
  \href{http://www.ams.org/mathscinet-getitem?mr=#1}{#2}
}
\providecommand{\href}[2]{#2}
\begin{thebibliography}{CLMW11}

\bibitem[AY21]{tpca_iterations}
Arnab Auddy and Ming Yuan, \emph{On estimating rank-one spiked tensors in the
  presence of heavy tailed errors}, CoRR \textbf{abs/2107.09660} (2021).

\bibitem[BB19]{DBLP:conf/colt/BrennanB19}
Matthew~S. Brennan and Guy Bresler, \emph{Optimal average-case reductions to
  sparse {PCA:} from weak assumptions to strong hardness}, Conference on
  Learning Theory, {COLT} 2019, 25-28 June 2019, Phoenix, AZ, {USA} (Alina
  Beygelzimer and Daniel Hsu, eds.), Proceedings of Machine Learning Research,
  vol.~99, {PMLR}, 2019, pp.~469--470.

\bibitem[BBH18]{BBH18}
Matthew Brennan, Guy Bresler, and Wasim Huleihel, \emph{Reducibility and
  computational lower bounds for problems with planted sparse structure},
  Proceedings of the 31st Conference On Learning Theory (Sébastien Bubeck,
  Vianney Perchet, and Philippe Rigollet, eds.), Proceedings of Machine
  Learning Research, vol.~75, PMLR, 06--09 Jul 2018, pp.~48--166.

\bibitem[BDJ{\etalchar{+}}22]{DBLP:conf-stoc-BakshiDJKKV22}
Ainesh Bakshi, Ilias Diakonikolas, He~Jia, Daniel~M. Kane, Pravesh~K. Kothari,
  and Santosh~S. Vempala, \emph{Robustly learning mixtures of \emph{k}
  arbitrary gaussians}, {STOC} '22: 54th Annual {ACM} {SIGACT} Symposium on
  Theory of Computing, Rome, Italy, June 20 - 24, 2022, 2022, pp.~1234--1247.

\bibitem[BJKK17]{DBLP:conf/nips/Bhatia0KK17}
Kush Bhatia, Prateek Jain, Parameswaran Kamalaruban, and Purushottam Kar,
  \emph{Consistent robust regression}, {NIPS}, 2017, pp.~2107--2116.

\bibitem[BR13a]{DBLP:conf/colt/BerthetR13}
Quentin Berthet and Philippe Rigollet, \emph{Complexity theoretic lower bounds
  for sparse principal component detection}, {COLT}, {JMLR} Workshop and
  Conference Proceedings, vol.~30, JMLR.org, 2013, pp.~1046--1066.

\bibitem[BR13b]{DBLP:journals/corr/abs-1304-0828}
\bysame, \emph{Computational lower bounds for sparse {PCA}}, CoRR
  \textbf{abs/1304.0828} (2013).

\bibitem[Cd22]{DBLP:conf-colt-Chend22}
Hongjie Chen and Tommaso d'Orsi, \emph{On the well-spread property and its
  relation to linear regression}, Conference on Learning Theory, 2-5 July 2022,
  London, {UK}, 2022, pp.~3905--3935.

\bibitem[CdO21]{sparse_tensor_pca}
Davin Choo and Tommaso d\textquotesingle Orsi, \emph{The complexity of sparse
  tensor pca}, Advances in Neural Information Processing Systems (M.~Ranzato,
  A.~Beygelzimer, Y.~Dauphin, P.S. Liang, and J.~Wortman Vaughan, eds.),
  vol.~34, Curran Associates, Inc., 2021, pp.~7993--8005.

\bibitem[CLMW11]{DBLP:journals/jacm/CandesLMW11}
Emmanuel~J. Cand{\`{e}}s, Xiaodong Li, Yi~Ma, and John Wright, \emph{Robust
  principal component analysis?}, J. {ACM} \textbf{58} (2011), no.~3,
  11:1--11:37.

\bibitem[DdNS21]{DBLP:conf-focs-DingdNS21}
Jingqiu Ding, Tommaso d'Orsi, Rajai Nasser, and David Steurer, \emph{Robust
  recovery for stochastic block models}, 62nd {IEEE} Annual Symposium on
  Foundations of Computer Science, {FOCS} 2021, Denver, CO, USA, February 7-10,
  2022, 2021, pp.~387--394.

\bibitem[DHS20]{tpca_walks}
Jingqiu Ding, Samuel~B. Hopkins, and David Steurer, \emph{Estimating rank-one
  spikes from heavy-tailed noise via self-avoiding walks}, Proceedings of the
  34th International Conference on Neural Information Processing Systems (Red
  Hook, NY, USA), NIPS'20, Curran Associates Inc., 2020.

\bibitem[dKNS20]{DBLP:conf-focs-dOrsiKNS20}
Tommaso d'Orsi, Pravesh~K. Kothari, Gleb Novikov, and David Steurer,
  \emph{Sparse {PCA:} algorithms, adversarial perturbations and certificates},
  61st {IEEE} Annual Symposium on Foundations of Computer Science, {FOCS} 2020,
  Durham, NC, USA, November 16-19, 2020, 2020, pp.~553--564.

\bibitem[DKWB19]{bandeira_spca}
Yunzi Ding, Dmitriy Kunisky, Alexander~S. Wein, and Afonso~S. Bandeira,
  \emph{Subexponential-time algorithms for sparse pca}, 2019.

\bibitem[dLN{\etalchar{+}}21]{DBLP:conf-nips-dOrsiLNNST21}
Tommaso d'Orsi, Chih{-}Hung Liu, Rajai Nasser, Gleb Novikov, David Steurer, and
  Stefan Tiegel, \emph{Consistent estimation for {PCA} and sparse regression
  with oblivious outliers}, Advances in Neural Information Processing Systems
  34: Annual Conference on Neural Information Processing Systems 2021, NeurIPS
  2021, December 6-14, 2021, virtual, 2021, pp.~25427--25438.

\bibitem[DM16]{DBLP:journals/jmlr/DeshpandeM16}
Yash Deshpande and Andrea Montanari, \emph{Sparse {PCA} via covariance
  thresholding}, Journal of Machine Learning Research \textbf{17} (2016),
  141:1--141:41.

\bibitem[dNS21]{ICML-linear-regression}
Tommaso d'Orsi, Gleb Novikov, and David Steurer, \emph{Consistent regression
  when oblivious outliers overwhelm}, 38th International Conference on Machine
  Learning, {ICML} 2021, 18-24 July 2021, Virtual Event, Proceedings of Machine
  Learning Research, {PMLR}, 2021.

\bibitem[FP07]{feral2007largest}
Delphine F{\'e}ral and Sandrine P{\'e}ch{\'e}, \emph{The largest eigenvalue of
  rank one deformation of large wigner matrices}, Communications in
  mathematical physics \textbf{272} (2007), no.~1, 185--228.

\bibitem[GLS81]{MR625550-Grotschel81}
M.~Gr\"otschel, L.~Lov\'asz, and A.~Schrijver, \emph{The ellipsoid method and
  its consequences in combinatorial optimization}, Combinatorica \textbf{1}
  (1981), no.~2, 169--197. \MR{625550}

\bibitem[GMZ17]{GMZ17}
Chao Gao, Zongming Ma, and Harrison~H. Zhou, \emph{Sparse cca: Adaptive
  estimation and computational barriers}, The Annals of Statistics \textbf{45}
  (2017), no.~5, 2074--2101.

\bibitem[HKP{\etalchar{+}}17]{HopkinsKPRSS17}
Samuel~B. Hopkins, Pravesh~K. Kothari, Aaron Potechin, Prasad Raghavendra,
  Tselil Schramm, and David Steurer, \emph{The power of sum-of-squares for
  detecting hidden structures}, 58th {IEEE} Annual Symposium on Foundations of
  Computer Science, {FOCS} 2017, Berkeley, CA, USA, October 15-17, 2017, 2017,
  pp.~720--731.

\bibitem[Hop20]{hopkins2020mean}
Samuel~B Hopkins, \emph{Mean estimation with sub-gaussian rates in polynomial
  time}, The Annals of Statistics \textbf{48} (2020), no.~2, 1193--1213.

\bibitem[HSS15]{tensor_pca_sos}
Samuel~B. Hopkins, Jonathan Shi, and David Steurer, \emph{Tensor principal
  component analysis via sum-of-square proofs}, Proceedings of The 28th
  Conference on Learning Theory (Paris, France) (Peter Grünwald, Elad Hazan,
  and Satyen Kale, eds.), Proceedings of Machine Learning Research, vol.~40,
  PMLR, 03--06 Jul 2015, pp.~956--1006.

\bibitem[JL09]{johnstone2009consistency}
Iain~M Johnstone and Arthur~Yu Lu, \emph{On consistency and sparsity for
  principal components analysis in high dimensions}, Journal of the American
  Statistical Association \textbf{104} (2009), no.~486, 682--693.

\bibitem[KKM18]{DBLP:conf-colt-KlivansKM18}
Adam~R. Klivans, Pravesh~K. Kothari, and Raghu Meka, \emph{Efficient algorithms
  for outlier-robust regression}, Conference On Learning Theory, {COLT} 2018,
  Stockholm, Sweden, 6-9 July 2018, 2018, pp.~1420--1430.

\bibitem[KWB19]{ldp_lowerbound_tpca}
Dmitriy Kunisky, Alexander~S. Wein, and Afonso~S. Bandeira, \emph{Notes on
  computational hardness of hypothesis testing: Predictions using the
  low-degree likelihood ratio}, 2019.

\bibitem[Las01]{MR1846160-Lasserre01}
Jean~B. Lasserre, \emph{New positive semidefinite relaxations for nonconvex
  quadratic programs}, Advances in convex analysis and global optimization
  ({P}ythagorion, 2000), Nonconvex Optim. Appl., vol.~54, Kluwer Acad. Publ.,
  Dordrecht, 2001, pp.~319--331. \MR{1846160}

\bibitem[Lau09]{Laurent2009SumsOS}
Monique Laurent, \emph{Sums of squares, moment matrices and optimization over
  polynomials}, 2009.

\bibitem[LZ20]{planted_clique_hypergraph}
Yuetian Luo and Anru~R Zhang, \emph{Open problem: Average-case hardness of
  hypergraphic planted clique detection}, Proceedings of Thirty Third
  Conference on Learning Theory (Jacob Abernethy and Shivani Agarwal, eds.),
  Proceedings of Machine Learning Research, vol. 125, PMLR, 09--12 Jul 2020,
  pp.~3852--3856.

\bibitem[MR14]{DBLP:journals/corr/MontanariR14a}
Andrea Montanari and Emile Richard, \emph{A statistical model for tensor
  {PCA}}, CoRR \textbf{abs/1411.1076} (2014).

\bibitem[MRS21]{planted_clique_conjecture}
Pasin Manurangsi, Aviad Rubinstein, and Tselil Schramm, \emph{The strongish
  planted clique hypothesis and its consequences}, 12th Innovations in
  Theoretical Computer Science Conference, {ITCS} 2021, January 6-8, 2021,
  Virtual Conference (James~R. Lee, ed.), LIPIcs, vol. 185, Schloss Dagstuhl -
  Leibniz-Zentrum f{\"{u}}r Informatik, 2021, pp.~10:1--10:21.

\bibitem[Nes00]{MR1748764-Nesterov00}
Yurii Nesterov, \emph{Squared functional systems and optimization problems},
  High performance optimization, Appl. Optim., vol.~33, Kluwer Acad. Publ.,
  Dordrecht, 2000, pp.~405--440. \MR{1748764}

\bibitem[Par00]{parrilo2000structured}
Pablo~A Parrilo, \emph{Structured semidefinite programs and semialgebraic
  geometry methods in robustness and optimization}, Ph.D. thesis, California
  Institute of Technology, 2000.

\bibitem[PJL20]{pensia2020robust}
Ankit Pensia, Varun Jog, and Po-Ling Loh, \emph{Robust regression with
  covariate filtering: Heavy tails and adversarial contamination}, arXiv
  preprint arXiv:2009.12976 (2020).

\bibitem[PWB20]{PWB20}
Amelia Perry, Alexander~S. Wein, and Afonso~S. Bandeira, \emph{{Statistical
  limits of spiked tensor models}}, Annales de l'Institut Henri Poincaré,
  Probabilités et Statistiques \textbf{56} (2020), no.~1, 230 -- 264.

\bibitem[SBRJ19]{SuggalaBR019}
Arun~Sai Suggala, Kush Bhatia, Pradeep Ravikumar, and Prateek Jain,
  \emph{Adaptive hard thresholding for near-optimal consistent robust
  regression}, Conference on Learning Theory, {COLT} 2019, 25-28 June 2019,
  Phoenix, AZ, {USA}, 2019, pp.~2892--2897.

\bibitem[Sho87]{MR939596-Shor87}
N.~Z. Shor, \emph{Quadratic optimization problems}, Izv. Akad. Nauk SSSR Tekhn.
  Kibernet. (1987), no.~1, 128--139, 222. \MR{939596}

\bibitem[SZF20]{sun2020adaptive}
Qiang Sun, Wen-Xin Zhou, and Jianqing Fan, \emph{Adaptive huber regression},
  Journal of the American Statistical Association \textbf{115} (2020), no.~529,
  254--265.

\bibitem[TJSO14]{tsakonas2014convergence}
Efthymios Tsakonas, Joakim Jald{\'e}n, Nicholas~D Sidiropoulos, and Bj{\"o}rn
  Ottersten, \emph{Convergence of the huber regression m-estimate in the
  presence of dense outliers}, IEEE Signal Processing Letters \textbf{21}
  (2014), no.~10, 1211--1214.

\bibitem[Wai19]{wainwright_2019}
Martin~J. Wainwright, \emph{High-dimensional statistics: A non-asymptotic
  viewpoint}, Cambridge Series in Statistical and Probabilistic Mathematics,
  Cambridge University Press, 2019.

\bibitem[WBS16]{WBS16}
Tengyao Wang, Quentin Berthet, and Richard~J. Samworth, \emph{{Statistical and
  computational trade-offs in estimation of sparse principal components}}, The
  Annals of Statistics \textbf{44} (2016), no.~5, 1896 -- 1930.

\bibitem[ZLW{\etalchar{+}}10]{DBLP:conf/isit/ZhouLWCM10}
Zihan Zhou, Xiaodong Li, John Wright, Emmanuel~J. Cand{\`{e}}s, and Yi~Ma,
  \emph{Stable principal component pursuit}, {ISIT}, {IEEE}, 2010,
  pp.~1518--1522.

\end{thebibliography}

\appendix

% APPENDIX
%\section{Facts about Huber loss}
%\label{sec:Huber-loss}

\section{Additional tools}\label{section:additionalTools}

This section contain tools used throughout the rest of the paper.

%
%\begin{lemma}\label{lem:normFTensorDiff}\Gnote{This lemma is perhaps not needed}
%	Let $x,v\in \Set{-\frac{1}{\sqrt{n}}, +\frac{1}{\sqrt{n}}}^n$ and let $\Delta:= \tensorpower{x}{3}-\tensorpower{x}{3}$. Then $\Normt{\Delta}$ is related to the distance between $x$ and $v$ as follows:
%	\begin{equation*}
%		\frac{1}{\sqrt{2}}\norm{v-x}\leq \normt{\Delta}\leq 3\norm{v-x}\,.
%	\end{equation*}
%\end{lemma}
%\begin{proof}
%	For the upper bound, we have
%	\begin{align*}
%	\normt{\Delta}&=\Normt{v^{\otimes 3} - x^{\otimes 3}}\\
%		&\leq \Normt{v^{\otimes 3} - x\otimes v^{\otimes 2}}+\Normt{x\otimes v^{\otimes 2}-x^{\otimes 2}\otimes v}+\Normt{x^{\otimes 2}\otimes v - x^{\otimes 3}}\\
%		&= \norm{v-x}\cdot\norm{v}^2+\norm{x}\cdot\norm{v-x}\cdot\norm{v}+\norm{x}^2\cdot\norm{v-x}\\
%		&= 3\norm{v-x}\,,
%	\end{align*}
%	where in the last equality we used the fact that $v$ and $x$ are unit vectors. For the lower bound, we have
%	\begin{align*}
%		\normt{\Delta}^2&=\Normt{v^{\otimes 3}}^2 + \Normt{x^{\otimes 3}}^2 - 2\iprod{v^{\otimes 3},x^{\otimes 3}}=2 - 2\iprod{v,x}^3\\
%		&= (2-2\iprod{v,x})\cdot\left(1+\iprod{v,x}+\iprod{v,x}^2\right)\\
%		&\stackrel{\text{(a)}}\geq \left(\norm{v}^2 + \norm{x}^2-2\iprod{v,x}\right)\cdot\frac{1}{2}\\
%		&= \frac{1}{2}\cdot \norm{v-x}^2\,,
%	\end{align*}
%	where (a) follows from the fact that $1+t+t^2\geq \frac{1}{2}$ for every $-1\leq t\leq 1$, and the fact that $v$ and $x$ are unit vectors.
%\end{proof}

\begin{fact}\cite{wainwright_2019}\label{fact:spectral_norm_gaussian}
	Let $\bm W$ be an $n\times d$ random matrix with iid standard Gaussian entries. 
	 Then with probability at least $1-\exp\Paren{-\tau/2}$, we have
	\[
	\Norm{\bm W}\le \sqrt{n}  + \sqrt{d} + \sqrt{\tau}\,.
	\]
\end{fact}

\begin{fact}\cite{wainwright_2019}\label{fact:eps-net-unit-ball}
	Let $B$ be a unit $m$-dimensional Euclidean ball. Then for every $\eps \in(0,1)$ there exists an $\eps$-net in $B$ of size $(3/\eps)^m$.
\end{fact}

\begin{fact}\label{fact:expectation_nonnegative_variable}
	Let $\bm \xi$ be a nonnegative random variable. Then $\E \bm\xi = \int_{0}^\infty \Pr\Brac{\bm\xi > \tau} d\tau$.
\end{fact}

\begin{lemma}\label{lemma:expectation_from_tails}
	Let $\bm\eta$ be a random variable such that for some $a \in \R$ and for all $\tau > 0$, $\bm \eta \le a + \tau$ with probability at least $1-f(\tau)$ for some nonnegative $f\in \cL_1\Paren{(0, \infty)}$. Then
	\[
	\E \bm\eta \le a + \int_0^\infty f(\tau) d\tau\,.
	\]
\end{lemma}
\begin{proof}
	Denote $\bm \xi =  \ind{\bm \eta \ge a}\Paren{\bm \eta - a}$. Note that $\bm \eta \le a + \bm \xi$ and $\bm \xi$ is nonnegative, hence by \cref{fact:expectation_nonnegative_variable} we have
	\[
	\E\bm \xi = \int_0^\infty \Pr\Brac{\bm\xi > \tau} d\tau = \int_0^\infty \Pr\Brac{\bm\eta - a > \tau} d\tau = \int_0^\infty f(\tau) d\tau < \infty \,.
	\] 
	Hence either $\E \bm \eta = -\infty$ and the bound is trivially satisfied, or we can take the expectations form both sides of $\bm \eta \le a + \bm \xi$.
\end{proof}

\section{Sum-of-squares toolbox}
\label{sec:SoStools}

\subsection{Rounding}\label{sec:rounding}
\begin{lemma}\label{lem:vectorRecoveryOdd}
	Let $v\in \R^n$ be a unit vector.
	Let $p\ge 3$ be an odd number and let $\pE$ be a pseudo-distribution over $\R[x_1,\ldots,x_n]$ of degree $t\ge p+1$ such that $\pE \normt{x}^2 = 1$.
	If for some $\eps > 0$
	\[
	\Normt{v^{\otimes p}-\pE x^{\otimes p}} \le \eps\,,
	\]
	then $\tilde{v} = \pE x/\normt{\pE x}$ satisfies
	\[
	{\iprod{v,\tilde{v}}} \ge 1 - 2\eps\,.
	\]
\end{lemma}
\begin{proof}
	The proof is almost the same as the proof of Lemma 50 from \cite{tensor_pca_sos}. 
	
	Consider the univariate polynomial $f(u) = 1 - 2u^p + u$. 
	It is easy to verify that $f(u) \ge 0$ for all $u\in[-1,1]$. 
	Hence by Theorem 3.23 from \cite{Laurent2009SumsOS}, $f$ can be written as
	\[
	f(u) = s_1(u)(1 + u) + s_2(u)(1 - u)\,,
	\]
	where $s_1$ and $s_2$ are SoS polynomials of degree at most $p-1$.
	
	Now consider  $\pE f(\iprod{v,x})$. Since $\normt{v}^2 + \normt{x}^2 \pm 2\iprod{v,x}$ are SoS polynomials of degree $2$ in variables $x_1,\ldots, x_n$, for every SoS polynomial $s$ of degree at most $p-1$, we have
	\[
	\Abs{\pE \Brac{s(\iprod{v,x})\iprod{v,x}}} \le \frac{1}{2}\pE \Brac{ s(\iprod{v,x})\Paren{\normt{v}^2+\normt{x}^2}}\le \pE \Brac{s(\iprod{v,x})}\,.
	\]
	Hence
	\[
	\pE f(\iprod{v,x}) = \pE \Brac{s_1(\iprod{v,x})(1 + \iprod{v,x})} + \pE \Brac{s_2(\iprod{v,x})(1 - \iprod{v,x})} \ge 0\,,
	\]
	which implies that $\pE \iprod{v,x} \ge 2\pE \iprod{v,x}^p -  1$. On the other hand, since 
	$$\normt{\pE x^{\otimes p}}\ge \normt{v^{\otimes p}} - \Normt{v^{\otimes p}-\pE x^{\otimes p}}\ge 1-\eps,$$ 
	we have
	\[
	\pE \iprod{v,x}^p = \frac{1}{2}\Paren{\normt{v}^{2p} + \normt{\pE x^{\otimes p}}^2 - 
		\normt{v^{\otimes p}-\pE x^{\otimes p}}^2} \ge 1 - \eps\,.
	\]
	We conclude that
	\[
	\pE \iprod{v,x} \ge 2\pE \iprod{v,x}^p -  1 \ge 1 - 2\eps\,.
	\] 
\end{proof}

\begin{lemma}\label{lem:vectorRecoveryEven}
	Let $v\in \R^n$ be a unit vector.
	Let $p\ge 2$ be an even number and let $\pE$ be a pseudo-distribution over $\R[x_1,\ldots,x_n]$ of degree $t\ge p$ such that $\pE \normt{x}^2 = 1$.
	If for some $\eps > 0$
	\[
	\Normt{v^{\otimes p}-\pE x^{\otimes p}} \le \eps\,,
	\]
	then the top (unit) eigenvector $\tilde{v}$ of $\pE x x^\top$ satisfies
	\[
	{\iprod{v,\tilde{v}}}^2 \ge 1 - 4\eps\,.
	\]
\end{lemma}
\begin{proof}
	Consider the polynomial $f(u) = 1 - 2u^p + u^2$.
	It is easy to verify that $f(u) \ge 0$ for all $u\in[-1,1]$. 
	Hence by Theorem 3.23 from \cite{Laurent2009SumsOS}, $f$ can be written as
	\[
	f(u) = s_3(u) + s_4(u)(1 - u^2)\,,
	\]
	where $s_3$ and $s_4$ are SoS polynomials satisfying $\deg(s_3)\le p$ and $\deg(s_4)\le p-2$. It is easy to see that $\pE f(\iprod{v,x}) \ge 0$, and so $\pE \iprod{v,x}^2 \ge 2\pE \iprod{v,x}^p -  1$. On the other hand, since 
	$$\normt{\pE x^{\otimes p}}\ge \normt{v^{\otimes p}} - \Normt{v^{\otimes p}-\pE x^{\otimes p}}\ge 1-\eps,$$  we have
	\[
	\pE \iprod{v,x}^p = \frac{1}{2}\Paren{\normt{v}^{2p} + \normt{\pE x^{\otimes p}}^2 - 
		\normt{v^{\otimes p}-\pE x^{\otimes p}}^2} \ge 1 - \eps\,.
	\]
	Therefore,
	
	\[
v^\top\Paren{\pE x x^\top}v = \pE \iprod{v,x}^2 \ge 2\pE \iprod{v,x}^p -  1 \ge 1 - 2\eps\,.
	\]
	Hence, by \cref{lem:linear-algebra-correlation-eigenverctor-large-quadratic-form}, the top eigenvector of $\pE x x^\top$ satisfies the desired bound.
\end{proof}

\begin{fact}\label{lem:linear-algebra-correlation-eigenverctor-large-quadratic-form}
	Let $M\in \R^{d\times d}$ be such that $M\sge 0$ and $\Tr M = 1$, and let $z\in \R^d$ be a unit vector such that $\transpose{z}Mz\geq 1-\eps$. Then the top eigenvector $v_1$ of $M$ satisfies $\iprod{v_1,z}^2\geq 1-2\eps$.
	\begin{proof}
		Write $z=\alpha v_1+\sqrt{1-\alpha^2}v_{\bot}$ where $v_\bot$ is a unit vector orthogonal to $v_1$. Let $\lambda_1\geq\ldots\geq \lambda_d\geq 0$ be the eigenvalues of $M$. We have
		\begin{align*}
		1-\eps&\leq \transpose{z}Mz \\
		&= \alpha^2 \transpose{v_1}Mv_1+\Paren{1-\alpha^2}\transpose{v_\bot} M v_\bot\\
		&=\alpha^2 \Paren{\lambda_1-\transpose{v_\bot}Mv_\bot}+\transpose{v_\bot}M v_\bot\\
		&\leq\alpha^2 +\transpose{v_\bot}M v_\bot \,,
		\end{align*}
		where the last inequality follows from $M\sge 0$ and $\Tr M = 1$, which imply that $\transpose{v_\bot}Mv_\bot\geq 0$ and $\lambda_1\leq 1$.
		
		Now since $M\preceq \lambda_1 I_d$, we have $\lambda_1\geq \transpose{z}Mz\geq 1-\epsilon$, and
		$$\lambda_2+\ldots+\lambda_d=\Tr M-\lambda_1\leq 1-(1-\epsilon)=\epsilon\,.$$ Therefore, $\transpose{v_\bot}Mv_\bot\leq \eps$ and
		\begin{align*}
		\iprod{v_1,z}^2&=\alpha^2 \\
		&\geq 1-\eps-\transpose{v_\bot}M v_\bot\\
		&\geq 1-2\eps.
		\end{align*}
	\end{proof}
\end{fact}

\subsection{Sum-of-squares certificates for sparse PCA}\label{sec:sparse_pca_sos}
	Let $t\in \N$ be such that $1\le t \le k$ and let $\cS_t$ be the set of all $n$-dimensional vectors with values in $\set{0,1}$ that have exactly $t$ nonzero coordinates.
	
	We start with the following definition.
	
	\begin{definition}
		For every $u\in \cS_t$ we define the following polynomial in variables $s:=(s_1,\ldots,s_n)$
		\[
		p_u(s) = \binom{k}{t}^{-1}\cdot \underset{i \in \supp\Set{u}}{\prod} s_i\,.
		\]
	\end{definition}

	Note that if $x$ is a $k$-sparse vector and $s$ is the indicator of its support, then for every
	$u\in \cS_t$, we have
	\begin{align*}
	p_u(s)=\begin{cases}
	\binom{k}{t}^{-1}&\text{ if} \supp\Set{u}\subseteq\supp\Set{x} \,,\\
	0& \text{ otherwise}\,.
	\end{cases}
	\end{align*}
	
	Now consider the following system $\cC_{s,x}$ of polynomial constraints.  
	%\Tnote{We haxe two formulations but the number of constraints is the same.}
	
	\begin{equation}\label{eq:sparseConstraints_copy}
	\cC_{s,x}\colon
	\left \{
	\begin{aligned}
	&\forall i\in [n],
	& s_i^2
	& =s_i \\
	&&\textstyle \underset{i \in [n]}{\sum}s_i&=k\\
	&\forall i \in [n], &s_i\cdot x_i &=x_i\\
	&&\textstyle \underset{i \in [n]}{\sum}x_i^2&=1\\
	%\textstyle \underset{u\in \cN_t}{\sum}\;\,\underset{i \in \supp\paren{u}}{\prod}s_i &= \binom{k}{t}\\
	&&\underset{u \in \cS_t}{\sum} p_u(s) &=1\\
	&\forall i\in [n],
	&\underset{u \in \cS_t}{\sum} u_ip_u(s) &=\frac{t}{k} \cdot s_i
	%\\&& x\transpose{x} &= \frac{k}{t}\underset{u,u' \in \cN_t}{\sum} u'\transpose{u}P(x)_u P(x)_{u'}
	\end{aligned}
	\right \}
	\end{equation} 
	
	It is easy to see that if $x$ is $k$-sparse and $s$ is the indicator of its support, then $x$ and $s$ satisfy these constraints.
	
\begin{lemma}\label{lem:sos_limited_brute_force}
	Let $M\in \R^{n\times n}$ be a matrix. Denote by $m_t$ the maximal spectral norm among all $2t\times 2t$ principal submatrices of $M$. Then
	\[
	\cC_{s,x} \sststile{2t+2}{s,x}\Set{\transpose{x}Mx  \le 2\cdot m_t\cdot k/t}\,.
	\]
\end{lemma}
	\begin{proof}
		Without loss of generality we can assume that $M$ is symmetric, since otherwise we can replace $M$ by its symmetrisation. Indeed,
		$\transpose{x}Mx  = \transpose{x}\Paren{\frac{1}{2}M + \frac{1}{2}\transpose{M}} x$, and symmetrisation does not increase the spectral norms of principal submatrices.
			Note that 
		\[
		\cC_{s,x} \sststile{2t}{s} \Set{s\transpose{s} = \frac{k^2}{t^2}\underset{u,u' \in \cS_t}{\sum} u'\transpose{u}p_{u'}(s)p_u(s)}\,.
		\]
		For $x,y\in \R^n$ we denote the Hadamard product of $x$ and $y$ as $x\odot y$, i.e., $x\odot y$ is the vector in $\mathbb{R}^n$ with entries $(x\odot y)_i=x_i\cdot y_i$ for all $i\in [n]$.
		It follows that
		\begin{align*}
		\cC_{s,x} &\sststile{4}{s,x} \Set{x\transpose{x} = \Paren{x \odot s}\transpose{\Paren{x\odot s}} }\\
		&\sststile{2t+2}{s,x} \Set{x\transpose{x} =
			\frac{k^2}{t^2}\underset{u ,u'\in \cS_t}{\sum} 
			\Paren{x\odot u'}\transpose{\Paren{x\odot u}}p_{u'}(s)p_u(s)}
		\\
		&\sststile{2t+2}{s,x}\Set{{\transpose{x}Mx}  = 
			\frac{k^2}{t^2}\underset{u ,u'\in \cS_t}{\sum} 
			\transpose{\Paren{x\odot u}}M\Paren{x\odot u'}p_{u'}(s)p_u(s)}
				\\
		&\sststile{2t+2}{s,x}\Set{{\transpose{x}Mx}  = 
			\frac{k^2}{2t^2}\underset{u ,u'\in \cS_t}{\sum}
			\Paren{
			\transpose{\Paren{x\odot u}}M\Paren{x\odot u'} + 	\transpose{\Paren{x\odot u'}}M\Paren{x\odot u}}p_{u'}(s)p_u(s)}\,.
		\end{align*}
		
		Now for every $u,u' \in \cS_t$, let $M_{u,u'}$ be the matrix that coincides with $M$ at the entries $(i,j)$ such that both $i$ and $j$ are from the union of the supports of $u$ and $u'$, and is zero at other entries. Note that $\Paren{x\odot u}^\top M\Paren{x\odot u'} = \Paren{x\odot  u}^\top M_{u,u'}\Paren{x\odot  u'}$.
		Since $M_{u,u'}$ is symmetric, it is a difference of two PSD matrices $M^+_{u,u'}$ and $M^-_{u,u'}$ whose spectral norms are at most $\norm{M_{u,u'}} \le m_t$. Since for every PSD matrix $S$ we have 
		$\sststile{2}{a,b}\Set{\iprod{ab^\top + ba^\top, S}\le \iprod{aa^\top + bb^\top, S}}$ for variables $a,b\in \R^n$,  we get
		\begin{align*}
				\sststile{2}{x} \Bigg\{
					&\Paren{x\odot u}^\top M_{u,u'}\Paren{x\odot u'} +\Paren{x\odot u}^\top M_{u,u'}\Paren{x\odot u'}
				\\&= 
					\Paren{x\odot u}^\top M^+_{u,u'}\Paren{x\odot u'}  + \Paren{x\odot u'}^\top M^+_{u,u'}\Paren{x\odot u}  
					\\
					&\quad\quad\quad\quad\quad\quad\quad\quad\quad+ \Paren{-\Paren{x\odot u}}^\top M^-_{u,u'}\Paren{x\odot u'}  + \Paren{x\odot u'}^\top M^-_{u,u'}\Paren{-\Paren{x\odot u}}
				\\&\le 
				\Paren{x\odot  u}^\top M^+_{u,u'}\Paren{x\odot  u}  + \Paren{x\odot  u'}^\top M^+_{u,u'}\Paren{x \odot u'}  \\
				&\quad\quad\quad\quad\quad\quad\quad\quad\quad+ \Paren{x\odot  u}^\top M^-_{u,u'}\Paren{x \odot u}  + \Paren{x \odot u'}^\top M^-_{u,u'}\Paren{x\odot  u'} 
				\\&\le
				2\cdot m_t \cdot \Paren{\snorm{\Paren{x\odot  u}} + \snorm{\Paren{x\odot  u'}}}
				\Bigg\}\,.
		\end{align*}
		
		Since $\cC_{s,x} \sststile{t}{s} \Set{p_u(s) \ge 0}$, it follows that
		\begin{align*}
		\cC_{s,x} &\sststile{2t+2}{s,x}\Set{\transpose{x}Mx
			\le
			\frac{k^2}{t^2}\underset{u ,u'\in \cS_t}{\sum} 
			\Paren{	m_t\cdot \snorm{\Paren{x\odot u}}+
				m_t\cdot\snorm{\Paren{x\odot u'}}}p_{u'}(s)p_u(s)}
		\\
		&\sststile{2t+2}{s,x}\Set{\resizebox{0.8\textwidth}{!}{$\transpose{x}Mx
			\le
			m_t\frac{k^2}{t^2}\Paren{
				\underset{u\in \cS_t}{\sum} \snorm{\Paren{x\odot u}} p_u(s) \Paren{\underset{u' \in \cS_t}{\sum}p_{u'}(s)}
				+  	\underset{u'\in \cS_t}{\sum} \snorm{\Paren{x\odot u'}} p_{u'}(s)  \Paren{\underset{u\in \cS_t}{\sum}p_{u}(s)}
		}$}}
		\\
		&\sststile{2t+2}{s,x}\Set{\transpose{x}Mx
			\le
			m_t\frac{k^2}{t^2}\Paren{
				\underset{u\in \cS_t}{\sum} \snorm{\Paren{x\odot u}} p_u(s) 
				+  	\underset{u'\in \cS_t}{\sum} \snorm{\Paren{x\odot u'}} p_{u'}(s) 
		}}
		\\
		&\sststile{2t+2}{s,x}\Set{\transpose{x}Mx
			\le
			2m_t\frac{k^2}{t^2}
			\underset{u\in \cS_t}{\sum} \snorm{\Paren{x\odot u}} p_u(s)
		}\,.
		\end{align*}
		
		Now observe that
		\begin{align*}
		\cC_{s,x} &\sststile{t+2}{s,x}\Set{
			\underset{u\in \cS_t}{\sum} \snorm{\Paren{x\odot u}} p_u(s) =
			\underset{u \in \cS_t}{\sum} \sum_{i=1}^n x_i^2 u_i^2 \cdot p_u(s)}
		\\
		&\sststile{t+2}{s,x}\Set{
			\underset{u\in \cS_t}{\sum} \snorm{\Paren{x\odot u}} p_u(s) =
			\sum_{i=1}^n x_i^2 \underset{u \in \cS_t}{\sum} u_i \cdot p_u(s)}
		\\
		&\sststile{t+2}{s,x}\Set{
			\underset{u\in \cS_t}{\sum} \snorm{\Paren{x\odot u}} p_u(s) =
			\frac{t}{k}\sum_{i=1}^n x_i^2 s_i}
		\\
		&\sststile{t+2}{s,x}\Set{
			\underset{u\in \cS_t}{\sum} \snorm{\Paren{x\odot u}} p_u(s) =
			\frac{t}{k}}\,.
		\end{align*}
	\end{proof}

\begin{lemma}
\label{lem:sparcePCAGassianComplexity}
	Suppose that $n \ge k \ge t \ge 2$.
	Let $\cP_\ell$ be the set of all pseudo-distributions of degree $\ell$ on $\R[x,s]=\R[x_1,\ldots,x_n,s_1,\ldots,s_n]$ and let $\bm W\in \R^{n\times n}$ be a random matrix with i.i.d. standard Gaussian $N(0,1)$ entries. We have
	\[
			\E\Brac{\sup_{\substack{\mu\in \cP_{2t+2}:\\\mu \models \cC_{s,x}}}\pE_{x\sim\mu} x^\top \bm W x }\le  O\Paren{ k \sqrt{\frac{\log n}{t}}}\,.
	\]
\end{lemma}
\begin{proof}
	Fix a pseudo-distribution $\mu \in \cP_{2t+2}$ that satisfies $\cC_{s,x}$.
By  \cref{lem:sos_limited_brute_force}, we have
	 \[
	 \pE_{x\sim \mu} x^\top \bm  W x \le 2\cdot \bm w_t \cdot k / t\,,
	 \]
	 
	 where $\bm w_t$ is the maximal spectral norm among all $2t\times 2t$ principal submatrices of $\bm W$. Since this is true for every  $\mu \in \cP_{2t+2}$ satisfying $\cC_{s,x}$, we get
	 
	 \begin{equation}
	 \label{eq:lemB6}
			\E\Brac{\sup_{\substack{\mu\in \cP_{2t+2}:\\\mu \models \cC_{s,x}}}\pE_{x\sim\mu} x^\top \bm W x }\le 2\cdot \E[\bm w_t] \cdot k / t \,.	 
	 \end{equation}	 
	 
	For every $A\subseteq[n]$, let $\bm W_A$ be the principal submatrix of $\bm W$ that is obtained by taking the rows and columns with indices in $A$.  By \cref{fact:spectral_norm_gaussian}, we know that for every fixed $A\subseteq[n]$ of size $2t$ and every $0<\delta'<1$, we have
	  $$\norm{\bm W_A}\leq\sqrt{2t}+\sqrt{2t}+\sqrt{2\log(1/\delta')}\,,$$\
	  with probability at least $1-\delta'$. By taking a union bound over all $\binom{n}{2t}$ subsets $A\subseteq[n]$ of size $2t$, we can see that for every $0<\delta<1$, the following holds with probability at least $1-\delta$:
	   \begin{align*}
	 \bm w_t &= \max_{\substack{A\subseteq [n]:\\|A|=2t}}\norm{\bm W_A}\leq 2\sqrt{2t}+\sqrt{2\log\Paren{\frac{\binom{n}{2t}}{\delta}}}\leq 2\sqrt{2t}+\sqrt{2\log\Paren{\frac{\Paren{\frac{ne}{2t}}^{2t}}{\delta}}}\\
	 &=2\sqrt{2t}+\sqrt{4t\log(n)+4t - 4t\log(2t) + 2\log(1/\delta)}\\
	 &\leq 10\sqrt{t\log(n)}+\sqrt{2\log(1/\delta)} \,.
	 \end{align*}
	 In other words, for every $\tau>0$, with probability at least $1-\exp(-\tau^2/2)$, we have
	 \begin{align*}
	 \bm w_t &\leq 10\sqrt{t\log(n)}+\tau\,.
	 \end{align*}
	 
	 By applying \cref{lemma:expectation_from_tails}, we get
	 \begin{align*}
	 \E[\bm w_t] &\leq 10\sqrt{t\log(n)}+\int_0^\infty \exp(-\tau^2/2)d\tau \leq 10\sqrt{t\log(n)}+O(1)\,.
	 \end{align*}
	 Combining this with \eqref{eq:lemB6}, we get the result.
	  
%	   and a union bound over $\binom{n}{2t}$ principal submatrices,
%	 \[
%	 \bm m_t \le  4\sqrt{t} + 2\sqrt{2t\log n} + 10\sqrt{\log(1/\delta)}\le 10\sqrt{t\log n} + 10\sqrt{\log(1/\delta)}
%	 \]
%	 with probability at least $1-\delta/10$.
%	 
%	 Since $\bm M$ is symmetric,
%	 \[
%	 \pE x^\top \bm  M x = \sum_{i=1}^n \bm M_{ii} \pE x_i^2 + 2\sum_{i< j}^n \bm M_{ij} \pE x_ix_j \ge -\normi{\bm M} + 2\sum_{i< j}^n \bm W_{ij} \pE x_ix_j\,,
%	 \]
%	 and $\normi{\bm M} \le 20\sqrt{\log(1/\delta)}$ with probability at least $1-\delta/10$. Since the upper triangle of $\bm W$ coincides with the upper triangle of $\bm W$, we get 
%	 \[
%	 \sum_{i< j}^n \bm W_{ij} \pE x_ix_j \le \Paren{10\sqrt{t\log n} + 20\sqrt{\log(1/\delta)}}\cdot\frac{k}{t}
%	 \]
%	 with probability $1-\delta/5$. Similarly argument about the lower triangle of $\bm W$ shows that with probability $1-\delta/5$, 
%	 \[
%	 \sum_{i > j}^n \bm W_{ij} \pE x_ix_j \le \Paren{10\sqrt{t\log n} + 20\sqrt{\log(1/\delta)}}\cdot\frac{k}{t}\,.
%	 \]
%	 
%	 Hence with probability at least $1-\delta$, 
%	 \[
%	 		\sup_{\mu\in \cP_{2t+2}}\pE_{\mu}  x^\top \bm W x 
%	 		\le 20 k \sqrt{\frac{\log n}{t}} + 100\cdot \frac{k}{t}\sqrt{\log(1/\delta)}  + 1 
%	 		\le 100 k \sqrt{\frac{\log n}{t}} + 100\cdot \frac{k}{t}\sqrt{\log(1/\delta)}\,.
%	 \]
%	 
%	 Applying \cref{lemma:expectation_from_tails} with function $f(\tau) = \exp\Paren{-\Paren{\frac{t}{100k} \tau}^2}$ whose integral over $(0,\infty)$ is bounded by $O\Paren{k/t}$, we get the desired bound.
\end{proof}

\section{Facts about Gaussian and Rademacher complexities}\label{sec:gauss_rademacher}

Recall that for a bounded set $A\in \R^m$, the Gaussian complexity $\cG(A)$ is defined as
\[
\cG(A) = \E_{\bm w \sim N(0, \Id_m)}\sup_{a\in A} \sum_{i=1}^m a_i\bm w_i\,,
\]
and Rademacher complexity $\cR(A)$ is defined as
\[
\cR(A) = \E_{\bm s \sim U(\set{\pm1}^m)}\sup_{a\in A} \sum_{i=1}^m a_i\bm s_i\,,
\]
where $U(\set{\pm1}^m)$ is the uniform distribution over $\{+1,-1\}^m$.

\begin{fact}[\cite{wainwright_2019}, Proposition 5.28]\label{fact:lipschitz_transformation_complexity}
	Let $A\subset \R^m$ be a bounded set, and let 
	$\phi_1,\ldots,\phi_m : \R\to \R$ be $1$-Lipschitz functions that satisfy $\phi_i(0) = 0$ for all $i\in[m]$. 
	Denote $\phi: \R^m \to \R^m$, $\phi(x_1,\ldots, x_m) = \Paren{\phi_1(x_1),\ldots,\phi_m(x_m)}$. 
	Then
	\[
	\cG(\phi(A)) \le \cG(A) \quad \text{and} \quad
	\cR(\phi(A)) \le 2\cR(A)\,.
	\]
\end{fact}

\begin{fact}\cite{wainwright_2019}\label{fact:gaussian_rademacher_complexities_relation}
	For every bounded set $A\subset \R^m$, 
	\[
	\cR(A) \le \sqrt{\pi/2}\cdot \cG(A) \,.
	\]
\end{fact}

\begin{proof}
Let $\bm w_1,\ldots \bm w_m$ be iid standard Gaussian variables. Denote $\bm s_i = \sign(\bm w_i)$ and $\bm z_i = \abs{\bm w_i}$.
Since $\bm w_i$ are symmetric, $\bm s_i$ and $\bm z_i$ are independent. Since $\E \bm z_i = \sqrt{2/\pi}$,
\[
\cG(A) = \E \Brac{\sup_{a\in A} \sum_{i=1}^m a_i\bm w_i }= 
\E \E \Brac{\sup_{a\in A} \sum_{i=1}^m a_i \bm z_i \bm s_i \given \bm s}
\ge \sqrt{2/\pi}\E \Brac{\sup_{a\in A} \sum_{i=1}^m  a_i\bm s_i } = \cR(A)\,.
\]
\end{proof}

\begin{fact}[\cite{wainwright_2019}, Theorem 3.4]\label{fact:convex_lipshitz_function_of_rv}
	Let  $m\in \N$, $h > 0$ and let $\bm \xi_1,\ldots,\bm \xi_m$ be independent random variables such that $\forall i\in[m]$, $\Abs{\bm \xi_i} \le h$ with probability $1$. Let $L>0$ and let $f: \R^m \to \R$ be a convex $L$-Lipschitz function. Then for all $t > 0$,
	\[
	\Pr\Brac{f(\bm \xi_1,\ldots,\bm \xi_m) \ge 
		\E f(\bm \xi_1,\ldots,\bm \xi_m) + t} \le \exp\Paren{-\frac{t^2}{16\cdot L^2\cdot  h^2}}\,.
	\]
\end{fact}

\begin{lemma}\label{lem:subgaussianity_of_gradient}
	Let $m\in \N$ , $h > 0$ and let $\bm \xi_1,\ldots,\bm \xi_m$ be independent, symmetric about zero random variables such that $\forall i\in[m]$, $\Abs{\bm \xi_i} \le h$ with probability $1$. 
	Let $A\subset \R^m$ be a bounded set and denote $r_A = \sup_{a\in A}\normt{a}$. Let
	\[
	\bm S_A = \sup_{a\in A} \sum_{i=1}^m a_i\bm \xi_i\,.
	\]
	Then 
	\[
	\E\bm S_A \le 2\cdot h \cdot \cR(A) \le 3\cdot h \cdot \cG(A)\,,
	\]
	and for all $t>0$,
	\[
	\Pr\Brac{\bm S_A \ge \E\bm S_A + t} \le \exp\Paren{-\frac{t^2}{16\cdot  r_A^2 \cdot h^2}}\,.
	\]
\end{lemma}

\begin{proof}
	Let us first show the concentration bound. Consider the function $f:\R^m \to \R$ defined as $f(x) = \sup_{a\in A} \iprod{x,a}$. It is a convex function (as the supremum of convex functions), and $r_A$-Lipschitz since for all $x,y\in \R^m$,
	\[
	\iprod{x,a} - f(y) \le \iprod{x-y, a} \le \normt{a}\normt{x-y}\,,
	\]
	and if we take $\sup$ over $a\in A$, we get $f(x) - f(y) \le r_A \normt{x-y}$. The desired bound follows from \cref{fact:convex_lipshitz_function_of_rv}.
	
	Now let us bound the expectation. Denote $\bm s_i = \sign(\bm \xi_i)$ and $\bm \eta_i = \frac{1}{h} \abs{\bm \xi_i}$ so that $\bm \xi_i = h \cdot \bm s_i \cdot \bm \eta_i$. 
	Since $\bm \xi_i$ are symmetric, $\bm s_i$ and $\bm \eta_i$ are independent. We have
	\[
	\E \Brac{\sup_{a\in A} \sum_{i=1}^m a_i\bm \xi_i }= 
	\E \Brac{\E \Brac{\sup_{a\in A} \sum_{i=1}^m a_i \cdot h \cdot \bm \eta_i \bm s_i \given \bm \eta}}
	= h\cdot \E \Brac{\E \Brac{\sup_{a\in A} \sum_{i=1}^m \bm \phi_i(a_i)\cdot \bm s_i \given \bm \eta}}\,,
	\]
	where $\bm \phi_i :\R \to \R$ is defined as $\bm \phi_i(x_i) = \bm \eta_i x_i$. Since $0\leq \bm\eta_i\leq 1$ for all $i\in[m]$, we can see that $\bm \phi_1,\ldots,\bm \phi_m$ are all 1-Lipschitz. It follows from\cref{fact:lipschitz_transformation_complexity} that
	\[
	\E \Brac{\sup_{a\in A} \sum_{i=1}^m \bm \phi_i(a_i)\cdot \bm s_i \given \bm \eta} \le 2\cR(A)\,,
	\]
	and hence 
	$$\E\bm S_A\leq  2\cdot h \cdot \cR(A)\le 3\cdot h \cdot \cG(A)\,,$$ where the last inequality follows from \cref{fact:gaussian_rademacher_complexities_relation}.
\end{proof}

\begin{fact}[Sudakov's Minoration, \cite{wainwright_2019}, Theorem 5.30]\label{fact:SudakovMinoration}
	Let $A\subset \R^m$ be a bounded set. Then
	\begin{align*}
	 \sup_{\eps > 0}\frac{\eps}{2}\sqrt{\log \Abs{\cN_{\eps}(A)}} \le \cG(A) \,,
	\end{align*}
	where $\Abs{\cN_{\eps}(A)}$ is the minimal size of $\eps$-net in $A$ with respect to Euclidean distance.
\end{fact}

\section{Decomposability}\label{sec:decomposability}

In \cite{DBLP:conf-nips-dOrsiLNNST21} the Huber loss minimization was studied with $\ell_1$-norm and nuclear norm regularizers.
The authors of \cite{DBLP:conf-nips-dOrsiLNNST21} used a well-known property of these norms that is called \emph{decomposability} 
(this property has been extensively used in the  literature, see \cite{wainwright_2019}).

A norm $\norm{\cdot}_{\circ}$ in $\R^m$ is said to be \emph{decomposable} with respect to a pair of vector subspaces $\Paren{V,\bar{V}}$ such that $V\subseteq\bar{V}$, if for all $v\in V, u\in \bar{V}^\perp$, we have
$\norm{v + u}_{\circ} = \norm{v}_{\circ} + \norm{u}_{\circ}$. 
In the context of PCA, where the signal is rank-one symmetric matrix $X^* = \lambda \cdot vv^\top$, and
if $X^*\in V$ and both $V$ and $\bar{V}$ contain only rank $O(1)$ matrices, 
then decomposability of the nuclear norm implies  $\bm \Delta \in \Set{M:\normn{M} \le O\Paren{\normf{M}}}$.
This fact can be used to get a good bound on the error.
It is easy to see that for the nuclear norm there are natural spaces $V$ and $\bar{V}$ that satisfy these properties: 
$V = \text{span}\set{vv^\top}$ 
and $\bar{V} = \text{span}\set{uv^\top + vu^\top: u\in \R^m}$.

Now assume that the signal is a tensor $X^* = \lambda \cdot v^{\otimes 3}$.
One could try to apply the same approach for tensors: That is, to minimize the Huber loss with the dual norm of the injective tensor norm 
as a regularizer (let us ignore in this discussion computational aspects of the problem for simplicity). 
Recall that the injective tensor norm of order $3$ symmetric tensor $T$ is defined as
\[
\normin{T} := \sup_{\normt{x} = 1} \iprod{x^{\otimes 3}, T}\,.
\]

Let us check whether its dual norm $\normin{\cdot}^*$ is decomposable with respect to some natural subspaces of $V$ and $\bar{V}$ of low-rank tensors. 
The choice of $V$ is simple: it is always better if it is as small as possible, so we should just take $V = \text{span}\set{v^{\otimes 3}}$. 
And there are two candidates for $\bar{V}$ similar to the corresponding subspace in the matrix case: 
$\bar{V}_1 = \text{span}\set{\text{Sym}\Paren{u\otimes u \otimes v}: u\in \R^m}$ and 
$\bar{V}_2 = \text{span}\set{\text{Sym}\Paren{u\otimes w \otimes v}: u,w\in \R^m}$. 

$\bar{V}_2$ is not a good choice: It contains some tensors of rank $n$, and hence we cannot a get good bound if we use it (additional $\sqrt{n}$ factor appears in the error bound if we use it).

Let us now show that $\bar{V}_1$ is also not a good choice: 
The norm $\normin{\cdot}^*$ is not decomposable with respect to $\Paren{V,\bar{V}_1}$. 
Indeed, for unit vectors $u, w \in \R^m$ such that all $u,v,w$ are orthogonal to each other, 
the tensor $S = {\text{Sym}\Paren{v\otimes u\otimes w}}$ is in $\bar{V}^\perp$. 
By definition of $\normin{\cdot}^*$ ,
\[
\normin{v^{\otimes 3} + S}^* = \max\limits_{\normin{T} \le 1} \iprod{T, v^{\otimes 3} + S}. 
\]
We can assume without loss of generality that $T = \lambda v^{\otimes 3}  + \mu S$ for some $\lambda, \mu\in \R$. 
Then $\iprod{T, v^{\otimes 3} + S} \le  \lambda + \mu/6$.

Note that $\abs{\lambda}\le 1$, 
otherwise $\normin{T} > 1$.
Now consider $$x = \frac{1}{\sqrt{3}}\Paren{\sign(\mu) \cdot u+ \sign(\lambda) \cdot v+ \sign(\lambda) \cdot w}\,,$$ 
and note that
\[
\iprod{x^{\otimes 3}, T} = 3^{-3/2}\abs{\lambda} + 3^{-3/2} \abs{\mu} \le 1\,.
\]
The maximal value of the linear function $\lambda + \mu/6$ on the polygon
\[
\Set{(\lambda,\mu)\in \R^2 : \abs{\lambda}\le 1\,,\; \abs{\lambda} + \abs{\mu} \le 3^{3/2}}
\] 
is achieved at one of the vertices of this polygon, hence 
\[
\lambda + \mu/6 \le \max\Set{\sqrt{3}/2, 1 + \sqrt{3}/2 - 1/6} \le 5/6 + \sqrt{3}/2\,,
\]
and $\normin{v^{\otimes 3} + S}^*\leq 5/6 + \sqrt{3}/2$.

It is easy to verify that $\normin{v^{\otimes 3}}^* = 1$ and $\normin{S}^* = \sqrt{3}/2$, hence
\[
\normin{v^{\otimes 3} + S}^* < \normin{v^{\otimes 3}}^* + \normin{S}^*
\]
and the norm $\normin{\cdot}^*$ is not decomposable with respect to $\Paren{V,\bar{V}_1}$. 

This shows that naive approach fails and either we need to look for other sets $\bar{V}$ and try to prove decomposability for them, or not to use decomposability at all. We show that decomposability is not necessary for obtaining vanishing error\footnote{The only advantage of the analysis that uses decomposability compared to our approach is that it guarantees better error convergence: when the decomposability guarantees the error bound $O(\eps)$, our analysis guarantees the bound $O(\sqrt{\eps})$.}, and hence we can study Huber loss minimization over more complicated sets than nuclear norm ball or $\ell_1$ ball.

\end{document}